%% file: MAIN.tex
\documentclass{article}

\usepackage{enumitem}
\usepackage{fontawesome5} 
\usepackage{pgfgantt}
\usepackage{tikz}
\usepackage{subcaption}  
\usepackage{pgfplots}
\pgfplotsset{compat=1.18}
\usepackage{graphicx}    
\usepackage{tabularx}
\usepackage{cancel}     
\usepackage{fix-cm} 

\usepackage{enumitem}
\usepackage{listings} 
\usetikzlibrary{patterns} 
\newcommand{\tikzfontsize}{\fontsize{6pt}{7pt}\selectfont}
\usetikzlibrary{shadows,shadows.blur,shapes.geometric,arrows.meta,positioning,calc}

\tikzstyle{startstop} = [rectangle, rounded corners, minimum width=3cm, minimum height=1cm, text centered, draw=black, fill=cyan!30]
\tikzstyle{process} = [rectangle, minimum width=3cm, minimum height=1cm, text centered, draw=black, fill=orange!30]
\tikzstyle{miniprocess} = [rectangle, minimum width=2cm, minimum height=1cm, text centered, draw=black, fill=orange!30]
\tikzstyle{decision} = [diamond, minimum width=3cm, minimum height=1cm, text centered, draw=black, fill=red!30]
\tikzstyle{arrow} = [thick,->,>=stealth]
\usetikzlibrary{calc}

\definecolor{annotateblue}{RGB}{0,0,128}
\definecolor{annotategreen}{RGB}{0,100,0}

\usepackage{mathtools}
\usepackage{amsthm}
\usepackage{bm} 
\usepackage{algorithm}
\usepackage[noend]{algpseudocode}   
\usepackage{multicol}
\usepackage{caption}
\usepackage{verbatim}
\usepackage{longtable}
\usepackage{array}
\usepackage{verbatim}
\usepackage{courier}
\usepackage{fontawesome5}
\usetikzlibrary{positioning}
\PassOptionsToPackage{table}{xcolor}
\usepackage[table]{xcolor}
\definecolor{lightgray}{gray}{0.9}
\definecolor{lightred}{rgb}{1.0,0.6,0.6}
\newtheorem{lemma}{Lemma}  
\newtheorem{corollary}{Corollary}

\lstdefinelanguage{PDDL}{
  keywords={
    define, domain, problem, :requirements, :typing, :constants, :predicates,
    :action, :parameters, :precondition, :effect, :objects, :init, :goal,
    :functions, :durative-action, :duration, :condition, :metric, minimize,
    maximize, and, or, not, forall, exists, when, assign, increase, decrease,
    at, start, end, over, all, number, object
  },
  keywordstyle=\color{blue}\bfseries, 
  keywords=[2]{
    =, >=, <=, <, >, +, -, *, /
  }, 
  keywordstyle=[2]\color{orange}, 
  identifierstyle=\color{black}, 
  sensitive=false, 
  morecomment=[l]{;}, 
  commentstyle=\color{purple}\itshape, 
  morestring=[b]", 
  stringstyle=\color{red},
}[keywords,comments,strings] 

\lstdefinestyle{pddlstyle}{
    language=PDDL, 
    basicstyle=\small\ttfamily,
    breaklines=true,
    showstringspaces=false,
    frame=single,
    numbers=left,
    numberstyle=\tiny\color{gray}
}

\usepackage{xspace} 

\DeclareMathOperator*{\argmin}{arg\,min}

\definecolor{codegreen}{rgb}{0,0.6,0}
\definecolor{codegray}{rgb}{0.5,0.5,0.5}
\definecolor{codepurple}{rgb}{0.58,0,0.82}
\definecolor{backcolour}{rgb}{0.95,0.95,0.92}

\lstdefinestyle{protobuf}{
    backgroundcolor=\color{backcolour},   
    commentstyle=\color{codegreen},
    keywordstyle=\color{blue},
    numberstyle=\tiny\color{codegray},
    stringstyle=\color{codepurple},
    basicstyle=\ttfamily\small,
    breakatwhitespace=false,         
    breaklines=true,                 
    captionpos=b,                    
    keepspaces=true,                 
    numbers=left,                    
    numbersep=5pt,                  
    showspaces=false,                
    showstringspaces=false,
    showtabs=false,                  
    tabsize=2,
    frame=single
}

\lstdefinestyle{python}{
    language=Python,
    backgroundcolor=\color{backcolour},
    commentstyle=\color{codegreen},
    keywordstyle=\color{blue},
    numberstyle=\tiny\color{codegray},
    stringstyle=\color{codepurple},
    basicstyle=\ttfamily\small,
    breakatwhitespace=false,         
    breaklines=true,                 
    captionpos=b,                    
    keepspaces=true,                 
    numbers=left,                    
    numbersep=5pt,                  
    showspaces=false,                
    showstringspaces=false,
    showtabs=false,                  
    tabsize=4,
    frame=single
}

\newcommand{\ALAS}{\mathsf{Alas}}
\newcommand{\MP}{\mathsf{Alas}}

     \usepackage[preprint]{neurips_2025}

\usepackage[utf8]{inputenc} 
\usepackage[T1]{fontenc}    
\usepackage{hyperref}       
\usepackage{needspace}
\usepackage{url}            
\usepackage{booktabs}       
\usepackage{amsfonts}       
\usepackage{nicefrac}       
\usepackage{microtype}      
\usetikzlibrary{patterns}
\makeatletter
\renewcommand\paragraph{\@startsection{paragraph}{4}{\z@}%
  {0.5ex \@plus0.5ex \@minus0.2ex}%
  {-1em}%
  {\normalfont\normalsize\bfseries}}
\makeatother

\title{ALAS: A Stateful Multi-LLM Agent Framework for Disruption-Aware Planning}

\author{%
  Edward Y. Chang\thanks{echang@cs.stanford.edu} \\
  Department of Computer Science\\
  Stanford University\\
   \And
  Longling Gerng\\
  Department of Computer Science\\
  Stanford University\\
}

\begin{document}
\maketitle

\begin{abstract}
Large language models (LLMs) excel at rapid generation of text and 
multimodal content, yet they falter on transaction‐style planning that 
demands \emph{ACID-like guarantees} and real-time disruption recovery. 
We present \emph{Adaptive LLM Agent System} ($\ALAS$), a framework that tackles four fundamental LLM deficits:
(i) absence of self-verification, (ii) context erosion, 
(iii) next-token myopia, and (iv) lack of persistent state. 
$\ALAS$ decomposes each plan into role-specialized agents, equips them with
automatic state tracking, and coordinates them through a lightweight
protocol. When disruptions arise, agents apply \emph{history-aware local 
compensation}, avoiding costly global replanning and containing cascade 
effects. 
On real-world, large-scale job-shop scheduling benchmarks, 
$\ALAS$ sets new best results for static sequential planning and
excels in dynamic reactive scenarios with unexpected disruptions. 
These gains show that principled modularization plus targeted
compensation can unlock scalable and resilient planning with LLMs.
\end{abstract}


\input{Introduction}
\input{RelatedWork}
\input{MetaPlanner}
\input{Experiments}

\input{Conclusion}

\bibliographystyle{plain}
\bibliography{TemporalPlanning, EdwardChang, ToBeVerifiedReferences, TSP}

\newpage
\appendix


\section*{Appendices}
\begin{itemize}
    \item Appendix A: $\ALAS$ Architecture and Detailed Workflow
    \item Appendix B: LCSR Algorithm Specification and Lemma Proofs
    \item Appendix C: Urban Ride Assignment Problem: Extended Experimental Results
    \item Appendix D: Family Reunion Planning Problem: Extended Experimental Results
    \item Appendix E: Additional Case Studies and Analyses
\end{itemize}
\noindent Additional supplementary materials, including benchmark datasets, code implementations, and extended analyses, will be uploaded to the Open Review submission site.

\input{Appendix4Section3}

\input{AppendixAgentFactory}
\input{AppendixLCSR}

\input{AppendixURS}
\input{AppendixFamilyReunion}

\input{AppendixJSSPSupplement}

\end{document}

%% file: Introduction.tex
\needspace{-0.1in}
\vspace{-.1in} 
\section{Introduction} \label{sec:ALAS-intro}
\vspace{-.1in} 

Large Language Models (LLMs) have revolutionized AI, demonstrating remarkable capabilities across a wide range of natural language tasks~\cite{LLMSurvey2025, matarazzo2025SurVeyLLMs, minaee2025llmssurvey, wan2023efficientLLMSurVey}. However, despite their fluency and versatility, standalone LLMs remain fundamentally limited when applied to planning and decision-making~\cite{lecun2022pathways}, particularly in scenarios that demand long-range consistency, coordination, or adaptive behavior. Even in seemingly simple tasks, such as generating travel itineraries, constructing execution sequences, or resolving scheduling constraints, LLMs often produce incomplete, inconsistent, or logically invalid outputs.

These shortcomings stem from architectural limitations intrinsic to transformer-based models. LLMs lack internal mechanisms for verifying their own outputs~\cite{godel1931english, hong2024verificationabilities}, 
exhibit solution space bias from maximum-likelihood decoding~\cite{SocraSynthChangCSCI2023, holtzman2020curious, radford2019language}, 
suffer from attention drift and information loss in long contexts~\cite{hsieh2024lostinthemiddle, liu-etal-2024-lost, vaswani2017attention, xiao2024attentionsink}, 
and accumulate cascading errors across reasoning chains~\cite{chu2024COTsurvey,patel2024multiLogiEval,xiong2024largelanguagemodelslearn}. 
Without persistent memory, they also fail to track commitments, causal dependencies, or temporal constraints in planning stages, leading to hallucinations, constraint violations, and incoherent updates.

Such limitations are particularly acute in dynamic environments where plans must respond to runtime disruptions, e.g., last-minute cancelations, emergent constraints, or unexpected failures. Domains like logistics, event coordination, and industrial operations demand systems that can revise partial plans while preserving consistency and feasibility. In these reactive settings, global replanning with traditional optimization methods can be counterproductive: even small disruptions may trigger wholesale rescheduling, resulting in excessive job movement, increased latency, and degraded performance. Conversely, the stateless and myopic behavior of LLMs often leads to invalid responses.

To address these challenges, we propose $\ALAS$ (Adaptive LLM Agent System), a multi-agent architecture designed for structured planning and adaptive execution. Rather than treating the LLM as a monolithic planner, $\ALAS$ orchestrates a network of lightweight agents, each tailored to mitigate a specific structural limitation:

\begin{enumerate}[leftmargin=1.2em, topsep=0pt, itemsep=-.1pt, label=\arabic*.]
\item \emph{Validation Agents} verify feasibility and enforce hard constraints.
\item \emph{Domain Agents} explore low-probability, high-utility alternatives to reduce solution bias.
\item \emph{Context Agents} preserve coherence by operating within semantically scoped subcontexts.
\item \emph{Monitoring Agents} detect anomalies and trigger local corrective actions.
\item \emph{Memory Modules} maintain evolving state and dependency graphs for rollback and consistency.
\end{enumerate}

\input{Figure-URSNetwork} 

To concretely illustrate these challenges, we adopt the \emph{Urban Ride Sharing} (URS) task as a running example throughout the paper. The URS problem, depicted in Figure~\ref{fig:URS-DirectedGraph}, involves coordinating multiple vehicles to transport passengers amid traffic delays and last-minute requests. Unlike the classic NP-hard traveling salesman problem (TSP) \cite{lawler1985traveling}, URS demands concurrent execution, inter-agent coordination, and adaptation to disturbances, underscoring the need for persistent state tracking and local rescheduling. More complex problems will be examined in Section~\ref{sec:ALAS-evaluation} after the basic concepts have been established.

$\ALAS$  excels in dynamic environments where reactive decision-making, incremental adjustments, and local consistency deliver superior results compared to global recomputation approaches. Whereas traditional optimization methods remain valuable for static combinatorial problems, $\ALAS$ shines uniquely by providing explicit reasoning throughout its process, generating comprehensive audit trails, and maintaining full transparency for real-time decisions and post-mortem analysis. This empowers human operators to understand agent rationales, trace the origin of any issues, and verify compliance with business constraints, which is invaluable in complex and rapidly changing scenarios.

$\ALAS$ contributions are:
\begin{enumerate}[leftmargin=1.2em, topsep=0pt, itemsep=-.05pt, label=\arabic*.]
\item \textit{Modular Agent Architecture:} Decomposes planning into specialized agents, overcoming LLM limitations in validation, context retention, error handling, and state tracking.
\item \textit{Persistent Execution Memory:} Introduces a lightweight memory abstraction to track state transitions, enabling rollback, compensation, and causal consistency.
\item \textit{Cross-Domain Generalization:} Demonstrates robust, adaptive planning in structurally diverse domains, including transportation, event scheduling, and industrial operations.
\item \textit{Disruption-Aware Replanning:} Enables localized corrections using persistent memory, avoiding costly global recomputation in response to runtime failures.
\end{enumerate}

%% file: Figure-URSNetwork.tex
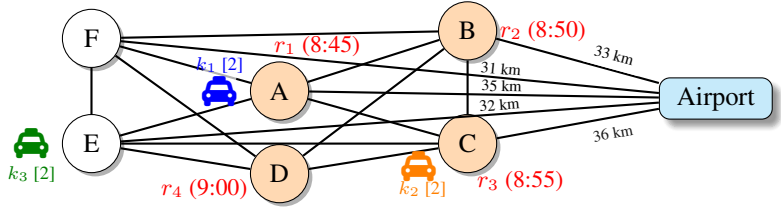
\begin{figure}[t!]
\vspace{-.15in}
    \centering
    \begin{tikzpicture}[
        location/.style={
            circle,
            draw=black,
            fill=yellow!20,
            minimum size=0.4cm, 
            text centered,
            font=\normalsize, 
            drop shadow={shadow xshift=0.5mm, shadow yshift=-0.5mm, opacity=0.6}
        },
        airport/.style={
            rectangle,
            rounded corners,
            draw=black,
            fill=cyan!20,
            minimum width=1.5cm, 
            minimum height=0.5cm, 
            text centered,
            font=\normalsize, 
            blur shadow={shadow xshift=1mm, shadow yshift=-1mm, shadow blur steps=5}
        },
        passenger/.style={
            font=\small\color{red} 
        },
        vehicle/.style={
            font=\small\color{blue} 
        },
        urban_path/.style={thick},
        airport_path/.style={thick},
        taxi/.style={scale=1.3, text opacity=1} 
    ]
        \node[location, fill=orange!30, draw, inner sep=2pt, minimum size=22pt] (A) at (0.0,0.9) {A}; 
        \node[location, fill=orange!30, draw, inner sep=2pt, minimum size=22pt] (B) at (2.5,1.7) {B}; 
        \node[location, fill=orange!30, draw, inner sep=2pt, minimum size=22pt] (C) at (2.5,0.2) {C};
        \node[location, fill=orange!30, draw, inner sep=2pt, minimum size=22pt] (D) at (0,-0.2) {D};
        \node[location, fill=white, draw, inner sep=2pt, minimum size=22pt] (E) at (-2.5,0.2) {E};
        \node[location, fill=white, draw, inner sep=2pt, minimum size=22pt] (F) at (-2.5,1.6) {F}; 
            
        \node[airport] (G) at (5.8,0.8) {Airport}; 
        
        \foreach \from/\to in {A/B, B/C, C/D, D/E, E/F, F/A, A/C, B/D, C/E, D/F, E/A, F/B}
            \draw[urban_path] (\from) -- (\to) node[midway,fill=none] {};
            
        \draw[airport_path] (A) -- (G) node[pos=0.55,sloped,above,yshift=-0.5ex,font=\tiny,opacity=1.0] {35 km};
        \draw[airport_path] (B) -- (G) node[pos=0.7,sloped,above,fill=white,font=\tiny,opacity=1.0] {33 km};
        \draw[airport_path] (C) -- (G) node[pos=0.7,sloped,below,font=\tiny,opacity=0.9] {36 km};
        \draw[airport_path] (E) -- (G) node[pos=0.705,sloped,above,yshift=-0.5ex,font=\tiny,opacity=1.0] {32 km};
        \draw[airport_path] (F) -- (G) node[pos=0.705,sloped,above,yshift=-0.5ex,font=\tiny,opacity=1.0] {31 km};
        
        \node[circle, fill=white, opacity=0.7, minimum size=13pt] at ($(A) + (-0.8,0.2)$) {};
        \node[blue, taxi] at ($(A) + (-0.8,0.0)$) {\faIcon{taxi}};
        \node[blue, font=\scriptsize] at ($(A) + (-0.8,0.35)$) {$k_1$ [2]}; 
        
        \node[circle, fill=white, opacity=0.7, minimum size=13pt] at ($(C) + (-0.6,-0.5)$) {};
        \node[orange, taxi] at ($(C) + (-0.65,-0.3)$) {\faIcon{taxi}};
        \node[orange, font=\scriptsize] at ($(C) + (-0.6,-0.6)$) {$k_2$ [2]}; 
        
        \node[circle, fill=white, opacity=0.7, minimum size=13pt] at ($(E) + (-0.8,0)$) {};
        \node[green!50!black, taxi] at ($(E) + (-0.8,0)$) {\faIcon{taxi}};
        \node[green!50!black, font=\scriptsize] at ($(E) + (-0.8,-0.4)$) {$k_3$ [2]}; 
        
        \node[passenger] at ($(A)+(0.5,0.6)$) {$r_1$ (8:45)};
        \node[passenger] at ($(B)+(1.0,0.0)$) {$r_2$ (8:50)};
        \node[passenger] at ($(C)+(0.7,-0.5)$) {$r_3$ (8:55)};
        \node[passenger] at ($(D)+(-1.0,-0.2)$) {$r_4$ (9:00)};
    \end{tikzpicture}
    \vspace{-.05in}
    \caption{Network $G=(V,E)$ with urban travel times $\tau_{ij}=10$ minutes and airport routes distance specified on the figure. Static scenarios can be solved by MILP or Column Generation. Dynamic scenarios (e.g., an accident, a cancellation, new passenger requests) must be addressed by $\ALAS$.}
    \label{fig:URS-DirectedGraph}
    \vspace{-.2in}
\end{figure}

%% file: RelatedWork.tex
\vspace{-.1in}
\section{Related Work} \label{sec:ALAS-related}
\vspace{-.05in}
We organize related work into three parts: 1) structural limitations of LLMs in planning, 2) comparisons with existing multi-agent systems, and 3) LLM-based planning frameworks and benchmarks.

\vspace{-.1in}
\subsection{Structural Limitations of LLMs} 
\label{sec:ALAS-related-llmlimits}
\vspace{-.05in}

Transformer-based LLMs~\cite{brown2020language,vaswani2017attention} have demonstrated impressive capabilities across NLP tasks. However, their architecture imposes limitations when applied to planning domains requiring reasoning over evolving states, enforcing global constraints, and maintaining consistency throughout execution.

\textit{Lack of Self-Validation.} LLMs lack internal mechanisms to assess constraint adherence \cite{bommasani2022foundationmodels}
and reasoning consistency \cite{SocraticIEEECCWC2023}, consistent with Gödel's incompleteness theorem~\cite{godel1931english} and confirmed empirically~\cite{hong2024verificationabilities}. Techniques such as Self-Ask~\cite{madaan2022}, Chain-of-Thought (CoT) dissection~\cite{li2023dissectingCOT}, and error detection~\cite{jiang2024selfincorrect} offer partial gains, constrained by model capacity~\cite{huang2024large}. Even structured validators for math~\cite{feng2023towards} and logic~\cite{gou2024tora} cannot solve open-domain validation challenges~\cite{chen-etal-2024-llmarena, huang2024planningllmsurvey}.

\textit{Solution Space Bias.} Maximum-likelihood training biases LLMs toward 
frequent, high-probability completions~\cite{SocraSynthChangCSCI2023, holtzman2020curious, radford2019language}, limiting diversity~\cite{EVINCEChang2024} and creativity~\cite{ismayilzada2024evaluatingcreativeshortstory}.
This hampers performance in domains where edge-case reasoning is essential.

\textit{Context Degradation.} As context length increases, LLMs exhibit attention dilution and positional confusion~\cite{liu-etal-2024-lost, park2023context, wei2023larger, xiao2024attentionsink}. These effects are particularly pronounced for mid-context tokens, where relevant information is often ignored.

\textit{Error Propagation.} 
Stepwise reasoning strategies such as Chain-of-Thought like reasoning~\cite{besta2024graph, WeiCOT3600270, yao2024tree} amplifies early mistakes, leading to logical drift in multi-step problems~\cite{prystawski2023why, stechly2024COTLimits}. Without explicit state tracking, LLMs struggle to recover from such inconsistencies.

\textit{Absence of Persistent State.} LLMs are stateless by design, lacking mechanisms to recall history or enforce temporal consistency. In dynamic tasks like Urban Ride Sharing (Figure~\ref{fig:appDURSSpec}), this leads to failures when previously valid plans become infeasible due to disruptions, e.g., a delay after the first pickup may invalidate a previously computed route that the LLM cannot revalidate without rollback.

These limitations motivate our architectural design. \(\ALAS\) addresses these issues through specialized agents, persistent memory, and runtime monitoring.

\vspace{-.1in}
\subsection{Multi-Agent Planning Systems}
\vspace{-.05in}
Multi-agent systems (MAS) have long been studied for distributed planning and coordination~\cite{wooldridge2009introduction}. Recent LLM-based MAS frameworks—such as AutoGen~\cite{wu2024autogen}, LangGraph~\cite{langgraph2024}, and CAMEL~\cite{li2023camel}—leverage LLM agents in conversational workflows, but often lack persistent memory, formal plan validation, and mechanisms for recovery from disruption. 

We collaborated with the AutoGen team to incorporate persistent memory into their April 2024 release. However, most LLM-MAS platforms still struggle to maintain consistency over time, particularly under dynamic conditions.

An independent MAS evaluation~\cite{cemri2025multiagentllmsystemsfail}, which includes AppWorld~\cite{trivedi2024appworld}, ChatDev~\cite{qian2023chatdev}, HyperAgent~\cite{phan2024hyperagent}, and MetaGPT~\cite{hong2023metagpt}, found no reliable improvement over single-agent systems across 150+ tasks. Common failures included miscoordination, unsatisfied preconditions, and invalid outputs. These issues are especially harmful in domains like URS, where agents must synchronize commitments and recover gracefully from partial failures.

In contrast, \(\ALAS\) emphasizes recovery, state tracking, and local error containment, enabling robust collaboration under uncertainty. Its persistent memory modules and compensatory agent design fill critical gaps overlooked by prior systems.

\vspace{-.1in} 
\subsection{LLM-Based Planning and Benchmarks} \label{sec:ALAS-relatedBench} 
\vspace{-.05in} 

\paragraph{LLM-Based Planning.}
Recent frameworks have applied LLMs to planning via language-driven decomposition (PLASMA~\cite{PLASMA2024procedural}), heuristic search (LLM-MCTS~\cite{zhao2023LLM-MCTS}), and multi-agent orchestration (ADAS~\cite{hu2024ADAS}, AFlow~\cite{zhang2024aflow}). Systems such as StarJob~\cite{abgaryan2025starjobdatasetllmdrivenjob} and LLM-guided scheduling~\cite{abgaryan2024llmsschedule} attempt to use LLMs as direct solvers for optimization problems. However, this overlooks a critical fact: LLMs are not competitive with classical combinatorial solvers in producing globally optimal static plans.

A survey by Huang et al.~\cite{huang2024planningllmsurvey} categorizes these LLM-agent planning frameworks and highlights the open challenges in state tracking, consistency, and reactivity. $\ALAS$ distinguishes itself by addressing these exact weaknesses through persistent memory, agent modularity, and reactive compensation strategies.
LLMs should instead be leveraged for their unique strengths: interpretability, adaptability, and contextual reasoning. Their potential lies in fine-tuning and repairing partial plans under disruption. This principle aligns with ReAct~\cite{yao2022react}, which combines reasoning and acting in LLMs, but lacks persistent state tracking or rollback mechanisms.

Classical methods such as the Shifting Bottleneck Procedure~\cite{adams1988shifting}, Genetic Algorithms~\cite{kirkpatrick1983optimization, bierwirth1995generalized}, Simulated Annealing~\cite{aarts1989simulated}, and CP/Tabu hybrids~\cite{nowicki1996fast} consistently outperform LLMs on standard static scheduling benchmarks. However, these solvers often falter in dynamic settings, where disruptions require real-time adjustment and context-aware reasoning.

Our experiments confirm that even basic planning tasks expose structural limitations in LLMs—particularly in validation, memory persistence, and execution consistency—as detailed in Section~\ref{sec:ALAS-related-llmlimits}.

\vspace{-.05in} 
\paragraph{Planning Benchmarks.}
We extend the classical Demirkol-DMU Job Shop Scheduling Problem (JSSP-DMU)~\cite{DEMIRKOL1998137, shylo2018DMU} by introducing execution-time disruptions such as machine failures, delayed jobs, and late-order changes. JSSP, a foundational operations research problem~\cite{XIONG2022105731}, involves resource allocation under precedence constraints and is highly sensitive to perturbations. These properties make it a strong testbed for evaluating persistent memory, rollback, and adaptive repair.

In contrast, common LLM benchmarks like HotPotQA~\cite{yang2018hotpotqa}, ALFWorld~\cite{shridhar2020alfworld}, and BIG-Bench~\cite{srivastava2022beyond} focus on static reasoning. More recent data sets, PlanBench~\cite{valmeekam2023planbench}, TimeBench~\cite{chu2024timebench}, and ACPBench~\cite{abdin2024acpbench} introduce temporal structure, but rarely include disruptions, revalidation, or compensation.

To address these limitations, we turn to the Job Shop Scheduling Problem (JSSP), a widely studied problem in operations research~\cite{XIONG2022105731} that offers valuable characteristics for evaluating planning under realistic conditions. JSSP involves scheduling $J$ jobs across $M$ machines with complex dependencies and inherently requires handling dynamic constraints, interruptions, and recovery, key concerns emphasized in Industry 5.0~\cite{destouet2023flexible} but underrepresented in existing LLM planning evaluations.

Although recent studies have begun to formulate JSSP-like benchmarks~\cite{abgaryan2025starjobdatasetllmdrivenjob} or evaluate static scenarios using LLMs~\cite{abgaryan2024llmsschedule}, the work remains preliminary. Our study uses the classical, large-scale Demirkol-DMU dataset~\cite{DEMIRKOL1998137, shylo2018DMU} and TA benchmark to evaluate disruption-aware planning at scale.


%% file: MetaPlanner.tex
\vspace{-.1in}
\section{The ALAS Three-Layer Architecture for Adaptive and Reactive Planning} 
\label{sec:ALAS-architecture}
\vspace{-.1in}
Existing work reveals critical limitations in LLM-based planning and coordination, particularly under dynamic and stateful conditions. In response, we introduce $\ALAS$ (Adaptive LLM Agent Scheduler), a general-purpose architecture for structured planning and adaptive execution.
$\ALAS$ adopts a three-layer architecture: \emph{workflow blueprint}, \emph{agent factory}, and \emph{runtime monitor}, working with \emph{persistent memory} to transform specifications into robust and verifiable execution workflows.

\input{ALASAlgorithmCompressed}

\vspace{-.05in}
\subsection{Workflow Blueprinting Layer: Template Construction}
\label{sec:ALAS-Layer1}
\vspace{-.1in}
The Template Construction Layer translates a planning input $\mathcal{O}$ into a structured workflow template $\mathcal{W}_{\text{template}} = (\mathcal{N}, \mathcal{E})$, where nodes $\mathcal{N}$ represent abstract roles and contexts (e.g., locations or coordination points), and edges $\mathcal{E}$ define execution dependencies between roles. This template, constructed by module $\mathcal{T}$, serves as an abstract blueprint for downstream agent synthesis.

\textbf{Phase 1: Workflow Template Construction.}  
Given $\mathcal{O}$, $\ALAS$ defines a directed graph in which each node may embed unresolved roles. In the URS scenario (Figure~\ref{fig:URS-DirectedGraph}), nodes correspond to locations (A–F). A passenger awaiting pickup at location D is expressed as $\text{Role}_{\text{ToBePickedUp}}(r_j, D)$ and becomes resolved when a real individual (e.g., Emily) submits a ride request. Similarly, a pickup role such as $\text{Role}_{\text{Pickup}}(k_i, \text{E})$ is filled when a driver (e.g., J. Doe from location E) is assigned. At scheduling time, these role bindings are resolved as:
$\text{Role}_{\text{Pickup}}(\text{J. Doe}, \text{E}) \rightarrow \text{Role}_{\text{ToBePickedUp}}(\text{Emily}, \text{D})$, which means driver J. Doe at location E is assigned to pick up Emily at location D.
This separation between role specification and runtime resolution enables flexible scheduling and dynamic adaptation.

\textbf{Phase 2: Agent Role Specification.}  
Each unresolved role is annotated with (i) an execution agent profile that defines the required capabilities, protocols, timing, states, logging schema onto persistent memory, and (ii) a compensation profile that prescribes recovery mechanisms for local failures. For example, a driver role agent's logic might include driver resolution, driver state, ETA updates, and fare processing, while its compensation logic handles rerouting or fallback assignment. Edge constraints are embedded into adjacent node logic and later enforced at execution. Appendix~\ref{app:ALAS-Layer2} presents details of the agent design process, including the state space model, formal agent specifications, and the implementation of different agent categories. 

\textbf{Phase 3: Validation and Refinement.}  
To ensure correctness, $\mathcal{T}$ delegates verification to an independent validator agent to check the soundness of the graph, temporal consistency and compensation coverage. Inspired by Gödel's incompleteness principle (Section~\ref{sec:ALAS-related-llmlimits}), this validator is external to $\mathcal{T}$ and avoids circular self-verification. If violations are found, $\mathcal{T}$ refines the workflow by adjusting node capabilities, rewiring dependencies, or modifying role assignments until $\mathcal{W}_{\text{template}}$ is in compliance.

Algorithm~\ref{alg:metaplanner} summarizes this three-phase construction; full implementation details are provided in Appendix~\ref{app:AlasArchitecture} as well as its flow chart in Figure~\ref{fig:ALAS-Layer1}.


\vspace{-.1in}
\subsection{Agent Factory Layer: Agent Sourcing and Compensation Logic}
\label{sec:ALAS-Layer2}
\vspace{-.1in}

The Agent Factory layer implements the abstract roles defined in $\mathcal{W}_{\text{template}}$ by instantiating executable agents that perform task-specific logic. For each role,
$\mathcal{W}_{\text{template}}$ provides the following specs:
the agent's capability requirements and execution context;
a logging schema \( \mathcal{L}_i \) describing what state to persist and when; and a compensation profile with recovery triggers and fallback actions.

These specifications are stored in \( \mathcal{W}_{\text{template}} \) and form the blueprint for constructing both the execution and compensation agents.

\textbf{Agent Signatures.}  
Each primary agent is defined as:
{\setlength{\abovedisplayskip}{2pt} \setlength{\belowdisplayskip}{2pt} 
\begin{equation}
\alpha_i = \langle \beta_i, \mathbf{c}_i, w_i, e_i, \mathcal{L}_i \rangle,
\label{eq:agent-signature}
\end{equation}}
where \( \beta_i \) is the inter-agent dependencies and communication their protocol, \( \mathbf{c}_i \) the capability profile, \( w_i \) the context, \( e_i \) the efficiency constraint, and \( \mathcal{L}_i \) the trace logging schema.

Each agent is paired with a fault-tolerant compensator:
{\setlength{\abovedisplayskip}{2pt} \setlength{\belowdisplayskip}{2pt} 
\begin{equation}
\alpha_i^{\text{comp}} = \langle \beta_i, t^{\text{comp}}, \mathbf{c}^{\text{comp}}, w_i, e_i, \mathcal{L}_i, \Gamma_i \rangle,
\label{eq:comp-agent-signature}
\end{equation}}
where \( t^{\text{comp}} \) defines the recovery protocol, and \( \Gamma_i \) tracks exception states and mitigation history.

\vspace{.5ex}
\noindent
The code generation process from these specifications is handled via prompting a selected LLM (e.g., Gemini, OpenAI, or Claude). Once the agents are successfully instantiated and their logic compiled, the workflow template is finalized as an executable plan, denoted by $\mathcal{W}_{\text{exec}}$.

Technical details and examples are provided in Appendix~\ref{app:ALAS-Layer2}.


\vspace{-.1in}
\subsection{Runtime Layer: Execution and Reactive Adaptation}
\label{sec:ALAS-Layer3}
\vspace{-.05in}

Given a fully instantiated execution plan \( \mathcal{W}_{\text{exec}} \) from the Agent Factory, the Runtime Layer activates agents in dependency order under temporal constraints. Each agent \( \alpha_i \) executes its assigned role, logs status to persistent memory, and signals readiness to downstream agents. Execution proceeds as a distributed, event-driven process, with global alerts mediating inter-agent communication.

\paragraph{Disruption Handling.} 
To maintain feasibility under execution-time perturbations—e.g., machine failures, delayed arrivals, or last-minute task insertions—$\ALAS$ activates corresponding compensation agents \( \alpha_i^{\text{comp}} \) and applies three classes of reactive strategies:

\begin{enumerate}[leftmargin=1.2em, topsep=-.1pt, parsep=-.1pt, itemsep=0pt, label=\arabic*.]
  \item \textit{Local Compensation}: Attempt retry, rollback, or local delay, scoped to the impacted agent and its immediate neighborhood.
  \item \textit{Queue Reordering}: Modify local execution queues to maximize slack or shift lower-priority operations. Notably, terminal tasks can be safely delayed, and initial tasks may be advanced without risk of WIP violation.
  \item \textit{Minimizing Re-optimization Costs}: Global replanning—while theoretically optimal—often triggers massive WIP movement and invalidates prior executions. To prevent this, $\ALAS$ imposes a disruption penalty (e.g., WIP cost \( t_{\text{WIP}} \)) and halts reordering if cost exceeds projected gain.
\end{enumerate}

\noindent
The Local Reactive Compensation Protocol (LRCP) prioritizes low-overhead, localized recovery. In most cases, a single round of compensation and minimal queue reordering suffices to restore feasibility. If recovery is infeasible or too costly, execution terminates to prevent cascading disruption.

Assuming $M \approx J$ and each machine queue evaluates up to $S$ swaps (bounded by the $t_{\text{WIP}}$ tradeoff), the message complexity is $\mathcal{O}(S J O_{\max} + J M O_{\max})$. Since $J > M$ and $S < O_{\max}$ in practice, the effective complexity reduces to $\mathcal{O}(J^2 O_{\max})$—scalable for real-time use.

Algorithm~\ref{alg:ALAS-LCSR-Full} presents the detailed steps of LRCP, with full procedural details and convergence lemmas, appears in Appendix~\ref{app:lcsr}.

\paragraph{Master Coordinator and Memory Logging.}
The LLM that synthesized \( \mathcal{W}_{\text{exec}} \) serves as \emph{master coordinator}, overseeing global progress:
(1) keeping constraints, validating local execution, and bookkeeping global states,
(2) monitoring disruptions and issuing alerts to affected agents, and
(3) receiving updates from agents.

Persistent memory stores all state transitions, dependencies, agent logs, and compensation actions. This supports rollback, post-hoc diagnostics, and statistical monitoring. While the current coordinator is single-threaded, failover extensions are possible.


%% file: AlasAlgorithmCompressed.tex
\begin{algorithm}[h!]
\caption{Workflow Template Construction by $\mathcal{T}$ (Summary). Details in Appendix~\ref{app:ALAS-Algorithm-Full}.}
\label{alg:metaplanner}
\vspace{-.2in}
{\fontsize{7.8}{9.2}\selectfont
\begin{multicols}{2}
\begin{algorithmic}[1]
\Require Task specification $\mathcal{O}$, constraint set $D$
\Ensure Validated $\mathcal{W}_{\text{template}} = (\mathcal{N}, \mathcal{E})$

\State Extract abstract roles $\mathcal{R}$ from $\mathcal{O}$
\State Map roles to nodes $\mathcal{N}$ with profiles $\mathcal{P}_{n_i}$
\State Determine inter-role dependencies $\mathcal{E}$ under constraints $D$
\State Annotate each $n_i \in \mathcal{N}$ with agent specification $\alpha_i$
\State Annotate each $n_i$ with compensation spec $\alpha_i^{\text{comp}}$
\State Assemble initial template $\mathcal{W}_{\text{template}} = (\mathcal{N}, \mathcal{E})$
\While{$\mathcal{W}_{\text{template}}$ fails validation}
  \State Check structure, constraints, and agent specs
  \State Refine nodes, edges, or agent specs
\EndWhile
\State \Return validated workflow template $\mathcal{W}_{\text{template}}$
\end{algorithmic}
\end{multicols}}
\vspace{-.12in}
\end{algorithm}
\vspace{-.12in}

%% file: Experiments.tex
\vspace{-.15in}
\section{Experimental Evaluation}
\label{sec:ALAS-evaluation}
\vspace{-.1in}
Our experimental evaluation assesses $\ALAS$ across three domains of increasing complexity. We designed experiments to demonstrate: (1) how $\ALAS$ overcomes LLM limitations and (2) its scalability in large real-world settings. Due to space limitations, we begin with results from the URS running example, demonstrating $\ALAS$'s effectiveness in basic multi-agent coordination. We then extend to \textit{Event Coordination}: the Family Reunion scenario, which reveals standalone LLMs' limitations in handling complex interdependencies while showing $\ALAS$'s capabilities in both initial planning and disruption response. Finally, we evaluate \textit{Job Shop Planning (JSP)}: this industrial benchmark tests $\ALAS$'s scalability and effectiveness in reactive planning for large, tightly constrained, disruption-prone environments with critical time dependencies.

\vspace{-.05in}
\paragraph{Metrics:} The evaluation metrics include: (1) static sequential planning and dynamic reactive adjustment accuracy, (2) planning efficiency measured by, e.g., the shortest travel distance, common sense consideration, and makespan (total time required to complete all jobs).

\vspace{-.05in}
\paragraph{Experiment Setup:} 
We compare $\ALAS$ with four leading LLMs: GPT-4o-Task \cite{openai2024gpt4o}, DeepSeek R1 \cite{deepseekai2025deepseekr1}, Claude 3.7 Sonnet \cite{anthropic2024claude}, and Gemini 2.5 Pro \cite{GeminiV25-2005}. All methods use online and API interfaces with default parameters (temperature=1.0). The total cost is less than US$200$ through service subscriptions and corporate-sponsored credits. Each experiment is repeated \textbf{ten} times with fresh context threads. While we summarize results in this section, detailed methodology, prompts, and complete results appear in Appendices~\ref{app:ALAS-URAP}--\ref{app:ALAS-JSSP}.

\input{Table-ProblemStatements1and2}

\vspace{-.1in}  
\subsection{Case Study 1: Transportation Scheduling}  
\vspace{-.05in}  
The purpose of this case study is to conclude the illustrative example and to demonstrate that even in a simple problem, standalone LLMs can perform poorly and fail to meet basic planning requirements.

\vspace{-.1in}  
\subparagraph{Problem Specification:}  
The Urban Ride Sharing (URS) problem (Table~\ref{fig:appDURSSpec}) involves coordinating vehicles to transport passengers between locations before deadlines. The objective is to minimize total travel distance while satisfying all temporal constraints and handling mid-execution disruptions.

\vspace{-.1in}  
\subparagraph{Prompts:}  
Initial prompt: ``Create an optimal schedule for this ride-sharing scenario that minimizes total travel distance while meeting all passenger deadlines: [Table~\ref{fig:appDURSSpec} specification].''  
Reactive prompt: ``Passenger $r_2$ cancels at 8:05 AM; a new passenger $r_5$ at location F requests pickup with a 9:30 AM airport deadline. Update the schedule accordingly.''  
$\ALAS$ augments these prompts with structured workflow templates (Algorithm~\ref{alg:metaplanner}) and role-based agent instantiation. Each method is evaluated in 10 independent trials with means and 
standard deviations reported.


\vspace{-.05in}  
\paragraph{\#1. Experimental Results of Sequential Planning:}  
All models met deadlines; however, $\ALAS$ achieved superior efficiency: average distance of 95.1 km (SD 13.04 km) vs. 118.9 km (SD 16.6 km) for baseline LLMs, a 20\% improvement ($p < 0.01$). Figure~\ref{fig:combined-URS} illustrates the optimal schedule proposed by $\ALAS$.

\input{Figure-Case1-TransportationScheduling}

\vspace{-.05in}  
\paragraph{\#2. Results of Reactive Planning:}  
When passenger $r_2$ canceled at 8:05 AM and $r_5$ appeared at 8:10 AM, $\ALAS$ replanned successfully in all trials. In contrast, baseline LLMs failed to maintain consistent internal states, frequently losing track of vehicle locations, duplicating assignments, or ignoring updated deadlines. These breakdowns are due to the structural limitations and statelessness
of LLMs (Section~\ref{sec:ALAS-related-llmlimits}).
(The complete results and discussion are provided in Appendix~\ref{app:ALAS-URAP}.)

\vspace{-.1in}
\subsection{Case Study 2: Event Coordination}
\label{sec:ALAS-exp-case2}
\vspace{-.05in}

This case study highlights the underlying reasons why most standalone LLMs often fail, even in static sequential planning, and exhibit inconsistent behavior in reactive scenarios. In contrast, $\ALAS$ consistently performs well on both static and dynamic settings.
\vspace{-.05in}
\paragraph{Problem Specification:} 
The Family Reunion problem (Table~\ref{fig:FamilyReunionSpec}) requires coordinating family members for interconnected tasks including airport pickups and food preparation before a 6:00 PM dinner deadline, while respecting travel times and capabilities.

\vspace{-.05in}
\paragraph{Prompts:} 
All methods received: ``Create a detailed schedule for this family reunion that satisfies all constraints: [Table~\ref{fig:FamilyReunionSpec} specification].'' For reactive planning: ``At noon, James' flight is delayed until 4:00 PM. Update the schedule to meet the deadline at 6:00 pm while satisfying all constraints.'' $\ALAS$ added workflow templates via Algorithm~\ref{alg:metaplanner} with each method tested on 10 independent runs.



\vspace{-0.05in}
\paragraph{\#1. Experimental Results in \emph{Sequential} Planning.}
Across ten independent trials, $\ALAS$ produced a \emph{feasible} schedule in every run, whereas standalone LLM baselines violated hard constraints and often required multiple retries before converging (see Tables~\ref{tab:reunion-Claude} and~\ref{tab:reunion-DeepThink} in Appendix~\ref{app:ALAS-MeetingProblem}).  
Typical errors included travel–time miscalculations—treating 60-minute journeys as 45 or 30 minutes (highlighted in shaded red)—and common-sense violations such as starting turkey preparation at 10 AM (four hours earlier than necessary) or allowing James to depart the airport without Emily.
The primary cause of these failures is \emph{context loss} in long prompts, which degrades arithmetic accuracy and consistency~\cite{liu-etal-2024-lost,xiao2024attentionsink}.  
$\ALAS$ avoids this degradation via its compartmentalised architecture: each agent receives only task-relevant facts, and an independent validator checks temporal and resource constraints before plan finalisation.
Missing commonsense was further mitigated by a dedicated \emph{commonsense agent} that augments the problem specification with realistic slack (e.g., luggage retrieval, rental-car pickup, turkey preparation).  Appendix~\ref{app:ThanksGiving-common-sense} details this augmentation process.
\vspace{-0.05in}
\paragraph{\#2. Experimental Results in \emph{Reactive} Planning.}
We next introduced a disruption—James's flight is delayed to \textbf{16{:}00}.  In ten trials, DeekSeek and Claude failed in \(7/10\) cases by assigning Emily to a taxi (violating family pickup constraints), postponing Grandma's pickup beyond 17{:}30, or missing the 18{:}00 dinner deadline (see Figure~\ref{tab:failed-reactive-planning}).  
These failures stem from greedy, one-shot rescheduling and enlarged prompt length during replanning, both of which exacerbate context loss~\cite{wei2023larger,zhang2023careful}.  
By contrast, $\ALAS$ succeeded in all ten runs (see Figure~\ref{tab:ALASReactivePlan}), through its reactive LRCP loop: persistently updating state (Michael's driving state and James's new arrival), detecting downstream conflicts (Emily/Grandma pickups), evaluating alternatives (Michael's direct airport route instead of home), and validating constraints (temporal/resource bounds). Crucially, $\ALAS$'s persistent state history enables targeted \emph{undo/redo} of affected tasks already scheduled, allocating more slack and ensuring on-time convergence.  Full logs and additional discussion appear in Appendix~\ref{app:ALAS-MeetingProblem}.
\vspace{-0.1in}
\subsection{Case Study 3: Job Shop Scheduling with Disruptions}
\label{sec:ALAS-exp-case3}
\vspace{-0.05in}
\paragraph{Problem Specification:} 
We evaluate our framework on JSSP as explained at the start. We selected two complementary benchmarks, and also extended them by introducing machine breakdowns that require schedule adjustments that minimize makespan and work-in-progress movement. 
\begin{itemize}[leftmargin=1.2em, topsep=-.05pt, itemsep=-.05pt]
    \item \textit{Demirkol-DMU}~\cite{demirkol1998computational} Designed for stress-testing algorithms with varying job priorities and routing requirements. These larger instances (20×15 to 50×20 job-machine configurations) evaluate scalability and robustness under complex constraints.
    \item \textit{Taillard (TA)} is a widely adopted set of job-shop scheduling problem instances characterized by medium to large-scale problem sizes (15×15, 20×20, 30×20, 100x20), with uniform job–machine mappings and tight constraints. This standard benchmark enables direct comparison with state-of-the-art optimization techniques.
\end{itemize}

\vspace{-0.1in}
\subsubsection{Experimental Results}

Based on the results reported in Tables II and III of the SeEvo (self-evolutionary) work~\cite{huang2024JSSPLLMs}, we compare $\ALAS$ against the same set of baselines (plus SeEvo). All algorithms evaluated in SeEvo are included in our experiments, namely: Gene Expression Programming (GEP)~\cite{NIE2013389}, Multi-Tree Genetic Programming (MTGP)~\cite{Zhang2018MTGP}, a suite of heuristic dispatching rules \cite{pinedo2012scheduling} (Random selection, Longest Processing Time (LPT), Shortest Processing Time (SPT), Shortest Total Processing Time (STPT), Most Process-Sequence Remaining (MPSR), Longest Subsequent (LSO), SPT$\times$TWK, SPT/TWKR), the deep-reinforcement-learning methods DRL-Chen~\cite{Chen2023}, DRL-Liu~\cite{Liu2024}, and DRL-Zhang~\cite{Zhang2020PDR}.

\vspace{-0.05in}
\paragraph{Static, Sequential Planning, both benchmarks}
Figure~\ref{fig:ALAS-JSSP-exp1} reports $\ALAS$+LRCP against the above methods on both Demirkol-DMU and TA benchmarks. In the Demirkol-DMU benchmark, $\ALAS$+LRCP ran on Claude-3.7 achieved the lowest mean gap to the theoretical optimal result, outperforming SPT (36\%), LPT (53\%), and DRL-Liu (21\%). 

While Figure~\ref{fig:ALAS-JSSP-exp1} reports mean makespan values across benchmarks, Figure~\ref{fig:JSSP-DMU-TA-trendline} in Appendix~\ref{app:ALAS-JSSP} provides a detailed comparison of performance at the individual JSSP instance level. As shown, $\ALAS$+LRCP consistently outperforms all other methods in virtually every case, demonstrating the robustness of our approach across varying problem sizes and complexities.

To assess cross-domain generalization, we tested $\ALAS$-dynamic on the TA benchmark consisting of seven real-world industrial instances. Here, $\ALAS$+LRCP again outperformed all baselines with a mean gap of only 0.86\%, compared to DRL-Chen (25\%), DRL-Zhang (18\%), and SeEvo-GPT3.5 (15\%). Classical heuristics such as LSO showed  substantially larger gaps of up to 34\%. These results demonstrate the effectiveness of $\ALAS$-dynamic scheduling in both classical heuristics and trained reinforcement learning approaches in different structured domains.

\input{Figure-JSSP-Sequential}

\vspace{-0.05in}
\paragraph{Dynamic, Reactive Planning under random disruptions}
We simulate 20 random machine failures on the DMU instances and compare LRCP ($\ALAS$'s reactive planner) with other baselines. LRCP adapts to disruptions  and maintains the lowest average gap (19.8\%), outperforming the SeEvo-GPT3.5/GLM3, GP, and DRL variants.

\vspace{-0.1in}
\subsubsection{LRCP Mechanism: A White Box Analysis}
\vspace{-0.05in}
Let us use a 5×3 JSSP example to illustrate how LRCP works, demonstrating its effectiveness, efficiency, and guaranteed convergence. Furthermore, unlike other methods, LRCP explicitly accounts for rescheduling overhead, a crucial practical consideration in manufacturing environments.

\paragraph{Phase 1: Local Compensation.} 
Figure~\ref{fig:combined-scheduling} shows how $\ALAS$'s LCSR protocol makes local adjustments to repair schedules after machine breakdowns. In Figure~\ref{fig:combined-scheduling}(a), the original CP/Tabu hybrid schedule has a 19-unit makespan. When M1 breaks down from $t=5-8$ (Figure~\ref{fig:combined-scheduling}(b)), its agent pushes operation J2(3) to start at $t=8$, causing a 3-unit delay. Each machine's agent $\alpha_i$ maintains information about scheduled jobs and their connected machines and timing. This delay triggers agent $\alpha_1$ to notify $\alpha_2$ on M2 about downstream impacts (Figure~\ref{fig:combined-scheduling}(c)). M2's compensation agent $\alpha^{comp}_2$ then reschedules J3(2) and J1(3) to start at $t=19$ and $t=22$. The cascade continues when $\alpha_2$ communicates with $\alpha_0$ about J(3)'s delayed arrival, prompting $\alpha^{comp}_0$ to delay J3(3) to $t=22$ (Figure~\ref{fig:combined-scheduling}(d)). This coordination extends the makespan from 19 to 26 units, reflecting both the 3-unit downtime and 4-unit delay propagation. During compensation, all operations are delayed rather than advanced, avoiding costly work-in-progress (WIP) movement that contemporary approaches often overlook \cite{park2023context, zhang2023careful}.

\input{Figure-JSSP-Disruptive-Phase1}
\input{Figure-JSSP-Disruptive-Phase2}

\vspace{-.05in}
\paragraph{Phase 2: Queue Reordering.} 
$\ALAS$ next employs queue reordering, which may incur a $t_{\text{WIP}}$ penalty for expedited job movement. Agents evaluate makespan improvement versus WIP penalties—unlike approaches optimizing solely for makespan \cite{wei2023larger}.
From Figure~\ref{fig:combined-scheduling}(d), $\ALAS$ follows Algorithm~\ref{alg:ALAS-LCSR-Full} to move J2(3) and J4(3) to M1's queue end. These terminal operations create no conflicts or WIP penalties since they're delayed (Figure~\ref{fig:queue-reordering}(a)).
This reordering creates a gap allowing J3(1) and J1(2) to advance by 7 units. J3(1) incurs a $t_{\text{WIP}} = 1$ penalty—a worthwhile trade. For J1(2), the agent verifies J1(1)'s completion time ($t = 6$), confirming feasibility with no WIP penalty as J3(2)'s timing masks its early staging. LRCP terminates after exhausting improvement options.
Figure~\ref{fig:queue-reordering}(b) shows the final 22-unit makespan, with only one unit of WIP movement and minimal messaging overhead. Given the 3-unit disruption from the 19-unit baseline, this demonstrates $\ALAS$'s effectiveness in disruption-aware targeted replanning.

%% file: Table-ProblemStatements1and2.tex
\begin{table}[t!]
\vspace{-.25in}
\begin{scriptsize}
\begin{minipage}[t]{0.45\textwidth}
\renewcommand{\arraystretch}{1.1}
\setlength{\fboxsep}{4pt}
\caption{Dynamic Urban Ride-Sharing}
\fbox{
\begin{minipage}{\textwidth}
       \fontsize{8pt}{9pt}\selectfont
\textbf{Objectives:} Schedule vehicles to deliver passengers to airport during [8:45, 9:00], minimizing vehicle travel distance while ensuring on-time arrivals and maximizing passenger satisfaction. \\
\textbf{Locations:} Seven locations: $V = \{A, B, \cdots, F\}$, where $G$ is Boston Airport (BOS). Urban locations $A$–$F$ are all 10 km apart, airport distances 30+ km. 
\fontsize{7pt}{9pt}\selectfont
    \[
    \begin{bmatrix}
        & A  & B & C & D & E & F \\
    \rightarrow G & 35 & 33 & 36 & 34 & 32 & 31
    \end{bmatrix}
    \]
\fontsize{8pt}{9pt}\selectfont
    \textbf{Travel Speed:} ($A$–$F$) 60, ($A$–$F \rightarrow G$) 100 km/h
    
\textbf{Passenger Requests:} with BOS arrival deadlines:  
    \begin{itemize}[leftmargin=1em, topsep=1.2pt, itemsep=0.7pt, label=-]
    \item $r_1$: $A$, to $G$ by 08:45 \hspace{0.5em} - $r_2$: $B$, to $G$ by 08:50
    \item $r_3$: $C$, to $G$ by 08:55 \hspace{0.5em} - $r_4$: $D$, to $G$ by 09:00
    \end{itemize}
\textbf{Available Vehicles} (Capacity 2 passengers): 
    \begin{itemize}[leftmargin=1em, topsep=1.2pt, itemsep=0.7pt, label=-]
        \item $k_1$: at $A$, $k_2$: at $C$, and $k_3$: at $E$
        \item Battery levels: $k_1$: 90\%, $k_2$: 75\%, $k_3$: 60\%
    \end{itemize}
    
\textbf{Potential Disruptions:} New passenger requests, vehicle availability changes (battery levels/breakdowns at 0.05/hour), and traffic delay, etc. Replanning may require rolling back promised pickup time to the existing passengers and replan for all.
    \end{minipage}
    } 
    \label{fig:appDURSSpec}
\end{minipage}
\hspace{0.03\textwidth}
\begin{minipage}[t]{0.45\textwidth}
        \caption{Family Reunion Planning Problem}
        \renewcommand{\arraystretch}{1.1}
        \setlength{\fboxsep}{4.pt}
        \fbox{
        \begin{minipage}{\textwidth}
        \fontsize{8.0pt}{9pt}\selectfont
        \textbf{Objectives:} On time family reunion dinner at 6:00 PM \\
        \textbf{Family Members and Arrivals:}
        \begin{itemize}[leftmargin=1em, topsep=-0.2pt, itemsep=-0.2pt, label=-]
        \item Sarah (Mom): Host, at home
        \item James (Dad): Lands at BOS 1:00 PM from SF
        \item Emily (Sister): Lands at BOS 2:30 PM from Chicago
        \item Michael (Bro): Driving, arrives 3:00 PM from NY
        \item Grandma: Needs pickup from suburban Boston
        \end{itemize}
        \textbf{Cooking Requirements:}
        \begin{itemize}[leftmargin=1em, topsep=-.25pt, itemsep=-.25pt, label=-]
        \item Turkey: 4 hours cooking time; side dishes: 2 hours
        \item Someone must stay home during oven baking time
        \end{itemize}
        \textbf{Transportation Constraints:}
        \begin{itemize}[leftmargin=1em, topsep=-.25pt, itemsep=-.25pt, label=-]
        \item James must rent car after landing
        \item Emily requires airport pickup
        \item Travel times: Home to BOS Airport: 60 min
        \item Travel times: BOS Airport to Grandma's: 60 min
        \item Travel times: Home to Grandma's: 30 min
        \end{itemize}
        \textbf{Key Constraints:}
        \begin{itemize}[leftmargin=1em, topsep=-.25pt, itemsep=-.25pt, label=-]
        \item All family members home before 6:00 PM
        \item Turkey and sides ready by 6:00 PM
        \item All pickups completed with available drivers
        \item Oven baking supervision maintained
        \end{itemize}
        \end{minipage}
}
        \label{fig:FamilyReunionSpec}
    \end{minipage}
    \end{scriptsize}
    \vspace{-.1in}
\end{table}

%% file: Figure-Case1-TransportationScheduling.tex
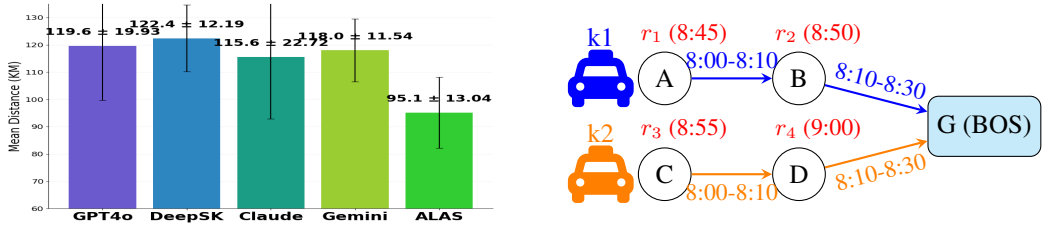
\begin{figure}[t!]
\vspace{-.1in}
    \centering
    \begin{minipage}{0.48\textwidth}
        \centering        \includegraphics[width=1.0\textwidth,height=0.45\textwidth]{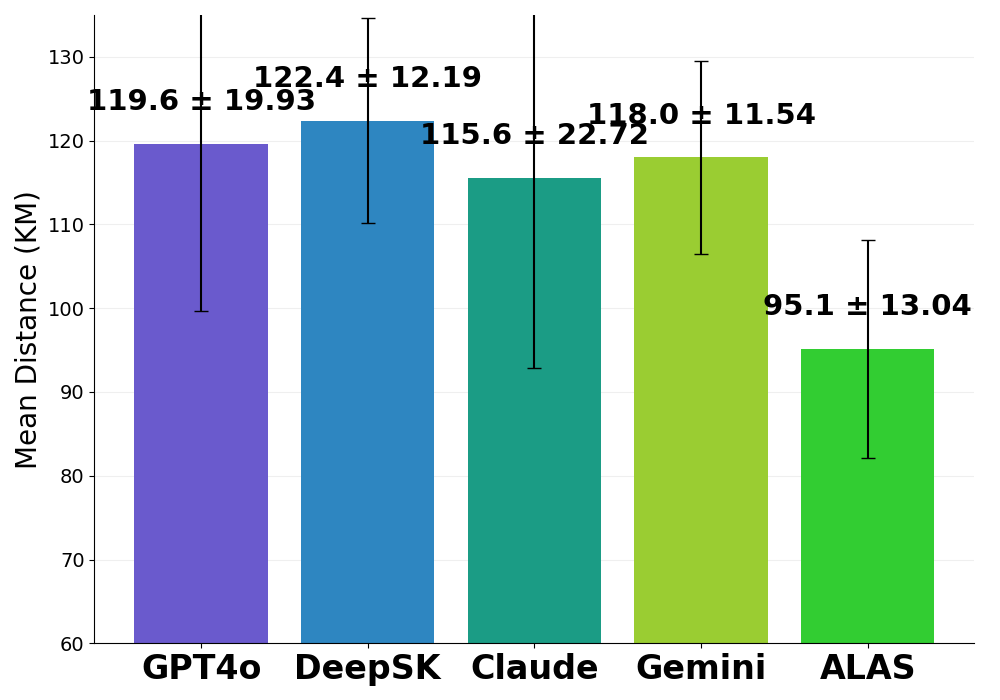}
    \end{minipage}%
    \hfill
    \begin{minipage}{0.48\textwidth}
        \centering
        \begin{tikzpicture}[
            location/.style={
                circle,
                draw=black,
                fill=white,
                minimum size=0.7cm,
                text centered,
                font=\normalsize
            },
            airport/.style={
                rectangle,
                rounded corners,
                draw=black,
                fill=cyan!20,
                minimum width=1.5cm,
                minimum height=0.8cm,
                text centered,
                font=\normalsize
            },
            passenger/.style={
                font=\small\color{red}
            },
            path_k1/.style={
                thick,
                blue,
                ->,
                >=stealth
            },
            path_k2/.style={
                thick,
                orange,
                ->,
                >=stealth
            },
            vehicle/.style={
                draw,
                fill=white,
                minimum width=0.8cm,
                minimum height=0.4cm,
                rounded corners=1pt
            },
            scale=0.85 
        ]
            \node[blue] at (-1,1.5) {\Huge\faTaxi};
            \node[blue] at (-1,2.15) {k1};
            \node[location] (A) at (0,1.5) {A};
            \node[location] (B) at (2.1,1.5) {B};
            \node[location] (C) at (0,0) {C};
            \node[location] (D) at (2.1,0) {D};
            
            \node[orange] at (-1,0) {\Huge\faTaxi};
            \node[orange] at (-1,0.65) {k2};
            
            \node[airport] (G) at (5,.75) {G (BOS)};
            
            \draw[path_k1] (A) -- (B) node[midway,above,sloped,fill=none,font=\footnotesize] {8:00-8:10};
            \draw[path_k1] (B) -- (G) node[midway,above,sloped,fill=none,font=\footnotesize] {8:10-8:30};
            
            \draw[path_k2] (C) -- (D) node[midway,below,fill=none,font=\footnotesize] {8:00-8:10};
            \draw[path_k2] (D) -- (G) node[midway,below,sloped,fill=none,font=\footnotesize] {8:10-8:30};
            
            \node[passenger] at ($(A)+(0.3,0.7)$) {$r_1$ (8:45)};
            \node[passenger] at ($(B)+(0.3,0.7)$) {$r_2$ (8:50)};
            \node[passenger] at ($(C)+(0.3,0.7)$) {$r_3$ (8:55)};
            \node[passenger] at ($(D)+(0.3,0.7)$) {$r_4$ (9:00)};
            
        \end{tikzpicture}
    \end{minipage}
    \caption{Comparison of ride-sharing solutions generated by $\MP$ and baseline LLMs. (Left) Mean total travel distance (km) with standard deviation error bars over 10 independent runs for each method, illustrating $\MP$'s improved efficiency. (Right) Optimal schedule generated by $\MP$ for the URS task, utilizing two vehicles ($k_1, k_2$) to serve four passengers ($r_1$-$r_4$).} 
\label{fig:combined-URS}
\vspace{-.18in}
\end{figure}

%% file: Figure-JSSP-Sequential.tex
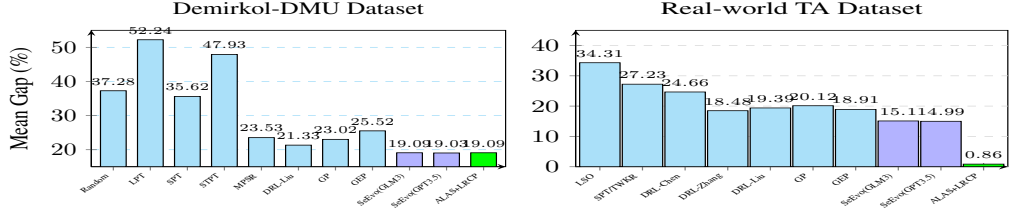
\begin{figure}[t!]
\centering

\begin{minipage}{0.5\textwidth}
\centering
\resizebox{\textwidth}{0.66\height}{ 
\begin{tikzpicture}
    \begin{axis}[
        ybar,
        bar width=0.35cm,           
        width=\textwidth,
        height=4.2cm,
        ylabel={Mean Gap (\%)},
        ymin=15, ymax=55,
        axis x line=bottom,         
        axis y line=left,
        title={Demirkol-DMU Dataset},
        symbolic x coords={Random, LPT, SPT, STPT, MPSR, DRL-Liu, GP, GEP, SeEvo(GLM3), SeEvo(GPT3.5), ALAS+LRCP},
        xtick=data,
        x tick label style={rotate=45, anchor=east, font={\fontsize{4pt}{5pt}\selectfont}},
        nodes near coords,
        grid=both,
        bar shift=0pt,
        minor grid style={draw=none},
        major grid style={dashed, cyan!30},
        ymajorgrids=true,
        xmajorgrids=false,
        ytick={0,10,20,30,40,50},
        axis background/.style={fill=none, draw=black},
        x post scale=0.65,           
        y post scale=0.5,           
        enlarge x limits={abs=0.3cm},
        nodes near coords style={font=\tiny},
        scale=1.05,                 
    ]
        \addplot[ybar, fill=cyan!30, forget plot] coordinates {
            (Random, 37.28)
            (LPT, 52.24)
            (SPT, 35.62)
            (STPT, 47.93)
            (MPSR, 23.53)
            (DRL-Liu, 21.33)
            (GP, 23.02)
            (GEP, 25.52)
        };
        
        \addplot[ybar, fill=blue!30, forget plot] coordinates {
            (SeEvo(GLM3), 19.09)
            (SeEvo(GPT3.5), 19.03)
        };
        
        \addplot[ybar, fill=green, forget plot] coordinates {
            (ALAS+LRCP, 19.09)
        };
    \end{axis}
\end{tikzpicture}
}
\end{minipage}
\hspace{-.1in}
\begin{minipage}{0.48\textwidth}
\centering
\resizebox{\textwidth}{0.66\height}{ 
\begin{tikzpicture}
    \begin{axis}[
        ybar,
        bar width=0.50cm,          
        width=\textwidth,           
        height=4.2cm,                 
        ymin=0, ymax=45,
        axis x line=bottom,         
        axis y line=left,           
        title={Real-world TA Dataset},
        symbolic x coords={LSO, SPT/TWKR, DRL-Chen, DRL-Zhang, DRL-Liu, GP, GEP, SeEvo(GLM3), SeEvo(GPT3.5), ALAS+LRCP},
        xtick=data,
        nodes near coords,          
        x tick label style={rotate=45, anchor=east, font={\fontsize{4pt}{5pt}\selectfont}},
        enlarge x limits={abs=0.3cm},
        bar shift=0pt,
        axis line style={draw},
        grid=both,
        minor grid style={draw=none},
        major grid style={dashed, gray!30},
        ymajorgrids=true,
        xmajorgrids=false,
        ytick={0,10,20,30,40,50},
        axis background/.style={fill=none, draw=black},
        x post scale=0.55,          
        y post scale=0.5,           
        nodes near coords style={font=\tiny},
        scale=1.05,                 
    ]
        \addplot[ybar, fill=cyan!30, forget plot] coordinates {
            (LSO, 34.31)
            (SPT/TWKR, 27.23)
            (DRL-Chen, 24.66)
            (DRL-Zhang, 18.48)
            (DRL-Liu, 19.39)
            (GP, 20.12)
            (GEP, 18.91)
        };
        \addplot[ybar, fill=blue!30, forget plot] coordinates {
            (SeEvo(GLM3), 15.10)
            (SeEvo(GPT3.5), 14.99)
        };
        \addplot[ybar, fill=green, forget plot] coordinates {
            (ALAS+LRCP, 0.86)
        };
    \end{axis}
\end{tikzpicture}
}
\end{minipage}%
\vspace{-.1in}
\caption{Mean Gap to Upper Bound comparison across two benchmark datasets}
\label{fig:ALAS-JSSP-exp1}
\vspace{-.15in}
\end{figure}

%% file: Figure-JSSP-Disruptive-Phase1.tex
\begin{figure}[t!]
\vspace{-.1in}
\begin{subfigure}{0.47\textwidth}
\centering
\begin{tikzpicture}[yscale=0.32, xscale=0.23, font=\tikzfontsize] 
\foreach \y in {0,1,2} \draw (0,-\y) -- (26,-\y);
\node[left] at (0,0) {$M_0$};
\node[left] at (0,-1) {$M_1$};
\node[left] at (0,-2) {$M_2$};
\foreach \x in {0,2,...,24} {
  \draw (\x,0.2) -- (\x,-2.8);
  \node[above] at (\x,0.2) {\x};
}
\filldraw[fill=red!20] (0,0) rectangle (2,-0.8) node[midway] {J2};
\filldraw[fill=orange!20] (2,0) rectangle (3,-0.8) node[midway] {J4};
\filldraw[fill=blue!20] (3,0) rectangle (6,-0.8) node[midway] {J1(1)};
\filldraw[fill=purple!20] (6,0) rectangle (10,-0.8) node[midway] {J5(2)};
\filldraw[fill=green!20] (17,0) rectangle (19,-0.8) node[midway] {J3(3)};
\filldraw[fill=purple!20] (0,-1) rectangle (2,-1.8) node[midway] {J5};
\filldraw[fill=red!20] (3,-1) rectangle (7,-1.8) node[midway] {J2(3)};
\filldraw[fill=orange!20] (7,-1) rectangle (10,-1.8) node[midway] {J4(3)};
\filldraw[fill=green!20] (10,-1) rectangle (14,-1.8) node[midway] {J3(1)};
\filldraw[fill=blue!20] (14,-1) rectangle (16,-1.8) node[midway] {J1(2)};
\filldraw[fill=orange!20] (0,-2) rectangle (2,-2.8) node[midway] {J4};
\filldraw[fill=red!20] (2,-2) rectangle (3,-2.8) node[midway] {J2};
\filldraw[fill=purple!20] (10,-2) rectangle (11,-2.8) node[midway] {J5(3)};
\filldraw[fill=green!20] (14,-2) rectangle (17,-2.8) node[midway] {J3(2)};
\filldraw[fill=blue!20] (17,-2) rectangle (19,-2.8) node[midway] {J1(3)};
\end{tikzpicture}
\vspace{-.15in}
\caption{Static schedule before disruption, makespan = 19}
\label{fig:static-plan}
\end{subfigure}
\hspace{0.01\textwidth}
\begin{subfigure}{0.47\textwidth}
\centering
\begin{tikzpicture}[yscale=0.33, xscale=0.23, font=\tikzfontsize] 
\foreach \y in {0,1,2} {
  \draw (0,-\y) -- (26,-\y);
}
\node[left] at (0,0) {$M_0$};
\node[left] at (0,-1) {$M_1$};
\node[left] at (0,-2) {$M_2$};

\foreach \x in {0,2,...,24} {
  \draw (\x,0.2) -- (\x,-2.8);
  \node[above] at (\x,0.2) {\x};
}

\filldraw[fill=red!20] (0,0) rectangle (2,-0.8) node[midway] {J2};
\filldraw[fill=orange!20] (2,0) rectangle (3,-0.8) node[midway] {J4};
\filldraw[fill=blue!20] (3,0) rectangle (6,-0.8) node[midway] {J1(1)};
\filldraw[fill=purple!20] (6,0) rectangle (10,-0.8) node[midway] {J5};
\filldraw[fill=green!20] (17,0) rectangle (19,-0.8) node[midway] {J3(3)};

\filldraw[fill=purple!20] (0,-1) rectangle (2,-1.8) node[midway] {J5};
\filldraw[fill=red!20] (3,-1) rectangle (5,-1.8) node[midway] {{\color{red}{J2}}};
\filldraw[pattern=north east lines, pattern color=gray] (5,-1) rectangle (8,-1.8);
\node at (6.5,-1.4) {\scriptsize {\color{red}\textbf{down}}};
\filldraw[fill=red!20] (8,-1) rectangle (12,-1.8) node[midway] {J2(3)};
\filldraw[fill=orange!20] (12,-1) rectangle (15,-1.8) node[midway] {J4(3)};
\filldraw[fill=green!20] (15,-1) rectangle (19,-1.8) node[midway] {J3(1)};
\filldraw[fill=blue!20] (19,-1) rectangle (21,-1.8) node[midway] {J1(2)};

\filldraw[fill=orange!20] (0,-2) rectangle (2,-2.8) node[midway] {J4};
\filldraw[fill=red!20] (2,-2) rectangle (3,-2.8) node[midway] {J2};
\filldraw[fill=purple!20] (10,-2) rectangle (11,-2.8) node[midway] {J5(3)};
\filldraw[fill=green!20] (14,-2) rectangle (17,-2.8) node[midway] {J3(2)};
\filldraw[fill=blue!20] (17,-2) rectangle (19,-2.8) node[midway] {J1(3)};
\end{tikzpicture}
\vspace{-.15in}
\caption{Repaire schedule after M1 failure at $t=5$}
\label{fig:maple-repair}
\end{subfigure}

\vspace{0.05in} 

\begin{subfigure}{0.47\textwidth}
\centering
\begin{tikzpicture}[yscale=0.33, xscale=0.23, font=\tikzfontsize] 
\foreach \y in {0,1,2} \draw (0,-\y) -- (26,-\y);
\node[left] at (0,0) {$M_0$};
\node[left] at (0,-1) {$M_1$};
\node[left] at (0,-2) {$M_2$};
\foreach \x in {0,2,...,24} {
  \draw (\x,0.2) -- (\x,-2.8);
  \node[above] at (\x,0.2) {\x};
}
\filldraw[fill=red!20] (0,0) rectangle (2,-0.8) node[midway] {J2};
\filldraw[fill=orange!20] (2,0) rectangle (3,-0.8) node[midway] {J4};
\filldraw[fill=blue!20] (3,0) rectangle (6,-0.8) node[midway] {J1(1)};
\filldraw[fill=purple!20] (6,0) rectangle (10,-0.8) node[midway] {J5(2)};
\filldraw[fill=green!20] (17,0) rectangle (19,-0.8) node[midway] {J3(3)};
\filldraw[pattern=north east lines, pattern color=gray!40] (5,-1) rectangle (8,-1.8);
\filldraw[fill=purple!20] (0,-1) rectangle (2,-1.8) node[midway] {J5};
\filldraw[fill=red!20] (8,-1) rectangle (12,-1.8) node[midway] {J2(3)};
\filldraw[fill=orange!20] (12,-1) rectangle (15,-1.8) node[midway] {J4(3)};
\filldraw[fill=green!20] (15,-1) rectangle (19,-1.8) node[midway] {J3(1)};
\filldraw[fill=blue!20] (19,-1) rectangle (21,-1.8) node[midway] {J1(2)};
\filldraw[fill=orange!20] (0,-2) rectangle (2,-2.8) node[midway] {J4};
\filldraw[fill=red!20] (2,-2) rectangle (3,-2.8) node[midway] {J2};
\filldraw[fill=purple!20] (10,-2) rectangle (11,-2.8) node[midway] {J5(3)};
\filldraw[fill=green!20] (19,-2) rectangle (22,-2.8) node[midway] {J3(2)};
\filldraw[fill=blue!20] (22,-2) rectangle (24,-2.8) node[midway] {J1(3)};
\end{tikzpicture}
\vspace{-.15in}
\caption{Schedule after delay notice for J3(2) on M2}
\label{fig:delay-notification}
\end{subfigure}
\hspace{0.01\textwidth}
\begin{subfigure}{0.47\textwidth}
\centering
\begin{tikzpicture}[yscale=0.33, xscale=0.23, font=\tikzfontsize] 
\foreach \y in {0,1,2} \draw (0,-\y) -- (26,-\y);
\node[left] at (0,0) {$M_0$};
\node[left] at (0,-1) {$M_1$};
\node[left] at (0,-2) {$M_2$};
\foreach \x in {0,2,...,24} {
  \draw (\x,0.2) -- (\x,-2.8);
  \node[above] at (\x,0.2) {\x};
}
\filldraw[fill=red!20] (0,0) rectangle (2,-0.8) node[midway] {J2};
\filldraw[fill=orange!20] (2,0) rectangle (3,-0.8) node[midway] {J4};
\filldraw[fill=blue!20] (3,0) rectangle (6,-0.8) node[midway] {J1(1)};
\filldraw[fill=purple!20] (6,0) rectangle (10,-0.8) node[midway] {J5(2)};
\filldraw[fill=green!20] (22,0) rectangle (26,-0.8) node[midway] {J3(3)};
\filldraw[pattern=north east lines, pattern color=gray] (5,-1) rectangle (8,-1.8);
\filldraw[fill=purple!20] (0,-1) rectangle (2,-1.8) node[midway] {J5};
\filldraw[fill=red!20] (8,-1) rectangle (12,-1.8) node[midway] {J2(3)};
\filldraw[fill=orange!20] (12,-1) rectangle (15,-1.8) node[midway] {J4(3)};
\filldraw[fill=green!20] (15,-1) rectangle (19,-1.8) node[midway] {J3(1)};
\filldraw[fill=blue!20] (19,-1) rectangle (21,-1.8) node[midway] {J1(2)};
\filldraw[fill=orange!20] (0,-2) rectangle (2,-2.8) node[midway] {J4};
\filldraw[fill=red!20] (2,-2) rectangle (3,-2.8) node[midway] {J2};
\filldraw[fill=purple!20] (10,-2) rectangle (11,-2.8) node[midway] {J5(3)};
\filldraw[fill=green!20] (19,-2) rectangle (22,-2.8) node[midway] {J3(2)};
\filldraw[fill=blue!20] (22,-2) rectangle (24,-2.8) node[midway] {J1(3)};
\end{tikzpicture}
\vspace{-.15in}
\caption{Schedule after delay notice for J3(3) on M0}
\label{fig:final-repair}
\end{subfigure}
\caption{LRCP Phase \#1 Local Compensation (makespan = 22):
(a) Static baseline schedule; 
(b) $M_1$ failure between $t = 5$–$8$; 
(c) $M_1$ notifies $M_2$ to delay $J3(2)$; 
(d) $M_2$ informs $M_0$ to push $J3(3)$ back.}
\label{fig:combined-scheduling}
\vspace{-.05in}
\end{figure}
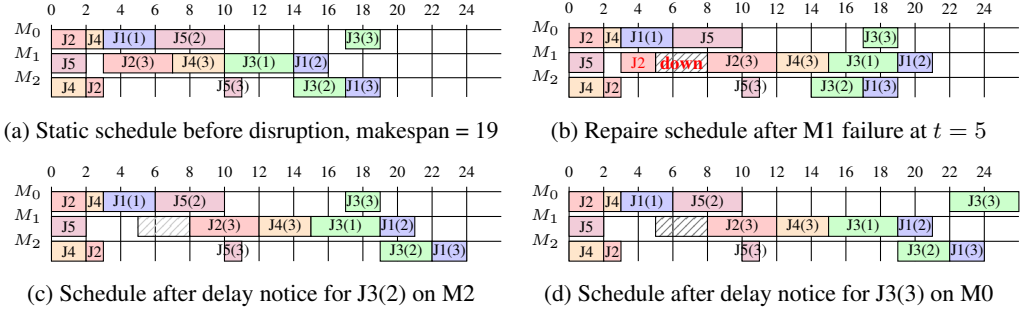

%% file: Figure-JSSP-Disruptive-Phase2.tex
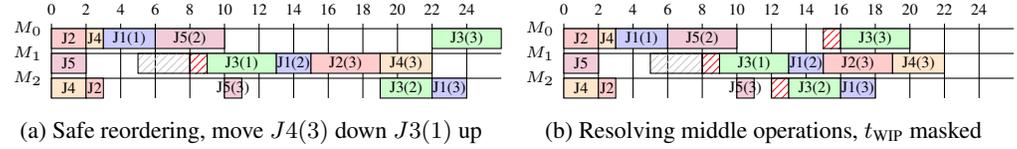
\begin{figure}[t!]
\vspace{-.1in}
\begin{subfigure}{0.47\textwidth}  
\centering
\begin{tikzpicture}[yscale=0.33, xscale=0.23, font=\tikzfontsize]  
\foreach \y in {0,1,2} \draw (0,-\y) -- (26,-\y);
\node[left] at (0,0) {$M_0$};
\node[left] at (0,-1) {$M_1$};
\node[left] at (0,-2) {$M_2$};
\foreach \x in {0,2,...,24} {
  \draw (\x,0.2) -- (\x,-2.8);
  \node[above] at (\x,0.2) {\x};
}
\filldraw[fill=red!20] (0,0) rectangle (2,-0.8) node[midway] {J2};
\filldraw[fill=orange!20] (2,0) rectangle (3,-0.8) node[midway] {J4};
\filldraw[fill=blue!20] (3,0) rectangle (6,-0.8) node[midway] {J1(1)};
\filldraw[fill=purple!20] (6,0) rectangle (10,-0.8) node[midway] {J5(2)};
\filldraw[fill=green!20] (22,0) rectangle (26,-0.8) node[midway] {J3(3)};

\filldraw[pattern=north east lines, pattern color=gray!40] (5,-1) rectangle (8,-1.8);
\filldraw[pattern=north east lines, pattern color=red!90] (8,-1) rectangle (9,-1.8);
\filldraw[fill=purple!20] (0,-1) rectangle (2,-1.8) node[midway] {J5};
\filldraw[fill=green!20] (9,-1) rectangle (13,-1.8) node[midway] {J3(1)};
\filldraw[fill=blue!20] (13,-1) rectangle (15,-1.8) node[midway] {J1(2)};
\filldraw[fill=red!20] (15,-1) rectangle (19,-1.8) node[midway] {J2(3)};
\filldraw[fill=orange!20] (19,-1) rectangle (22,-1.8) node[midway] {J4(3)};
\filldraw[fill=orange!20] (0,-2) rectangle (2,-2.8) node[midway] {J4};
\filldraw[fill=red!20] (2,-2) rectangle (3,-2.8) node[midway] {J2};
\filldraw[fill=purple!20] (10,-2) rectangle (11,-2.8) node[midway] {J5(3)};
\filldraw[fill=green!20] (19,-2) rectangle (22,-2.8) node[midway] {J3(2)};
\filldraw[fill=blue!20] (22,-2) rectangle (24,-2.8) node[midway] {J1(3)};
\end{tikzpicture}
\vspace{-.15in}
\caption{Safe reordering, move $J4(3)$ down $J3(1)$ up}
\label{fig:queue-reorder-1}
\end{subfigure}
\hspace{0.005\textwidth}  
\begin{subfigure}{0.47\textwidth}  
\centering
\begin{tikzpicture}[yscale=0.33, xscale=0.23, font=\tikzfontsize]  
\foreach \y in {0,1,2} \draw (0,-\y) -- (26,-\y);
\node[left] at (0,0) {$M_0$};
\node[left] at (0,-1) {$M_1$};
\node[left] at (0,-2) {$M_2$};
\foreach \x in {0,2,...,24} {
  \draw (\x,0.2) -- (\x,-2.8);
  \node[above] at (\x,0.2) {\x};
}
\filldraw[pattern=north east lines, pattern color=red!90] (15,-0.8) rectangle (16,0);
\filldraw[fill=red!20] (0,0) rectangle (2,-0.8) node[midway] {J2};
\filldraw[fill=orange!20] (2,0) rectangle (3,-0.8) node[midway] {J4};
\filldraw[fill=blue!20] (3,0) rectangle (6,-0.8) node[midway] {J1(1)};
\filldraw[fill=purple!20] (6,0) rectangle (10,-0.8) node[midway] {J5(2)};
\filldraw[fill=green!20] (16,0) rectangle (20,-0.8) node[midway] {J3(3)};

\filldraw[pattern=north east lines, pattern color=gray!40] (5,-1) rectangle (8,-1.8);
\filldraw[pattern=north east lines, pattern color=red!90] (8,-1) rectangle (9,-1.8);
\filldraw[fill=purple!20] (0,-1) rectangle (2,-1.8) node[midway] {J5};
\filldraw[fill=green!20] (9,-1) rectangle (13,-1.8) node[midway] {J3(1)};
\filldraw[fill=blue!20] (13,-1) rectangle (15,-1.8) node[midway] {J1(2)};
\filldraw[fill=red!20] (15,-1) rectangle (19,-1.8) node[midway] {J2(3)};
\filldraw[fill=orange!20] (19,-1) rectangle (22,-1.8) node[midway] {J4(3)};
\filldraw[pattern=north east lines, pattern color=red!90] (12,-2) rectangle (13,-2.8);
\filldraw[fill=orange!20] (0,-2) rectangle (2,-2.8) node[midway] {J4};
\filldraw[fill=red!20] (2,-2) rectangle (3,-2.8) node[midway] {J2};
\filldraw[fill=purple!20] (10,-2) rectangle (11,-2.8) node[midway] {J5(3)};
\filldraw[fill=green!20] (13,-2) rectangle (16,-2.8) node[midway] {J3(2)};
\filldraw[fill=blue!20] (16,-2) rectangle (18,-2.8) node[midway] {J1(3)};
\end{tikzpicture}
\vspace{-.15in}
\caption{Resolving middle operations, $t_{\text{WIP}}$ masked}
\label{fig:queue-reorder-2}
\end{subfigure}
\caption{LRSP Phase \#2 Queue Reordering (makespan = 22): 
(a) Safe moves: moving last operations down, first operations forward with potential penalty;
(b) Resolving remaining operations.}
\label{fig:queue-reordering}
\vspace{-.15in}
\end{figure}

%% file: Conclusion.tex
\vspace{-.1in}
\section{Conclusion}
\label{sec:MAPLE-conclusion}
\vspace{-.1in}
\noindent
We presented $\ALAS$ (Adaptive LLM Agent System), a framework that mitigates core LLM limitations by decomposing planning into memory-augmented agents coordinated through modular templates. Unlike traditional solvers that excel in static settings, $\ALAS$ is designed for dynamic environments where global replanning is costly and disruptive.
By leveraging localized compensation and persistent state tracking, $\ALAS$ employs LRCP to adapt efficiently to runtime disturbances with minimal messaging overhead. Empirical results demonstrate significant improvements over both standalone LLMs and classical
operations research solutions in constraint satisfaction and disruption handling.

{\color{black}\textbf{Limitations.}} This work has three primary limitations. First, deployment and validation in a real-world job-shop setting is deferred to future work. Second, the agent factory assumes reliable code generation from prompts, which may require human oversight in complex or high-risk domains. Third, current $\ALAS$ relies on static statistical models (e.g., constant operation times per machine and $t_{\text{WIP}}$), which would need dynamic adjustment in real-world deployments.
Future work will explore scalable deployment, dynamic cost tuning, and automated agent refinement to broaden $\ALAS$'s applicability in time-sensitive and safety-critical settings.

%% file: Appendix4Section3.tex
\section{Supplemental Information for Section 3}
\label{app:AlasArchitecture}

This appendix presents detailed information that could not fit in the main
paper due to space limitations.
Figure~\ref{fig:ALAS-architecture} depicts the three-layer architecture of $\ALAS$.
Figure~\ref{fig:ALAS-Layer1} presents the three phases of the first layer.
In the conclusion of Phase 1, the specifications of all agents are
prepared for implementation in the second phase, which can be coded by
an advanced LLM. Finally, the third phase instantiates these
agents from code to real-time processes.

For detailed descriptions of each figure, please
refer to the main paper.

\input{Appendix-ALASFullAlgorithm}

\subsection{ALAS Architecture}
\input{Figure-ALAS-Architecture}

\subsection{ALAS Layer \#1 Workflow}
\input{Figure-ALAS-Layer1-Flow-Diagram}

%% file: Appendix-ALASFullAlgorithm.tex
\subsection{Complete Meta-Planner Algorithm for Workflow Generation}
\label{app:ALAS-Algorithm-Full}

\begin{algorithm}[h!]
\caption{Workflow Template $\mathcal{W_\text{template}}$ Generation}
\label{alg:ALAS-Meta-Full}
\begin{small}
\vspace{-.2in}
\begin{multicols}{2}
\setlength{\multicolsep}{2pt}
\begin{algorithmic}[1]
\Require Problem specification $\mathcal{O}$, constraints $D = D_G \cup D_I \cup D_N$, performance metrics $\mathcal{M}$
\Statex \hspace{-1.9em} \textbf{Local Variables:}
\State Roles $\mathcal{R}$; Nodes $\mathcal{N}$; Edges $\mathcal{E}$
\State Log schemas $\mathcal{L}_{n_i}, \mathcal{L}_{e_{ij}}$
\State Agents, Comp Agents $\alpha_{n_i}, \alpha_{e_{ij}}$, $\alpha^{comp}_{n_i}, \alpha^{comp}_{e_{ij}}$
\Ensure Validated $\mathcal{W_\text{template}} = (\mathcal{N}, \mathcal{E})$
\Statex

\Statex \hspace{-1.5em} \textbf{Phase 1: Network Construction}(Sec.~\ref{sec:ALAS-Layer1})
\State $\mathcal{R} \gets \text{ExtractRoles}(\mathcal{O})$
\State $\{(n_i, \mathcal{P}_i)\} \gets \text{map}_{\text{role}}(\mathcal{O}, \mathcal{R})$ 
\State $\mathcal{N} \gets \{n_i\}$, $\mathcal{E} \gets \text{map}_{\text{dep}}(\mathcal{N}, D)$ 
\State $\mathcal{W_\text{template}} \gets (\mathcal{N}, \mathcal{E})$
\Statex

\Statex \hspace{-1.5em} \textbf{Phase 2: Agent Specification} (Sec.~\ref{sec:ALAS-Layer2})
\ForAll{$n_i \in \mathcal{N}$}
  \State $\mathcal{L}_{n_i} \gets \text{DefineLogSchema}(n_i, \mathcal{P}_{n_i})$
  \State $\alpha_{n_i} \gets \text{DefineNodeAgent}(n_i, \mathcal{L}_{n_i})$
  \State $\alpha^{comp}_{n_i} \gets \text{DefineCompAgent}(\alpha_{n_i}, \mathcal{L}_{n_i})$
\EndFor
\ForAll{$e_{ij} \in \mathcal{E}$}
  \State $\mathcal{L}_{e_{ij}} \gets \text{DefineLogSchema}(e_{ij}, \mathcal{P}_{e_{ij}})$
  \State $\alpha_{e_{ij}} \gets \text{DefineEdgeAgent}(e_{ij}, \mathcal{L}_{e_{ij}})$
  \State $\alpha^{comp}_{e_{ij}} \gets \text{DefineCompAgent}(\alpha_{e_{ij}}, \mathcal{L}_{e_{ij}})$
\EndFor
\Statex

\Statex \hspace{-1.5em} \textbf{Phase 3: Validation and Refinement} (Sec.~\ref{sec:ALAS-Layer3})
\State $\mathcal{W_\text{template}} \gets \text{UpdateWorkflow}(\mathcal{N}, \mathcal{E}, \alpha, \alpha^{comp})$
\While{not $\text{ValidateWorkflow}(\mathcal{W_\text{template}}, \mathcal{M})$}
  \State $\text{StructuralValidation}(\mathcal{W_\text{template}})$
  \State $\text{ConstraintValidation}(\mathcal{W_\text{template}}, D)$
  \State $\text{CompensationValidation}(\mathcal{W_\text{template}}), \{\alpha^{comp}\})$
  \State $\mathcal{W_\text{template}} \gets \text{RefineWorkflow}(\mathcal{W_\text{template}}, \mathcal{M})$
\EndWhile
\State \Return $\mathcal{W_\text{template}}$
\end{algorithmic}
\end{multicols}
\end{small}
\vspace{-0.15in}
\end{algorithm}

%% file: Figure-ALAS-Architecture.tex
\vspace{-.1in}
\begin{figure}[h!]
\centering
\begin{tikzpicture}[
  font=\normalsize,
  every node/.style={align=center, rounded corners=2pt, scale=0.8, every node/.style={scale=0.8, font=\normalsize}},
  layer/.style={draw, thick, fill=blue!10, minimum width=2.2cm, minimum height=0.8cm},
  mem/.style={draw, thick, fill=yellow!25, minimum width=2.0cm, minimum height=1.6cm},
  arrow/.style={->, thick}
]


\node[layer] (layer1) at (-3.4, 2.6) {\large \textbf{Template Layer} \\ \normalsize $\mathcal{T}$: Roles \& Constraints};
\node[layer] (layer2) at (-3.4, 1.6) {\large \textbf{Factory Layer} \\ \normalsize $\mathcal{F}$: Agent Instantiation};
\node[layer] (layer3) at (-3.4, 0.6) {\large \textbf{Runtime Layer} \\ \normalsize $\mathcal{R}$: Execution \& Repair};

\node[mem] (memory) at (2.4, 2.0) {\large \textbf{Persistent Memory} \\ \normalsize Logs, State, Rollback, Validation};

\draw[arrow] (layer1.south) -- (layer2.north);
\draw[arrow] (layer2.south) -- (layer3.north);

\draw[->, >=latex, thick] (layer1.east) -- (memory.west);
\draw[->, >=latex, thick] (layer2.east) -- (memory.west);
\draw[<->, >=latex, thick] (layer3.east) -- (memory.west);

\node at (-3.4, 3.3) {\large \textbf{Problem Specification}};

\draw[arrow] (-3.4,3.1) -- (layer1.north);

\end{tikzpicture}
\vspace{.1in}
\caption{ALAS architecture: a lightweight LLM-driven planner with layered decomposition. Persistent memory supports all layers by storing state, validating constraints, and enabling recovery.}
\label{fig:ALAS-architecture}
\end{figure}
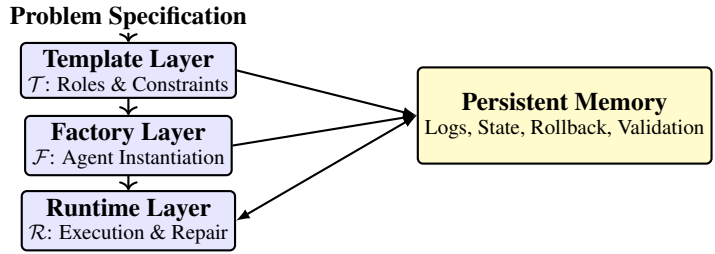

%% file: Figure-ALAS-Layer1-Flow-Diagram.tex
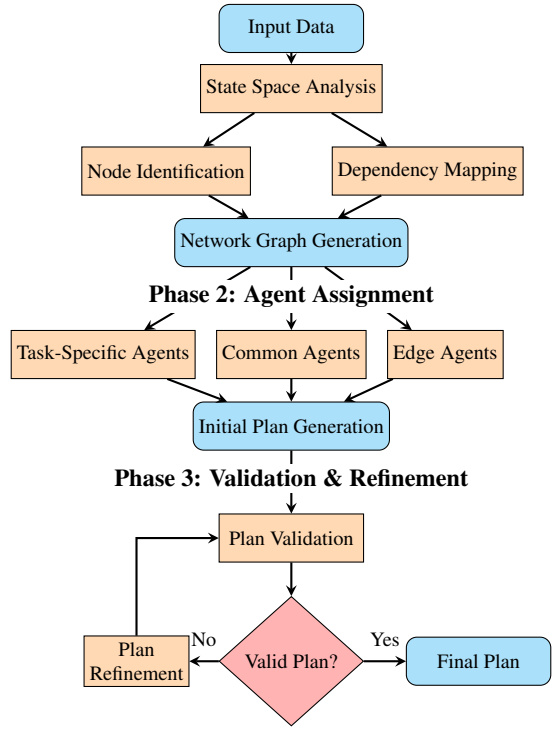
\begin{figure}[ht!]
    \centering
    \begin{tikzpicture}[node distance=1.25cm, transform shape, scale=0.8, every node/.style={scale=0.8, font=\large}] 
        \node (input) [startstop] {Input Data};
        \node (state) [process, below of=input] {State Space Analysis};
        \node (node) [process, below left=.56cm and -0.84cm of state] {Node Identification};
        \node (dep) [process, below right=.56cm and -0.84cm of state] {Dependency Mapping};
        \node (graph) [startstop, below=1.75cm of state] {Network Graph Generation};
        \node (task) [miniprocess, below left=1.05cm and -.35cm of graph] {Task-Specific Agents};
        \node (common) [miniprocess, below=1.05cm of graph] {Common Agents};
        \node (edge) [miniprocess, below right=1.05cm and -.35cm of graph] {Edge Agents};
        \node (plan) [startstop, below=2.24cm of graph] {Initial Plan Generation};
        \node (validate) [process, below=1.05cm of plan] {Plan Validation};
        \node (decision) [decision, below=0.56cm of validate] {Valid Plan?};
        \node (refine) [process, left=0.49cm of decision, minimum width=1.26cm, minimum height=1.05cm, align=center] {Plan \\ Refinement};
        \node (final) [startstop, right=0.7cm of decision] {Final Plan};
        \draw [arrow] (input) -- (state);
        \draw [arrow] (state) -- (node);
        \draw [arrow] (state) -- (dep);
        \draw [arrow] (node) -- (graph);
        \draw [arrow] (dep) -- (graph);
        \draw [arrow] (graph) -- (task);
        \draw [arrow] (graph) -- (common);
        \draw [arrow] (graph) -- (edge);
        \draw [arrow] (task) -- (plan);
        \draw [arrow] (common) -- (plan);
        \draw [arrow] (edge) -- (plan);
        \draw [arrow] (plan) -- (validate);
        \draw [arrow] (validate) -- (decision);
        \draw [arrow] (decision) -- node[anchor=west, above=0.14cm] {No} (refine);
        \draw [arrow] (refine) |- (validate);
        \draw [arrow] (decision) -- node[anchor=south, above=0.14cm] {Yes} (final);
        \node [above=0.28cm of input, font=\bfseries\Large] {Phase 1: Workflow Template Construction};
        \node [above=0.28cm of common, fill=white, font=\bfseries\Large] {Phase 2: Agent Assignment};
        \node [above=0.35cm of validate, fill=white, font=\bfseries\Large] {Phase 3: Validation \& Refinement};
    \end{tikzpicture}
\caption{$\ALAS$ Planning Layer \#1 Architecture. \text{Color Scheme:} \textcolor{cyan!50}{cyan} for input/output/intermediate results, \textcolor{orange!50}{orange} for processes, and \textcolor{red!50}{red} for decisions. This figure illustrates how the meta-planner generates a planning workflow template $\mathcal{W_{\text{template}}}$.}
    \label{fig:ALAS-Layer1}
\end{figure}

%% file: AppendixAgentFactory.tex
\subsection{Agent Factory Implementation Details}
\label{app:ALAS-Layer2}

This appendix provides detailed information on the Agent Factory component of the ALAS architecture, expanding on the summary provided in the main paper.

\subsubsection{Agent Factory Overview}

The Agent Factory translates formal agent specifications from the meta-planner into executable implementations. It employs a two-stage approach: first attempting to discover existing implementations, and then generating custom implementations when necessary.

\subsubsection{Agent Discovery Process}

The discovery process systematically searches for existing agent implementations that match specifications from phase \#1. For each agent specification $\alpha_i$, the discovery mechanism:

\begin{enumerate}[leftmargin=1.2em, topsep=0pt, itemsep=0pt, label=\arabic*.]
    \item Extracts the capability profile $\mathbf{c}_i$ and constructs a search query to identify potential implementations
    \item Retrieves candidate implementations from:
\begin{itemize}[leftmargin=1.0em, topsep=-.0em, parsep=-.0em, label=-]        \item Public agent repositories (e.g., GitHub, HuggingFace)
        \item API directories and marketplaces
        \item Pre-validated component libraries
        \item Domain-specific collections
    \end{itemize}
    \item Evaluates candidate suitability using multiple criteria:
\begin{itemize}[leftmargin=1.0em, topsep=-.0em, parsep=-.0em, label=-]
        \item Capability matching: Verifies that all required capabilities in $\mathbf{c}_i$ are supported
        \item Protocol compatibility: Ensures compatibility with the specified protocol buffer $\beta_i$
        \item Efficiency compliance: Validates performance against efficiency requirements $e_i$
        \item Context sizing: Confirms the implementation can operate within context window $w_i$
        \item Logging support: Verifies support for the required logging schema $\mathcal{L}_i$
    \end{itemize}
    \item Ranks candidates using a weighted scoring function $S(\alpha_i, I_j)$ where $I_j$ is a candidate: 
    \begin{equation}
        S(\alpha_i, I_j) = \sum_{k} w_k \cdot f_k(\alpha_i, I_j)
    \end{equation}
    where $w_k$ is the weight assigned to criterion $k$, and $f_k$ is an evaluation function for that criterion.
\end{enumerate}

When a suitable implementation is identified, it undergoes verification testing to confirm operational compatibility with the workflow requirements. Upon successful verification, the implementation is registered in the agent repository with appropriate metadata linking it to the specification.

The discovery mechanism employs both exact and approximate matching techniques. Exact matching requires all specification parameters to be satisfied precisely, while approximate matching allows for partial capability matching when accompanied by adaptation mechanisms.

\subsubsection{Agent Coding Mechanism}

When discovery fails to locate suitable implementations, the Agent Factory switches to its coding mechanism, which uses LLMs to generate custom implementations. The coding process follows a structured methodology:

\begin{enumerate}[leftmargin=1.2em, topsep=0pt, itemsep=0pt, label=\arabic*.]
    \item \textbf{Specification Translation}: The formal agent specification is translated into a natural language implementation brief that serves as the prompt for the LLM. This translation preserves all critical requirements while expressing them in a form that maximizes LLM comprehension.
    
    \item \textbf{LLM Selection}: An appropriate LLM is selected based on:
\begin{itemize}[leftmargin=1.0em, topsep=-.0em, parsep=-.0em, label=-]
        \item Domain expertise matching capability requirements in $\mathbf{c}_i$
        \item Demonstrated proficiency in generating the required implementation type
        \item Context window compatibility with the complexity of the specification
        \item Robustness against hallucination for critical components
    \end{itemize}
    
    \item \textbf{Implementation Generation}: The selected LLM generates implementation code with:
\begin{itemize}[leftmargin=1.0em, topsep=-.0em, parsep=-.0em, label=-]
        \item Embedded logging that conforms to schema $\mathcal{L}_i$
        \item Protocol handling for buffer $\beta_i$
        \item Optimizations for efficiency parameters $e_i$
        \item Adaptation to context window constraints $w_i$
    \end{itemize}
    
    \item \textbf{Implementation Validation}: The generated implementation undergoes validation to ensure:
\begin{itemize}[leftmargin=1.0em, topsep=-.0em, parsep=-.0em, label=-]
        \item Functional correctness against specification requirements
        \item Proper integration with the compensation mechanisms defined in $\alpha_i^{comp}$
        \item Robustness against edge cases and exceptional conditions
        \item Compliance with system-wide constraints and protocols
    \end{itemize}
\end{enumerate}

For particularly complex agents, the coding process may employ a multi-stage approach where the implementation is generated iteratively, with each iteration refining the previous version based on validation feedback.

\subsubsection{Compensation Agent Generation}

Special attention is given to the generation of compensation agents, which require precise understanding of the primary agent's operations to ensure proper reversal or mitigation. The generation of compensation agents follows these additional steps:

\begin{enumerate}[leftmargin=1.2em, topsep=0pt, itemsep=0pt, label=\arabic*.]
    \item Extraction of the primary agent's state-modifying operations
    \item Analysis of operation dependencies and sequencing constraints
    \item Determination of appropriate compensation strategies (e.g., undo, retry, escalate)
    \item Generation of the recovery sequence $\Gamma_i$ that defines the steps for returning to a consistent state
\end{enumerate}

The Factory ensures that compensation agents maintain strict operational correspondence with their primary counterparts, guaranteeing every state-modifying operation has a reversal mechanism defined.

\subsubsection{Deployment Artifact Production}

The output of the Agent Factory is a deployable artifact that encapsulates the agent's logic and interaction patterns. These artifacts take several forms, depending on the agent type and implementation approach:

\begin{itemize}[leftmargin=1.0em, topsep=-.0em, parsep=-.0em, label=-]
    \item \textbf{Code Snippets}: Executable code implementing the agent's functionality, typically for computationally intensive or specialized tasks
    \item \textbf{Prompt Templates}: Structured prompts that guide LLMs to implement the specified behavior at runtime, used for cognitively complex or reasoning-intensive tasks
    \item \textbf{API Configurations}: Parameter sets and endpoint specifications for interacting with external services or pre-existing agents
    \item \textbf{Hybrid Implementations}: Combined approaches that leverage both code and LLM prompting for different aspects of the agent's functionality
\end{itemize}

Each artifact is accompanied by metadata that defines its:

\begin{itemize}[leftmargin=1.0em, topsep=-.0em, parsep=-.0em, label=-]
    \item Execution requirements (e.g., runtime environment, dependencies)
    \item Interface specifications for input/output handling
    \item State persistence requirements and mechanisms
    \item Monitoring hooks for runtime observation
    \item Recovery points for compensation handling
\end{itemize}

\subsubsection{Factory Design Pattern Implementation}

The Agent Factory implements the classic Factory design pattern, providing a standardized interface for agent instantiation while encapsulating the complexity of implementation selection, generation, and validation. This pattern enables:

\begin{itemize}[leftmargin=1.0em, topsep=-.0em, parsep=-.0em, label=-]
    \item Decoupling of agent specifications from implementation details
    \item Support for heterogeneous implementation technologies
    \item Runtime substitution of agents when needed for recovery or optimization
    \item Maintenance of a growing repository of reusable components
\end{itemize}

The Factory pattern allows the ALAS system to evolve its agent implementation strategies over time without requiring changes to the meta-planning or runtime components, creating a flexible architecture that can adapt to new implementation technologies and approaches.

\subsubsection{Implementation Efficiency Considerations}

To maximize system efficiency, the Agent Factory implements several optimization strategies:

\begin{enumerate}[leftmargin=1.2em, topsep=0pt, itemsep=0pt, label=\arabic*.]
    \item \textbf{Caching}: Previously generated implementations are cached and indexed by their specifications to avoid redundant generation
    \item \textbf{Component Reuse}: Complex implementations are decomposed into reusable components that can be shared across multiple agents
    \item \textbf{Incremental Refinement}: When similar agents have been previously implemented, the Factory uses delta-based generation to create variants rather than generating entirely new implementations
    \item \textbf{Resource Scaling}: Implementation generation resources are allocated proportionally to the complexity and criticality of the agent
\end{enumerate}

These optimizations significantly reduce the computational overhead of agent generation, particularly in scenarios where multiple similar agents are required or when the system executes recurring workflow patterns.

\subsubsection{Theoretical Foundations}

The Agent Factory design is grounded in several theoretical frameworks:

\begin{itemize}[leftmargin=0.8em, topsep=-.0em, parsep=-.0em, label=-]
    \item \textit{Program Synthesis}: Formal methods for generating programs from specifications
    \item \textit{Component-Based Software Engineering}: Principles of component composition and reuse
    \item \textit{LLM Prompt Engineering}: Techniques for directing LLM behavior through structured prompts
    \item \textit{Agent-Oriented Software Engineering}: Methodologies for developing autonomous software agents
\end{itemize}

These foundations provide a rigorous basis for the Factory's approach to transforming abstract agent specifications into concrete, executable implementations.

%% file: AppendixLCSR.tex
\section{LCSR Specifications and Lemma Proofs}
\label{app:lcsr}

\begin{algorithm}[H]
\caption{$\ALAS$ Reactive Planning: LCSR with Cascading Repair and Queue Reordering}
\label{alg:ALAS-LCSR-Full}
\begin{footnotesize}
\begin{algorithmic}[1]
\Require Global tracker $\mathcal{T}$; queues $\{Q_k\}$; disruption time $t_d$; machine $M_k$ down for duration $\delta t$
\Ensure Updated execution tracker $\mathcal{T}$ and queues $\{Q_k\}$

\vspace{0.2em}
\Statex \textit{Phase I: Status Update}
\For{each $(J_i, M_k, s_{ij}, e_{ij}, j) \in \mathcal{T}$}
  \If{$e_{ij} < t_d$}
    \State $\ell_{ij} \gets 2$ \Comment{Completed}
  \ElsIf{$s_{ij} \le t_d < e_{ij}$}
    \State $\ell_{ij} \gets 1$ \Comment{In progress (WIP)}
    \State Reschedule $O_{ij}$ to $t_d + \delta t$
  \Else
    \State $\ell_{ij} \gets 0$ \Comment{Waiting}
  \EndIf
\EndFor

\vspace{0.2em}
\Statex \textit{Phase II: Delay Propagation}
\For{each rescheduled $O_{ij}$ on $M_k$}
  \State Update $s_{ij}, e_{ij}$ in $Q_k$ and $\mathcal{T}$
  \State Send \texttt{DELAY\_NOTIFY}$(J_i, j, e_{ij})$ to the agent of $M_{k'}$
\EndFor

\vspace{0.2em}
\Statex \textit{Phase III: Local Queue Optimization}
\For{each pair $(O_a, O_b)$ in $Q_k$, where $a < b$}
  \If{$O_b$ is the final operation of its job $J_i$}
    \State $\Delta t \gets \text{EvalSwap}(O_a, O_b)$ \Comment{Makespan impact}
    \If{$\Delta t < 0$}
      \State Swap $O_a \leftrightarrow O_b$ in $Q_k$
      \State Update $s$, $e$ in $\mathcal{T}$
    \EndIf
  \Else
    \State Skip (non-terminal operations are less flexible)
  \EndIf
\EndFor

\vspace{0.2em}
\Statex \textit{Phase IV: Cascading Delay Handling}
\While{message queue not empty}
  \State Receive \texttt{DELAY\_NOTIFY}$(J_i, j, e_{ij})$
  \State Let $O_{i(j+1)} = (M_{k'}, s, e, d)$
  \If{$s < e_{ij}$}
    \State $s \gets e_{ij}$; $e \gets s + d$
    \State Update $Q_{k'}$ and $\mathcal{T}$
    \If{$\ell_{i(j+1)} \ne 2$}
      \State Send \texttt{DELAY\_NOTIFY}$(J_i, j+1, e)$
    \EndIf
  \EndIf
\EndWhile

\vspace{0.2em}
\State \Return Updated tracker $\mathcal{T}$ and queues $\{Q_k\}$
\end{algorithmic}
\end{footnotesize}
\end{algorithm}

\begin{lemma}[Generalized LCSR Complexity]
\mbox{}\\[0.5ex]
For a system with:
\begin{itemize}[leftmargin=1.2em, topsep=-.15pt, itemsep=-.15pt]
    \item $J$ jobs
    \item $M$ machines 
    \item At most $O_{\max}$ operations per job
    \item $S$ average swap evaluations per queue ($1 \leq S \leq J$)
\end{itemize}
The worst-case time complexity is:
\begin{equation}
\mathcal{O}\left(\frac{J^2 O_{\max}^2}{M} + J M O_{\max}\right)
\end{equation}
\end{lemma}

\begin{proof}
The complexity derives from four components:

\begin{enumerate}[leftmargin=1.5em, topsep=-.15pt, itemsep=-.15pt]
    \item \textit{Status Update}: $\mathcal{O}(J O_{\max})$ \\
    Must check all operations of all jobs
    
    \item \textit{Delay Propagation}: $\mathcal{O}(J O_{\max})$ \\
    Each job's operation chain may have $O_{\max}$ elements
    
    \item \textit{Queue Optimization}:
    \begin{itemize}[leftmargin=1.2em, topsep=-.15pt, itemsep=-.15pt]
        \item Full analysis: $\mathcal{O}\left(\frac{J^2 O_{\max}^2}{M}\right)$ \\
        All operation pairs on all machines
        \item Practical bound: $\mathcal{O}(S J O_{\max})$ \\
        When swaps are limited to $S$ evaluations
    \end{itemize}
    
    \item \textit{Cascading Delay}: $\mathcal{O}(J M O_{\max})$ \\
    Worst-case propagation through all machines
\end{enumerate}

The dominant terms combine to give the final complexity:
\[
\underbrace{\frac{J^2 O_{\max}^2}{M}}_{\text{queue optimization}} + \underbrace{J M O_{\max}}_{\text{cascading delays}}
\]
\end{proof}

\begin{corollary}[Special Cases]
\begin{itemize}[leftmargin=1.2em, topsep=-.15pt, itemsep=-.15pt]
    \item \textit{Single-operation jobs} ($O_{\max}=1$): 
    $\mathcal{O}(J^2/M + JM)$
    
    \item \textit{Fully parallel systems} ($M \approx J$): 
    $\mathcal{O}(J O_{\max}^2 + J^2 O_{\max})$
    
    \item \textit{Swap-limited implementations}: 
    $\mathcal{O}(S J O_{\max} + J M O_{\max})$
\end{itemize}
\end{corollary}

\vspace{0.5em}
\textit{Key Observations}:
\begin{itemize}[leftmargin=1.2em, topsep=-.15pt, itemsep=-.15pt]
    \item Complexity is quadratic in job count and operations
    \item Machine count appears both in numerator (delays) and denominator (parallelization)
    \item Practical implementations can achieve better bounds through swap heuristics
\end{itemize}

%% file: AppendixURS.tex
\section{Urban Ride Assignment Problem}
\label{app:ALAS-URAP}

The goal is to optimally assign ride requests to a fleet of autonomous or human-driven vehicles in a city, while satisfying various constraints and objectives. The key elements are the following.

\begin{itemize}[leftmargin=1.0em, topsep=-.0em, parsep=-.0em, label=*]
\item \textbf{City Map:} A graph $G = (V, E)$ where $V$ is the set of locations and $E$ is the set of roads connecting them, with associated distances and travel times.
\item \textbf{Ride Requests:} A set of requests $R$, where each request $r_i \in R$ is characterized by:
\begin{itemize}[leftmargin=1.0em, topsep=-.0em, parsep=-.0em, label=-]
    \item Passenger ID $p_i$
    \item Pickup location $v_{p_i} \in V$ 
    \item Drop-off location $v_{d_i} \in V$
    \item Desired pickup time window $[t_{p_i}^{min}, t_{p_i}^{max}]$
    \item Desired drop-off time window $[t_{d_i}^{min}, t_{d_i}^{max}]$
\end{itemize}

\item \textbf{Vehicles:} A set of vehicles $K$, where each vehicle $k_j \in K$ has:
\begin{itemize}[leftmargin=1.0em, topsep=-.0em, parsep=-.0em, label=-]
    \item Vehicle ID $k_j$ 
    \item Current location $v_{k_j} \in V$
    \item Battery/fuel level $b_{k_j} \in [0, 1]$
    \item Passenger capacity $c_{k_j} \in \mathbb{Z}^+$
    \item Speed $s_{k_j} \in \mathbb{R}^+$
\end{itemize}

\end{itemize}

\subsection{A Simplified URS Problem Statement}

Table~\Ref{fig:appDURSSpec} in the main text depicts a URS problem
with three drivers and four passengers.
Using this problem, we walk through how $\ALAS$ works.

\subsection{Generating Planner W* Walkthrough}

Given the problem statement of URS, $\ALAS$ generates a planning template $\mathcal{W_\text{template}}$.

\begin{table*}[ht!]
\caption{Agent Specifications and Protocols}
\begin{footnotesize}
\begin{tabular}{|p{1.8cm}|p{3.4cm}|p{3.3cm}|p{3.5cm}|}
\hline
\textbf{Agent Type} & \textbf{Input Protocol} & \textbf{Output Protocol} & \textbf{Key Functions} \\
\toprule
\hline
\multicolumn{4}{|c|}{\textbf{Task-Specific Agents}} \\
\hline
Route \newline Planning & 
- Location map $G(V,E)$ \newline
- Travel times matrix \newline
- Vehicle positions & 
- Optimized routes \newline
- Distance calculations \newline
- Path sequences & 
- Path optimization \newline
- Distance minimization \newline
- Route feasibility checks \\
\hline
Scheduling & 
- Required arrival times \newline
- Travel duration estimates \newline
- Vehicle availability & 
- Pickup schedule \newline
- Timing constraints \newline
- Buffer allocations & 
- Schedule generation \newline
- Timing verification \newline
- Buffer management \\
\hline
Capacity \newline Management & 
- Passenger requests \newline
- Vehicle capacities \newline
- Route timing & 
- Passenger groupings \newline
- Vehicle  \newline
- Capacity utilization & 
- Group optimization \newline
- Capacity verification \newline
- Load balancing \\
\hline
\midrule
\multicolumn{4}{|c|}{\textbf{Common Agents}} \\
\hline
Temporal \newline Constraint & 
- Schedule requirements \newline
- Time windows \newline
- Buffer needs & 
- Timing validations \newline
- Constraint satisfaction \newline
- Buffer adequacy & 
- Time verification \newline
- Constraint checking \newline
- Buffer analysis \\
\hline
Resource \newline Allocation & 
- Vehicle inventory \newline
- Request demands \newline
- Location data & 
- Resource assignments \newline
- Utilization plans \newline
- Coverage maps & 
- Resource optimization \newline
- Coverage verification \newline
- Efficiency analysis \\
\hline
Distance \newline Optimization & 
- Route options \newline
- Distance matrix \newline
- Time constraints & 
- Optimized paths \newline
- Distance metrics \newline
- Efficiency scores & 
- Path optimization \newline
- Distance reduction \newline
- Efficiency maximization \\
\hline
\midrule
\multicolumn{4}{|c|}{\textbf{Validation Agents}} \\
\hline
Plan Validator & 
- Complete plan \newline
- System constraints \newline
- Quality metrics & 
- Validation results \newline
- Constraint checks \newline
- Performance scores & 
- Plan verification \newline
- Constraint validation \newline
- Quality assessment \\
\hline
Refinement Agent & 
- Validation results \newline
- Improvement options \newline
- Performance metrics & 
- Refinement suggestions \newline
- Update priorities \newline
- Optimization paths & 
- Plan improvement \newline
- Update sequencing \newline
- Performance optimization \\
\hline
\end{tabular}
\end{footnotesize}
\label{tab:appAgentSpecs}
\end{table*}

\subsubsection{State-Space Analysis}

Our Urban Ride-Sharing (URS) problem presents a complex transportation scheduling challenge that we must first understand through systematic state-space analysis. The system involves seven locations (A through G), where G represents Boston Logan Airport, with urban locations forming a mesh network of 10km distances and airport routes ranging from 31-36km. Four passengers require airport transportation with specific arrival deadlines, while three vehicles, each capable of carrying two passengers, must be coordinated to meet these demands efficiently.

Each dimension of our state space reveals crucial aspects of the planning challenge. In the \emph{Who} dimension, we track four passenger requests ($r_1$ through $r_4$) and three vehicles ($k_1$ through $k_3$). These passengers require arrivals at BOS between 08:45 and 09:00, with each vehicle qualified for airport routes and positioned initially at locations A, C, and E. 

The \emph{Where} dimension maps our network topology, distinguishing between urban segments with uniform 10km distances and airport routes varying from 31-36km. This spatial arrangement, combined with the \emph{When} dimension's speed constraints (60km/h urban, 100km/h airport routes), creates our fundamental timing framework. Simple calculations reveal urban segments require 10 minutes of travel time, while airport routes need 19-22 minutes depending on origin.

Our \emph{What} dimension monitors vehicle resources throughout plan execution, ensuring we respect the two-passenger capacity limit while maximizing sharing opportunities. The \emph{Why} dimension establishes our optimization objectives: ensuring on-time airport arrivals while minimizing total distance traveled. The \emph{How} dimension defines our execution protocols, including pickup sequencing and route navigation strategies.

\subsubsection{Phase 1: Network Construction}

Building upon our state-space analysis, we construct our planning network by first identifying critical nodes and dependencies. Our node set $\mathcal{N}$ comprises:

Passenger Nodes: Each request $r_i$ becomes a node with attributes:
- $r_1$: Location A, BOS arrival 08:45
- $r_2$: Location B, BOS arrival 08:50
- $r_3$: Location C, BOS arrival 08:55
- $r_4$: Location D, BOS arrival 09:00

Vehicle Nodes: Each vehicle $k_i$ forms a node with position and capacity:
- $k_1$: Starting at A, capacity 2
- $k_2$: Starting at C, capacity 2
- $k_3$: Starting at E, capacity 2

Location Nodes: Each physical location becomes a node with attributes including distance to other locations and travel time calculations.

Our dependency set $\mathcal{E}$ captures relationships between these nodes through several categories:

Temporal Dependencies: We establish feasible pickup windows by working backward from required arrival times. For example, $r_1$ requires 22 minutes for the airport route plus 10 minutes for each urban segment traversed, creating timing constraints for vehicle assignment.

Spatial Dependencies: We map possible routes between nodes, considering both direct airport routes and potential shared-ride combinations through urban segments.

Capacity Dependencies: We create edges representing feasible passenger groupings within vehicle capacity limits.

\subsubsection{Phase 2: Agent Assignment}

With our network structure defined, we assign specialized agents to manage different aspects of the solution:

Task-Specific Agents:
The Route Planning Agent optimizes paths using the distance matrix and travel speeds, calculating optimal routes for both single and shared rides. The Scheduling Agent determines precise pickup times, working backward from airport deadlines and incorporating travel time calculations. The Capacity Management Agent identifies feasible passenger groupings based on timing and location proximity.

Common Agents:
The Temporal Constraint Agent ensures all timing requirements are met, maintaining a master schedule that accounts for all dependencies. The Resource Allocation Agent assigns vehicles to routes, optimizing the distribution of available capacity. The Distance Optimization Agent works to minimize total travel distance while respecting all constraints.

Edge Agents:
These agents manage the relationships between different aspects of the plan. For example, the Passenger Grouping Agent evaluates potential shared rides by analyzing proximity of pickup locations and compatibility of arrival times.

\subsubsection{Phase 3: Validation and Refinement}

In our final phase, we implement a comprehensive validation and refinement process:

Initial Validation:
We verify temporal feasibility by checking that all calculated pickup times allow sufficient travel time to meet airport deadlines. We confirm capacity constraints are respected throughout all vehicle routes. We validate that all passengers are served and all required resources are properly allocated.

Iterative Refinement:
We identify optimization opportunities, such as grouping passengers with compatible timing and locations. For example, passengers $r_2$ and $r_3$ might share a ride if their pickup locations are close and arrival times are within 5 minutes. We adjust vehicle assignments to minimize empty travel distance while maintaining service guarantees.

Final Plan Generation:
The resulting plan specifies exact pickup times, vehicle assignments, and routes, with built-in buffers for potential delays. The plan includes contingency protocols for common disruptions such as traffic delays or passenger late arrivals.

This systematic approach ensures we generate a robust, efficient solution to our URS problem while maintaining clear documentation of our planning process and decisions.

\subsubsection{Output}

Table~\ref{tab:appAgentSpecs} the list of required agents and
their functional specifications and protocols.

\begin{table*}[ht!]
    \centering
    \caption{Agent Placement in the Urban Ride Sharing Network}
        \begin{footnotesize}
        \begin{tabular}{|l|l|p{8cm}|}
        \hline
        \textbf{Location} & \textbf{Type} & \textbf{Agents and Their Responsibilities} \\
        \toprule \hline
        A--F & Nodes & \textbf{Resource Allocation Agent}: Manages vehicle assignments and passenger pickups at urban locations \\
        \hline
        G (Airport) & Node & 
        \begin{minipage}[t]{8cm}
        \textbf{Plan Validator Agent}: Verifies arrival times and  plan feasibility\\
        \textbf{Temporal Constraint Agent}: Ensures all arrival deadlines met
        \end{minipage} \\
        \hline
        A--F edges & Urban Routes & 
        \begin{minipage}[t]{8cm}
        \textbf{Route Planning Agent}: Optimizes urban route segments (10 min travel time)\\
        \textbf{Scheduling Agent}: Coordinates pickup sequences and timing
        \end{minipage} \\
        \hline
        (A,...,F)--G & Airport Routes & 
        \begin{minipage}[t]{8cm}
        \textbf{Capacity Management Agent}: Ensures vehicle capacity constraints during airport trips\\
        \textbf{Distance Optimization Agent}: Minimizes total travel distance
        \end{minipage} \\
        \hline
        Network-wide & Global & 
        \begin{minipage}[t]{8cm}
        \textbf{Refinement Agent}: Iteratively improves solutions based on validation results\\
        Monitors and adjusts both urban and airport route segments
        \end{minipage} \\
        \hline
    \end{tabular}
    \end{footnotesize}
    \vspace{-.1in}
    \label{tab:agent-placement}
\end{table*}

\subsection{From Workflow Template to Execution Workflow}

Once the template $\mathcal{W_\text{template}}$ is defined, it serves as a structured blueprint that outlines how the problem should be approached. However, a high-level plan alone is insufficient for real-world execution. The next step is to transform the planning workflow into a \textit{real execution workflow} $\mathcal{W_\text{exec}}$, where abstract roles and dependencies are resolved into concrete actionable tasks based on real-world data.

To clarify this transition, consider the difference between $\mathcal{W_\text{template}}$ and $\mathcal{W_\text{exec}}$ in our ride-sharing scenario. In the planning phase, roles such as \texttt{Driver} and \texttt{Passenger} are defined as abstract entities. The template workflow $\mathcal{W_\text{template}}$ specifies how these entities interact, matching drivers with passengers, optimizing routes, and scheduling pickups, without assigning real-world counterparts yet. 

In contrast, the execution workflow $\mathcal{W_\text{exec}}$ performs \textbf{role resolution}, mapping abstract roles to real-world instances. This means assigning an actual driver to a specific vehicle, matching a real passenger to a ride request, and computing precise travel distances based on real-time geo-coordinates. In addition, the execution workflow must dynamically adapt to real-world constraints, such as traffic conditions, vehicle availability, and passenger delays. 

In this process, the meta-planner generates $\mathcal{W_\text{exec}}$, a directed graph where nodes correspond to concrete actions (e.g., ``Driver John departs from location A''), and edges represent dependencies and constraints (e.g., ``Driver John must reach location B before 10:30 AM''). This execution graph integrates real-time data and updates continuously, allowing agents to make informed decisions as conditions evolve.

Thus, the \textbf{template workflow} $\mathcal{W_\text{template}}$ structures how to plan, while the \textbf{execution workflow} $\mathcal{W_\text{exec}}$ governs how real-world actions are performed. Transformation from one to the other is a critical step in $\MP$, ensuring that strategic reasoning is translated into actionable real-time operations.

Now, based on the URS problem specified in Table~\Ref{fig:appDURSSpec} of Section~\ref{sec:ALAS-evaluation} in the main text, 
the list of agents required and
their functional specifications and protocols in Table~\ref{tab:appAgentSpecs}, $\MP$ proceeds generating an execution workflow $\mathcal{W_\text{exec}}$.

\subsubsection{Observation on Sequential Planning}

Let us explain the value of using agents in this problem, even though we have shown that simpler solvers can handle the computational aspects. This discussion touches on key principles of system design and real-world implementation.  

While our Monte Carlo solver effectively found good solutions for this specific instance, $\MP$ offers several advantages that become particularly valuable in real-world ride-sharing systems.  

First, $\MP$ helps manage complexity in dynamic environments. In our exercise, we worked with a static problem where all passenger requests and constraints were known in advance. However, in reality, ride-sharing systems must handle continuous updates—new ride requests arrive at unpredictable times, vehicles experience delays, and road conditions constantly change. With $\MP$, each agent operates independently, monitoring and reacting to changes in its own domain. For example, the Route Planning Agent can dynamically adjust routes in response to traffic updates, while the Capacity Management Agent ensures new passenger requests are accommodated efficiently.  

Second, $\MP$ enables distributed decision-making and parallel processing. Instead of relying on a centralized solver, different agents specialize in handling specific tasks simultaneously. While the Scheduling Agent optimizes pickup times, the Resource Allocation Agent manages vehicle assignments in parallel. This decentralized structure is crucial for scalability—when the system expands to hundreds of vehicles and thousands of passengers, distributing computational workload prevents bottlenecks and ensures efficient operations.  

Third, $\MP$ provides modularity, allowing the system to evolve naturally. Ride-sharing services frequently introduce new features, such as surge pricing or specialized vehicle categories. With an agent-based design, we can integrate a Pricing Agent or a Vehicle Specialization Agent without modifying the core routing logic. Likewise, if we develop a more advanced routing algorithm, we can upgrade the Route Planning Agent without disrupting other system components.  

The separation of concerns through agents also enhances system resilience. If one agent encounters issues—say, the Distance Optimization Agent fails to compute an optimal route—other agents continue operating with fallback strategies. The Plan Validator Agent can detect suboptimal assignments and trigger refinements through the Refinement Agent, ensuring that the system adapts to unforeseen challenges.  

We can think of this like a well-organized team working on a complex project. While a single individual might handle everything, a structured team of specialists—each with clear roles and defined communication protocols—is often more effective, robust, and scalable. In this way, while our Monte Carlo solver demonstrates what is mathematically possible, the agent-based architecture of $\MP$ shows how we can implement it reliably in real-world systems.

\subsection{Reactive Planning under Disruptions}
\label{app:MAPLE-URAP-Disruptions}

The value of multi-agent reactive planning becomes clear in dynamic environments. For example, consider a sudden road closure between locations B and C. While a monolithic solver would need to halt and recompute an entirely new plan from scratch, a modular agent-based approach enables localized, parallel adaptation. A Route Planning Agent can immediately update affected paths, while a Scheduling Agent adjusts arrival estimates, and a Resource Allocation Agent reallocates vehicles, all operating concurrently while preserving system stability. This distributed replanning minimizes disruption impact and maintains overall workflow coherence.

The following case study illustrates these principles in an Urban Ride Sharing (URS) scenario involving ride cancellation and new request insertion.

\paragraph{URS Disruption Handling.}
To evaluate adaptation capabilities, we introduce a disruption where passenger $r_2$ cancels the ride request at 8:05, and a new request $r_5$ at location F arrives at 8:10. $\MP$ replans dynamically, adjusting vehicle assignments while preserving all passenger deadlines. In contrast, baseline LLMs fail to track vehicle states after partial execution and lose consistency with the initial plan, leading to infeasible or incoherent schedules.

%% file: AppendixFamilyReunion.tex
\section{Family Reunion Planning Problem}
\label{app:ALAS-MeetingProblem}

Table~\ref{apptab:ThanksgivingDinner} presents the specification of the problem.  The participating LLMs and their configurations
are depicted in Section~\ref{sec:ALAS-exp-case2} of the main text.



\begin{table}[ht!]
\centering
\caption{Thanksgiving Dinner Coordination Problem}
\begin{small}
\renewcommand{\arraystretch}{1.1}
\fbox{
\begin{minipage}{0.65\textwidth}
\textbf{Objective:} Coordinate family arrivals and dinner preparation for 6:00 PM dinner in Boston

\textbf{Family Members and Arrivals:}
\begin{itemize}[leftmargin=1em, topsep=-.1pt, itemsep=-.1pt, label=-]
\item Sarah (Mom): Host, at home
\item James (Dad): Lands at BOS 1:00 PM from SF
\item Emily (Sister): Lands at BOS 2:30 PM from Chicago
\item Michael (Brother): Driving, arrives 3:00 PM from NY
\item Grandma: Needs pickup from suburban Boston
\end{itemize}

\textbf{Cooking Requirements:}
\begin{itemize}[leftmargin=1em, topsep=-.1pt, itemsep=-.1pt, label=-]
\item Turkey: 4 hours cooking time
\item Side dishes: 2 hours preparation
\item Someone must stay home during cooking
\end{itemize}

\textbf{Transportation Constraints:}
\begin{itemize}[leftmargin=1em, topsep=-.1pt, itemsep=-.1pt, label=-]
\item James must rent car after landing
\item Emily requires airport pickup
\item Travel times:
   \begin{itemize}
   \item Home to BOS Airport: 60 min
   \item BOS Airport to Grandma's: 60 min
   \item Home to Grandma's: 30 min
   \end{itemize}
\end{itemize}

\textbf{Key Requirements:}
\begin{itemize}[leftmargin=1em, topsep=-.1pt, itemsep=-.1pt, label=-]
\item All family members at home for 6:00 PM dinner
\item Turkey and sides ready by dinner time
\item All pickups completed with available drivers
\item Cooking supervision maintained
\end{itemize}
\end{minipage}
}
\end{small}
\label{apptab:ThanksgivingDinner}
\end{table}

\subsection{Phase 1: Network Construction}

\subsubsection{Node (Role) Specifications}

First, meta-planner $\mathcal{MP}$ of $\ALAS$ extracts roles ($\mathcal{N}$) with their required qualifications:
\begin{itemize}[leftmargin=1.5em, topsep=-.15em, parsep=-.15em]
    \item $n_{\text{cook}}$: capability to prepare dinner
    \item $n_{\text{driver1}}$: capability to drive, pick up from airport
    \item $n_{\text{driver2}}$: capability to drive, pick up grandma
    \item $n_{\text{supervisor}}$: capability to monitor oven
\end{itemize}

\subsubsection{Edge (Dependency) Specifications}
Next, $\mathcal{MP}$ identifies dependencies ($\mathcal{E}$) between roles:
\begin{equation}
\mathcal{E} = \{e_{\text{temporal}}, e_{\text{spatial}}, e_{\text{safety}}\}
\end{equation}

The critical dependencies include:
\begin{itemize}[leftmargin=1.5em, topsep=-.15em, parsep=-.15em]
    \item $e_{\text{temporal}}$: 
        - Turkey (4 hours) must finish by 6:00 PM
        - Side dishes (2 hours) must finish by 6:00 PM
        - Airport pickups must align with landing times
    \item $e_{\text{spatial}}$: 
        - Driver-passenger location matching
        - Travel time constraints between locations
    \item $e_{\text{safety}}$:
        - Continuous oven supervision requirement
\end{itemize}


\subsection{Phase 2: Agent Assignments}

After constructing the network structure, $\mathcal{MP}$  selects and assigns agents to monitor both the roles and dependencies.

\subsubsection{Node (Role) Agent Assignment}
For each role, $\mathcal{MP}$ selects monitoring agents with the required capabilities:

\begin{equation}
f_{\text{role}}: \mathcal{N} \rightarrow \mathbf{A}
\end{equation}

The role monitoring agents include:
\begin{itemize}[leftmargin=1.5em, topsep=-.15em, parsep=-.15em]
\item Cook Monitor: Tracks cooking timeline, coordinates meal components
\item Driver Monitor: Validates driver availability 
\item Supervisor Monitor: Ensures oven supervision
\item Resource Monitor: Manages vehicle assignments and actor schedules
\end{itemize}

\subsubsection{Edge (Dependency) Agent Assignment}
For the identified dependencies, $\mathcal{MP}$ assigns specialized monitoring agents:

\begin{equation}
f_{\text{edge}}: \mathcal{E} \rightarrow \mathbf{A}
\end{equation}

Dependencies require these monitoring agents:
\begin{itemize}[leftmargin=1.5em, topsep=-.15em, parsep=-.15em]
\item Temporal Agent: Manages timing constraints (cooking durations, travel times, arrival schedules)
\item Spatial Agent: Tracks location constraints (airport-home-grandma routes)
\item Safety Agent: Ensures oven supervision constraint remains satisfied
\end{itemize}

The resulting agent assignments create a complete monitoring system where:
\begin{itemize}[leftmargin=1.5em, topsep=-.15em, parsep=-.15em]
\item Role agents track individual actor assignments and qualifications
\item Edge agents monitor interactions and dependencies between roles
\item All agents coordinate to maintain global constraint satisfaction
\end{itemize}

\begin{table}[th!]
\centering
\caption{Node and Edge Monitoring Agent Requirements}
\vspace{-.05in}
\begin{minipage}{0.48\textwidth}
\centering
\label{tab:node_agents}
\caption*{(a) Node (Role) Monitoring Agent}
{\fontsize{7pt}{9pt}\selectfont
\begin{tabular}{|p{0.15\textwidth}|p{0.32\textwidth}|p{0.38\textwidth}|}
\hline
\textbf{Agent} & \textbf{Input Protocol} & \textbf{Output Protocol} \\
\toprule \hline
Cook \newline Monitor & Role: cook \newline Qualifications: skills \newline Time: prep and cook & Status: progress \newline Alerts: timing issues! \newline Updates: completed? \\
\hline
Driver \newline Monitor & Role: driver \newline Qs: license, rest \newline Where: current GPS & Status: availability \newline Alerts: fatigue warnings \newline Updates: new GPS \\
\hline
Supervisor \newline Monitor & Role: supervisor \newline Location: house \newline Duration: cover time & Status: covered? \newline Alerts: coverage gaps! \newline Updates: role transitions \\
\hline
\end{tabular}
}
\end{minipage}
\hfill
\begin{minipage}{0.48\textwidth}
\centering
\caption*{(b) Edge (Dependency) Monitoring Agent}
\label{tab:edge_agents}
{\fontsize{7pt}{9pt}\selectfont
\begin{tabular}{|p{0.12\textwidth}|p{0.34\textwidth}|p{0.30\textwidth}|}
\hline
\textbf{Agent} & \textbf{Input Protocol} & \textbf{Output Protocol} \\
\toprule \hline
Temporal & Start times \newline Durations \newline Deadlines & Schedule conflicts \newline Timing violations \newline Schedule updates \\
\hline
Spatial & Locations \newline Routes \newline Travel time (variations) & Route violations \newline Location conflicts \newline Path updates \\
\hline
Safety & Critical constraints \newline Resource states \newline Coverage requirements & Safety violations \newline Resource conflicts \newline Mitigation plans \\
\hline
\end{tabular}
}
\end{minipage}
\vspace{-.1in}
\end{table}

\subsubsection{Common Sense Constraint Analysis (Performed by an LLM)}
\label{app:ThanksGiving-common-sense}

A common sense agent identifies the following implicit constraints that can affect Thanksgiving dinner planning.
This list is generated by Claude given the problem statement.

\begin{itemize}[leftmargin=1.5em, topsep=-.15em, parsep=-.15em]
\item \textit{Physical Processing Times:}
   \begin{itemize}
   \item Airport luggage claim: 30 minutes
   \item Car rental procedures: 30 minutes
   \item Holiday traffic variations
   \item Winter weather considerations
   \end{itemize}

\item \textit{Human Factors:}
   \begin{itemize}
   \item Driver fatigue after long trips
   \item Cooking preparation overhead
   \item Optimal turkey baking tips (non-disruptive baking and ready 30 minutes before eating)
   \item Task switching delays
   \item Required rest periods
   \end{itemize}

\item \textit{Resource Dependencies:}
   \begin{itemize}
   \item Vehicle passenger capacity
   \item Oven temperature management
   \item Kitchen workspace limits
   \item Shared resource coordination
   \end{itemize}

\item \textit{Social Considerations:}
   \begin{itemize}
   \item Personal preferences for interactions
   \item Family dynamics in assignments
   \item Post-travel guest comfort
   \item Host preparation requirements
   \end{itemize}
\end{itemize}

\input{TableWorkflow}

\subsubsection{Common Sense Constraint Analysis and Verification (Human in the Loop)}
The common sense constraints identified above require different verification approaches:

\paragraph{Agent-Required Information}
These constraints need specialized agents to verify and quantify:
\begin{itemize}[leftmargin=1.5em, topsep=-.15em, parsep=-.15em]
\item \textit{Airport Operations}
   \begin{itemize}
   \item United Airlines' average luggage delivery time at BOS Terminal B
   \item Terminal B to rental car center: shuttle schedule, walking options
   \item Historical flight delay patterns for November at BOS
   \end{itemize}

\item \textit{Weather and Traffic}
   \begin{itemize}
   \item Boston weather forecast for the event date
   \item Historical traffic patterns on Thanksgiving days
   \item Impact on airport-city-suburb travel times
   \end{itemize}

\item \textit{Task Dependencies}
   \begin{itemize}
   \item Kitchen workflow analysis for parallel cooking tasks
   \item Resource contention in meal preparation
   \item Critical path identification in cooking timeline
   \end{itemize}
\end{itemize}

\paragraph{Human Verification}
Certain constraints require explicit human input to ensure that the planning process takes into account subtle interpersonal and individual considerations. These include:

\begin{itemize}[leftmargin=1.5em, topsep=-.15em, parsep=-.15em]
    \item \textit{Family Dynamics}
    \begin{itemize}[leftmargin=1.5em, topsep=-.1em, parsep=-.1em]
        \item Preferred pickup arrangements for Grandma.
        \item Optimal relationship-based task pairings.
        \item Social comfort factors in assignments (e.g., Sarah and Grandma do not share a kitchen).
    \end{itemize}

    \item \textit{Personal Capabilities}
    \begin{itemize}[leftmargin=1.5em, topsep=-.1em, parsep=-.1em]
        \item Individual cooking experience levels.
        \item Driver comfort with airport navigation.
        \item Multi-tasking abilities of participants.
    \end{itemize}
\end{itemize}

This separation ensures that agents focus on collecting quantifiable data while humans provide essential social and personal insights. $\mathcal{MP}$ can then integrate both types of information into the final workflow design.

\subsection{Agent Requirements and Assignments}

The $\mathcal{MP}$ requires two categories of agents. $\mathcal{MP}$ specifies their requirements in the protocol buffer format in Table~\ref{tab:node_agents} for the nodes and
Table~\ref{tab:edge_agents} for the edges, respectively.

Each agent must implement these protocols to participate in the workflow. The meta-planner selects agents from the pool based on their ability to satisfy these interface requirements. During execution, agents communicate through these standardized protocols while maintaining their specialized monitoring functions.

\subsection{Monitoring Protocols and Dynamic Adjustments}

The workflow monitoring operates through a hierarchical protocol system that enables both routine supervision and dynamic adjustments.

\paragraph{Basic Monitoring Protocol}
Each agent maintains a continuous monitoring cycle:
\begin{equation}
\text{monitor}: \text{State} \rightarrow \{\text{normal, warning, violation}\}
\end{equation}

For example, the temporal agent tracks schedule adherence:
\begin{equation}
\Delta t = t_{\text{planned}} - t_{\text{actual}}
\begin{cases}
   \text{normal} & \text{if } |\Delta t| < \text{buffer} \\
   \text{warning} & \text{if } \text{buffer} \leq |\Delta t| < \tau \\
   \text{violation} & \text{if } |\Delta t| \geq \text{ threshold } \tau 
\end{cases}
\end{equation}

\paragraph{Dynamic  Adjustment Mechanism}
When deviations occur, the system initiates a three-phase response:

1. \textit{Impact Assessment}:
\begin{equation}
\text{impact}(e) = \sum_{n \in \text{affected}(e)} \text{severity}(n) \times \text{urgency}(n)
\end{equation}

2. \textit{Solution Generation}:
\begin{equation}
S^* = \argmin_{s \in \text{Solutions}} \{\text{cost}(s) | \text{feasible}(s)\}
\end{equation}

3. \textit{Coordination Protocol}:
\begin{equation}
\text{update}: (W_{\text{current}}, S^*) \rightarrow W_{\text{new}}
\end{equation}

For instance, if James's flight is delayed:
\begin{itemize}[leftmargin=1.5em, topsep=-.15em, parsep=-.15em]
\item Spatial agent detects arrival time change
\item Temporal agent calculates ripple effects
\item Role agents evaluate reassignment options
\item Safety agent verifies continued supervision coverage
\end{itemize}

The meta-planner $\mathcal{MP}$ coordinates these responses while maintaining global constraint satisfaction.

\subsection{Integrated Workflow Network}

Table~\ref{tab:workflow_spec} presents the
resulting workflow network $\mathbf{W^*}$, which includes
all nodes and edges, and their assigned agents and protocols.

\begin{enumerate}[leftmargin=1.2em, topsep=-.15em, parsep=-.15em]
\item \textit{Role Nodes:}
    \begin{itemize}[leftmargin=1.0em, topsep=-.15em, parsep=-.15em]
    \item Cook1: Sarah (primary) or Grandma (if at home) with 4-hour turkey + 2-hour sides
    \item Driver1: James (after car rental) or Michael
    \item Driver2: Available person after initial pickups
    \item Supervisor: Must be present while turkey cooks
    \end{itemize}

\item \textit{Dependencies:}
    \begin{itemize}[leftmargin=1.0em, topsep=-.15em, parsep=-.15em]
    \item Temporal: Verified airport processing + travel times
    \item Spatial: Traveling routes with traffic consideration
    \item Safety: Continuous oven supervision requirement
    \end{itemize}

\item \textit{Agent Monitoring:}
    \begin{itemize}[leftmargin=1.0em, topsep=-.15em, parsep=-.15em]
    \item Temporal Agent: Schedules with verified buffer times
    \item Spatial Agent: Real-time location and route mgmt.
    \item Safety Agent: Role coverage for supervision
    \end{itemize}
\end{enumerate}

\subsection{Agent Interaction Specifications}

Please, see Table~\ref{tab:agent_protocols}.

\begin{table*}[th!]
\caption{Agent Interaction Protocols and State Transitions}
{\fontsize{7pt}{9pt}\selectfont
\begin{tabular}{|p{0.14\textwidth}|p{0.27\textwidth}|p{0.27\textwidth}|p{0.20\textwidth}|}
\hline
\textbf{Interaction Type} & \textbf{Protocol} & \textbf{State Transitions} & \textbf{Validation Rules} \\
\hline
\multicolumn{4}{|l|}{\textit{Node-to-Node Interactions}} \\
\hline
Cook$\leftrightarrow$ Supervisor & 
Protocol: cooking\_handoff() \newline
Message: (task, duration, reqs.) &
States: prep → cooking → comp. \newline
Trigger: task\_state\_change() &
Validate: coverage() \newline
Alert: coverage\_gap() \\
\hline
Driver1 $\leftrightarrow$ Driver2 &
Protocol: pickup\_handoff() \newline
Message: (location, time, passenger) &
States: available → enroute → comp. \newline
Trigger: location\_change() &
Validate: timing\_feasible() \newline
Alert: schedule\_conflict() \\
\hline
\multicolumn{4}{|l|}{\textit{Edge Agent Operations}} \\
\hline
Temporal Agent &
Protocol: schedule\_monitor() \newline
Message: (event, time, dependencies) &
States: scheduled → active → comp. \newline
Trigger: time\_milestone() &
Validate: timing\_feasible() \newline
Alert: delay\_impact() \\
\hline
Spatial Agent &
Protocol: location\_track() \newline
Message: (actor, position, dest.) &
States: idle → moving → arrived \newline
Trigger: position\_update() &
Validate: route\_feasible() \newline
Alert: travel\_delay() \\
\hline
\end{tabular}}
\label{tab:agent_protocols}
\vspace{-.1in}
\end{table*}

\subsection{Augmented Problem Statement Revised with \textbf{W*}}
\label{app:ThanksGivingAugmentPS}

Given the $\mathbf{W^*}$ generated by $\MP$'s meta-planner $\mathcal{MP}$, the Thanksgiving Dinner Planning problem statement is revised as follows: \\
\newline
\noindent \underline{\textit{Initial Setup:}}
\begin{itemize}[leftmargin=1.5em, topsep=-.15em, parsep=-.15em]
\item Mom (Sarah) is hosting Thanksgiving dinner at 6:00 PM in Boston. The following family members are traveling:
\item Dad (James) flying from San Francisco, landing at 1:00 PM Eastern time.
\item Sister (Emily) flying from Chicago, landing at 2:30 PM
\item Brother (Michael) driving from New York, estimated arrival 3:00 PM at home
\item Grandma is healthy and needs to be picked up from her home in suburban Boston
\end{itemize}

\noindent \underline{{\color{red}{\textbf{*}}} \textbf{Common Sense Augmented Constraints:}}
\begin{itemize}[leftmargin=1.5em, topsep=-.15em, parsep=-.15em]
\item The airport luggage pickup time after landing is 30 minutes.
\item Renting a car takes 30 minutes.
\item One person can simultaneously prepare turkey and side dishes.
\item Grandma prefers Michael to pick her up, provided that it does not cause the dinner time delay (soft constraint).
\item Grandma and Sarah prefer not to cook together in the kitchen.
\item The best turkey receipt and baking instructions included.
\item Traffic congestion is not factored into current planning.
\end{itemize}

\vspace{.1in}
\noindent \underline{\textit{Planning Validation Set:}}
\begin{enumerate}[leftmargin=1.5em, topsep=-.15em, parsep=-.15em]
\item All tasks and dependencies must be strictly
observed in the plan, or the plan fails.
\item Dinner time is strictly at 6:00 PM, all tasks must
be completed by then (redundancy).
\item Account for the idle time of each person.
\item The schedule consists of three columns: time, task, and assigned person(s).
\end{enumerate}

\begin{table}[t!]
\centering
\caption{DeepSeek's Plan, Two Iterations}
{\fontsize{7pt}{9pt}\selectfont
\begin{tabular}{|l|p{0.22\textwidth}|l|}
\hline
\textbf{Time} & \textbf{Task} & \textbf{Assigned} \\ 
\toprule \hline
1:00 PM & James lands at Boston & James \\ \hline
1:00--1:30 PM & James picks up luggage & James \\ \hline
1:30--2:00 PM & James rents a car & James \\ \hline
2:00 PM & {Turkey in oven} (4 hours; requires monitoring) & Sarah \\ \hline
2:00--3:00 PM & James waits at airport (idle) & James \\ \hline
2:30 PM & Emily lands at Boston & Emily \\ \hline
2:30--3:00 PM & Emily waits for luggage & Emily \\ \hline
3:00 PM & James picks up Emily & James \\ \hline
3:00 PM & Michael arrives home & Michael \\ \hline
3:00 PM & Michael departs to Grandma & Michael \\ \hline
3:30 PM & Michael picks up Grandma & Michael \\ \hline
3:30--4:00 PM & Michael drives back home with Grandma & Michael \\ \hline
3:00--4:00 PM & James drives Emily home (airport to home: 1 hour) & James \\ \hline
4:00 PM & James and Emily home & James \\ \hline
4:00 PM & M. and Grandma home & Michael \\ \hline
4:00--6:00 PM & Sarah prepares side dishes & Sarah \\ \hline
6:00 PM & Thanksgiving dinner begins & All \\ \hline
\end{tabular}}
\label{tab:DeepSeekPlan1}
\vspace{-.1in}
\end{table}

\subsection{Experiment \#1: Sequential Planner with Common Sense}
\label{app:SequentialPlanExperiment}
The first experiment utilized the augmented problem specification with common sense reasoning, incorporating realistic constraints such as luggage claim time and rental car pickup time.

We evaluated four standalone LLMs alongside $\ALAS$.
Both $\ALAS$ and Gemini consistently generated feasible schedules similar to Table~\ref{tab:DeepSeekPlan1}, while other LLMs encountered significant challenges.  

Upon analyzing the number of iterations required for a feasible plan, DeepSeek and Claude each needed one revision (two iterations), while GPT4o required two revisions (three iterations). In terms of scheduling quality—measured by slack time, total driving distance, and load balance—DeepSeek (Table~\ref{tab:DeepSeekPlan1}) outperformed both GPT4o (Table~\ref{tab:GPT4oPlan1}) and Claude (Table~\ref{tab:ClaudePlan1}). DeepSeek optimized efficiency by having James wait at the airport for 30 minutes to pick up Emily. In contrast, Claude inefficiently scheduled James to drive home and then return to the airport for Emily, creating unnecessary travel. GPT4o assigned James to return home and tasked Michael with separately picking up Emily and then Grandma, resulting in suboptimal load distribution. A more efficient solution would have scheduled Michael to collect Emily first, then proceed with her to Grandma's home, allowing all three to return together—saving 30 minutes of driving time while enhancing Grandma's experience of seeing both grandchildren simultaneously.

\input{Figure-FamilyReunion-Exp1}


\subsubsection{Observations of Errors in Standalone LLMs}
\label{app:exp1-observations}

Although DeepSeek and Claude eventually produced feasible static plans in their second iterations, Tables~\ref{tab:reunion-Claude} and \ref{tab:reunion-DeepThink} highlight critical errors in their initial attempts.

These errors included misestimated travel times (calculating 60-minute trips as 45 or 30 minutes, highlighted in shaded red) and implausible scheduling decisions, such as beginning turkey preparation at 10 AM or allowing James to depart the airport without Emily. 

\input{Table-Case2-LLM-Failures}


Such failures result from context erosion in extended prompts~\cite{liu-etal-2024-lost, xiao2024attentionsink} and expanding context windows. Research demonstrates that extended contexts accelerate information loss~\cite{park2023context, wei2023larger}, leading to constraint violations. $\ALAS$ circumvents these issues through its modular architecture, where specialized agents process only domain-relevant information while independent validation mechanisms ensure constraint adherence.

\paragraph{Handling Long Dependencies}
Complex scheduling problems reveal cascading errors when dependencies overlap. Critical constraints, particularly those involving multiple factors, frequently get dropped during iterative problem-solving. \\
\textbf{Reason}: Cognitive limitations restrict simultaneous constraint tracking, making exhaustive verification challenging in single processing passes. \\
\textbf{Solution Framework}:
\begin{itemize}[leftmargin=1.2em, topsep=-.15em, parsep=-.15em]
\item Isolate and systematically enumerate atomic task dependencies.
\item Implement comprehensive verification of global constraint satisfaction.
\item Develop robust mechanisms for systematic conflict resolution.
\end{itemize}

\paragraph{Stale Memory and Iterative Revisions}
Iterative solutions risk propagating errors due to incomplete constraint resets. \\
\textit{Reason}: Excessive reliance on previous solutions without comprehensive constraint re-evaluation leads to compounding errors. \\
\textbf{Relation to Gödel's Incompleteness}:
\begin{itemize}[leftmargin=1.2em, topsep=-.15em, parsep=-.15em]
\item Formal systems capable of arithmetic necessarily contain unprovable truths.
\item Similarly, inherited solution errors inhibit consistent constraint satisfaction.
\item Clean-state resets become essential for systematic error prevention.
\end{itemize}

\paragraph{Implementation Strategy}
Reset to a clean baseline state for each iteration, thoroughly re-evaluating all constraints. \\
\textit{Core Challenges}:
\begin{itemize}[leftmargin=1.2em, topsep=-.15em, parsep=-.15em]
\item Effective management of nested dependencies.
\item Prevention of residual errors across iterations.
\item Maintenance of cross-iteration consistency.
\end{itemize}

\begin{table}[ht]
\centering
\caption{Sequential Planning. (\# = iterations)}
{\fontsize{7pt}{9pt}\selectfont
\begin{tabular}{|l|c|p{0.60\linewidth}|}
\hline
\textbf{LLM} & \textbf{\#}  & \textbf{Notable Features} \\
\toprule \hline
\textbf{DeepSeek} & 2  & Optimized airport wait time for James; balanced workload \\
\hline
Claude & 2 &  Unnecessary travel between pickup tasks (no need to go home before next pickup) \\
\hline
GPT4o & 3 &  Extra travel for Michael; suboptimal load balance \\
\hline
\end{tabular}}
\label{tab:sequential_results}
\end{table}

Table~\ref{tab:sequential_results} synthesizes the detailed schedules documented in Tables~\ref{tab:DeepSeekPlan1}, \ref{tab:GPT4oPlan1}, and \ref{tab:ClaudePlan1}. 
DeepSeek demonstrated good scheduling efficiency by optimizing James's airport wait time for Emily's pickup, requiring only two iterations for convergence. Although GPT4o eventually produced a valid solution after three iterations, it created suboptimal travel patterns with redundant trips by Michael. Claude's solution, though feasible in two iterations, incorporated unnecessary travel between pickup tasks. 
In contrast to the inconsistent performance of standalone LLMs, $\ALAS$ consistently generated feasible and efficient plans in all test runs.

\subsection{Experiment \#2: Reactive Planner for Flight Delay}
\label{app:ReactivePlanExperiment}


This disruption scenario was stated in the main body of the paper in
Section~\ref{sec:ALAS-exp-case2}. Under the prompt for
reactive planning, it states the disruption as:
``At noon, James' flight is delayed until 4:00 PM. Update the schedule to meet the deadline at 6:00 pm while meeting all constraints.''

We tested only $\ALAS$, DeepSeek and Claude 3.7 for disturbance handling, because they survived sequential planning. 
$\ALAS$ successfully generated a feasible reactive plan, whereas
DeepSeek and Claude 3.7 failed.

The challenge: James (assigned to pick up Emily) is delayed until 4:00 PM. Both DeepSeek and Claude failed in seven out of ten runs by: (1) scheduling Emily for a taxi, violating family pick-up constraints; (2) delaying or missing Grandma's pickup; or (3) missing the dinner deadline. An example failed schedule (the name of the LLM is purposely concealed) is presented in Table~\ref{tab:failed-reactive-planning}, where one can see that when all family members are ready to eat the delicious turkey, ``oops, Grandma was not picked up.''

\begin{table}[htb]
\centering
\caption{Failed Reactive Planning Schedule with Flight Delay Disruption}
\label{tab:failed-reactive-planning}
{\fontsize{7pt}{9pt}\selectfont
\begin{tabular}{|p{1.5cm}|p{4.2cm}|p{6.2cm}|}
\hline
\textbf{Time} & \textbf{Task} & \textbf{Explanation} \\
\toprule
\hline
(Before Noon) & Tasks like initial prep are underway & The disruption is received at noon. \\
\hline
2:00 PM & Sarah puts turkey in the oven. & \textbf{Unchanged:} Critical path for food. Sarah must stay home to supervise (2:00 PM - 6:00 PM). \\
\hline
2:30 PM & Emily lands at BOS. & \textbf{Unchanged:} Sister's flight arrives as scheduled. \\
\hline
3:00 PM & Michael arrives home. & \textbf{Unchanged:} Brother arrives from NY. \textbf{New Role:} He is now the designated driver for Emily's pickup. \\
\hline
$\sim$3:15 PM & Michael departs home for BOS Airport. & Michael must leave soon after arriving to pick up Emily. (60 min travel time to BOS). \\
\hline
4:00 PM & James lands at BOS. & \textbf{Disruption Effect:} Dad arrives 3 hours later than planned. \\
\hline
4:00 PM & Sarah starts cooking side dishes. & \textbf{Unchanged:} Side dishes still need 2 hours. Sarah starts them as she is the only one home. \\
\hline
4:00 - 5:00 PM & James deplanes, gets luggage, rents a car. & \textbf{Delayed:} James needs to rent a car after landing. \\
\hline
$\sim$4:15 PM & Michael arrives at BOS Airport. & Michael reaches the airport to meet Emily. \\
\hline
$\sim$4:15 - 4:45 & Michael and Emily meet at BOS & Allow time for Emily to get luggage and meet Michael. \\
\hline
$\sim$4:45 PM & Michael and Emily depart BOS for home. & They begin the 60-minute drive home. \\
\hline
$\sim$5:00 PM & James departs BOS for home (in rental car). & James starts his 60-minute drive home after getting rental car. \\
\hline
$\sim$5:45 PM & Michael and Emily arrive home. & Brother and Sister are home. \\
\hline
6:00 PM & James arrives home. & Dad arrives home just in time for the planned dinner. \\
\hline
6:00 PM & Dinner Time & Turkey and side dishes are ready. Sarah, Michael, Emily, and James are home. \\
\hline
\colorbox{red!15} POST 6:00 PM & Grandma's Pickup - Unresolved & \textbf{Major Impact:} With James arriving at 4:00 PM and needing an hour for car rental (leaving BOS at 5:00 PM), he cannot pick up Grandma and return by 6:00 PM. Michael is busy with Emily's pickup until 5:45 PM and cannot make the $\sim$1.5 hr round trip for Grandma before 6:00 PM. Sarah must supervise the oven. \textbf{Under these constraints, Grandma cannot join the 6:00 PM dinner.} \\
\hline
\end{tabular}}
\end{table}

These failures are due to greedy rescheduling and increased context loss during replanning. Studies confirm that longer contexts can worsen reliability~\cite{wei2023larger, zhang2023careful}. $\ALAS$ generated feasible plans in all runs by: (1) updating James's state in persistent memory; (2) detecting conflicts systematically; (3) evaluating alternatives; and (4) validating all constraints. Following LRCP (Local Reactive Compensation Protocol, depicted in Algorithm~\ref{alg:ALAS-LCSR-Full}), $\ALAS$ produced four distinct solution patterns, all feasible. While requiring 2.5× the computation time (12.1s vs. 4.8s), this overhead ensures feasibility, crucial in time-sensitive domains.

\subsection{ALAS's Reactive Plan and Stateful Explanation}
$\ALAS$ proposes a simple yet effective reroute: Michael drives straight to Boston Airport rather than stopping at home first.  
This common-sense spatial adjustment—overlooked by the other LLMs—originates from $\mathcal{MP}$'s state-aware reasoning module.  
After collecting Emily, Michael proceeds directly to Grandma's house, trimming roughly 30 minutes of travel and giving Grandma the pleasant surprise of seeing two grandchildren arrive together.
Table \ref{tab:ALASReactivePlan} lists the resulting feasible schedule.  
The critical advantage is that $\ALAS$'s continuous tracking of each participant's state and history enables timely compensations and preserves the on-time family reunion.

\begin{table}[ht!]
\vspace{-.1in}
\centering
\caption{$\ALAS$ Reactive Plan: Optimized routing via persistent state history and compensation.}
\vspace{.05in}
{\fontsize{7pt}{9pt}\selectfont
\begin{tabular}{|l|p{0.42\textwidth}|l|}
\hline
\textbf{Time} & \textbf{Task} & \textbf{Assigned} \\
\toprule
\hline
12:00 PM & James' delay known, Michael on his way from NYC to home, is rerouted to Boston airport to meet Emily (4-hour drive). & Michael \\
\hline
2:00 PM & Start cooking turkey & Sarah \\
\hline
2:30 PM & Emily lands at Boston & Emily \\
\hline
3:00 PM & Emily gets her luggage & Emily \\
\hline
3:00 PM & Michael arrives at Logan airport, picks up Emily. & Michael \\
\hline
3:00–4:00 PM & Michael drives Emily home & Michael \\
\hline
4:00 PM & Michael departs for Grandma & Michael \\
\hline
4:00 PM & James lands at Boston Airport & James \\
\hline
4:00–4:30 PM & James picks up luggage & James \\
\hline
4:30–5:00 PM & James rents car (30 minutes). & James \\
\hline
4:30 PM & Michael arrives at Grandma's & Michael \\
\hline
5:00 PM & Michael \& Grandma arrive home. & Michael, Grandma \\
\hline
5:00–6:00 PM & James drives home from BOS & James \\
\hline
4:00–6:00 PM & Sarah prepares side dishes (overlaps with turkey). & Sarah \\
\hline
6:00 PM & James arrives home. Dinner served. & All \\
\hline
\end{tabular}}
\label{tab:ALASReactivePlan}
\vspace{-.1in}
\end{table}

%% file: TableWorkflow.tex
\begin{table*}[th!]
\caption{Complete Workflow Specification: Nodes, Edges, and Agent Assignments}
{\fontsize{7pt}{9pt}\selectfont
\begin{tabular}{|p{0.08\textwidth}|p{0.12\textwidth}|p{0.2\textwidth}|p{0.30\textwidth}|p{0.16\textwidth}|}
\hline
\textbf{Type} & \textbf{Component} & \textbf{Requirements} & \textbf{Agent Protocol} & \textbf{Dependencies} \\
\hline
\multicolumn{5}{|l|}{\textit{Node Components (Roles)}} \\
\hline
Node & Cook Role \newline (Sarah) & 
- Turkey (4hr) \newline
- Side dishes (2hr) \newline
- Kitchen management \newline
- Time management &
Input: schedule, resources, recipes \newline
Output: task progress, completion \newline
Monitor: kitchen\_state() → status \newline
Validate: cooking\_constraints() &
Connected to: \newline
- Supervisor \newline
- Resource edges \\
\hline
Node & Driver1 \newline (James/Michael) &
- Valid license \newline
- Airport navigation \newline
- Car rental capable \newline
- Rest state adequate &
Input: flight times, routes \newline
Output: location, ETA \newline
Monitor: driver\_state() → status \newline
Validate: driver\_constraints() &
Connected to: \newline
- Airport pickup \newline
- Travel edges \\
\hline
Node & Driver2 \newline (Flexible) &
- Valid license \newline
- Local navigation \newline
- Availability window \newline
- Rest state adequate &
Input: pickup schedule, route \newline
Output: location, ETA \newline
Monitor: driver\_state() → status \newline
Validate: driver\_constraints() &
Connected to: \newline
- Grandma pickup \newline
- Travel edges \\
\hline
Node & Supervisor \newline (Flexible) &
- Home presence \newline
- Oven monitoring \newline
- Safety awareness \newline
- Time commitment &
Input: cooking schedule, rules \newline
Output: supervision status \newline
Monitor: safety\_state() → status \newline
Validate: safety\_constraints() &
Connected to: \newline
- Cook role \newline
- Safety edges \\
\hline
\multicolumn{5}{|l|}{\textit{Edge Components (Dependencies)}} \\
\hline
Edge & Temporal &
- Schedule tracking \newline
- Buffer management \newline
- Sequence logic \newline
- Critical path &
Input: timestamps, durations \newline
Output: schedule conflicts \newline
Monitor: schedule\_state() → alerts \newline
Optimize: timeline\_adjust() &
Connects: \newline
- All roles \newline
- All activities \\
\hline
Edge & Spatial &
- Location tracking \newline
- Route optimization \newline
- Traffic updates \newline
- Distance constraints &
Input: locations, routes \newline
Output: travel updates \newline
Monitor: location\_state() → alerts \newline
Optimize: route\_adjust() &
Connects: \newline
- Drivers \newline
- Locations \\
\hline
Edge & Resource &
- Vehicle allocation \newline
- Kitchen resources \newline
- People availability \newline
- Capacity limits &
Input: resource demands \newline
Output: allocation status \newline
Monitor: resource\_state() → alerts \newline
Optimize: resource\_adjust() &
Connects: \newline
- All roles \newline
- All resources \\
\hline
Edge & Safety &
- Oven monitoring \newline
- Driving safety \newline
- Food safety \newline
- Critical rules &
Input: safety requirements \newline
Output: violation alerts \newline
Monitor: safety\_state() → alerts \newline
Enforce: safety\_rules() &
Connects: \newline
- All roles \newline
- Critical tasks \\
\hline
\end{tabular}}
\label{tab:workflow_spec}
\end{table*}

%% file: Figure-FamilyReunion-Exp1.tex
\begin{table}[ht!]
\centering
\begin{minipage}{0.48\textwidth}
    \centering
    \caption{GPT4o's Plan, Three Iterations}
    {\fontsize{7pt}{9pt}\selectfont
    \begin{tabular}{|l|l|l|}
    \hline
    \textbf{Time}        & \textbf{Task} & \textbf{Assigned}      \\ 
\toprule \hline
    1:00 PM              & Land at BOS Airport                      & James                     \\ \hline
    1:00-1:30 PM         & Luggage pickup                           & James                     \\ \hline
    1:30-2:00 PM         & Rent car                                 & James                     \\ \hline
    2:00 PM              & Start turkey                             & Sarah                     \\ \hline
    2:00-3:00 PM         & Drive home                               & James                     \\ \hline
    2:30 PM              & Land at BOS Airport                      & Emily                     \\ \hline
    3:00 PM              & Arrive home                              & Michael                   \\ \hline
    3:00-4:00 PM         & Drive to airport, pick up Emily          & Michael                   \\ \hline
    4:00-5:00 PM         & Return home with Emily                   & Michael                   \\ \hline
    5:00-5:30 PM         & Drive to Grandma's                       & Michael                   \\ \hline
    5:30-6:00 PM         & Return with Grandma                      & Michael                   \\ \hline
    4:00-6:00 PM         & Prepare side dishes                      & Sarah                     \\ \hline
    6:00 PM              & Dinner served                            & All                       \\ \hline
    \end{tabular}}
    \label{tab:GPT4oPlan1}
\end{minipage}%
\hfill
\begin{minipage}{0.48\textwidth}
    \centering
    \caption{Claude's Plan, Two Iterations}
    {\fontsize{7pt}{9pt}\selectfont
    \begin{tabular}{|l|l|l|}
    \hline
    \textbf{Time} & \textbf{Task} & \textbf{Assigned} \\
    \toprule \hline
    1:00 PM & Land at BOS Airport & James \\ \hline
    1:00-1:30 PM & Luggage pickup & James \\ \hline
    1:30-2:00 PM & Rent car & James \\ \hline
    2:00 PM & Start turkey & Sarah \\ \hline
    2:00-3:00 PM & Drive home & James \\ \hline
    2:30 PM & Land at BOS Airport & Emily \\ \hline
    3:00 PM & Arrive home & Michael \\ \hline
    3:00-4:00 PM & Drive to airport, pick up Emily & James \\ \hline
    4:00-5:00 PM & Return home with Emily & James \\ \hline
    4:30-5:00 PM & Drive to Grandma's & Michael \\ \hline
    5:00-5:30 PM & Return with Grandma & Michael \\ \hline
    4:00-6:00 PM & Prepare side dishes & Sarah \\ \hline
    6:00 PM & Dinner served & All \\
    \hline
    \end{tabular}}
    \label{tab:ClaudePlan1}
\end{minipage}
\vspace{-.1in}
\end{table}

%% file: Table-Case2-LLM-Failures.tex
\begin{table}[ht!]
\vspace{-.15in}
\centering
\begin{minipage}{0.485\textwidth}
    \centering
    \caption{Family Reunion - Claude. Gray indicates questionable, red wrong time calculation.}
    {\fontsize{7pt}{9pt}\selectfont
    \begin{tabular}{|p{0.1\textwidth}|p{0.56\textwidth}|p{0.15\textwidth}|}
    \hline
    \textbf{Time} & \textbf{Task} & \textbf{Assigned} \\
    \toprule
    \hline
    \colorbox{yellow} 10:00 & Put turkey in oven (4-hour) & S \\
    \colorbox{white} 13:00 & James lands at BOS & J \\
    \colorbox{white} 13:30 & James rents car at BOS & J \\
    \colorbox{yellow} 14:00 & James leaves airport for home & J \\
    \colorbox{white} 14:30 & Emily lands at BOS & E \\
    \colorbox{white} 15:00 & James \& Michael arrive home;  & J, M \\
    \colorbox{white} 15:15 & M leaves to pick up E from BOS & M \\
    \colorbox{white} 15:30 & S begins side dishes (2-hour) & S \\
    \colorbox{white} 16:15 & M arrives at BOS to pick up E & M, E \\
    \colorbox{white} 16:30 & James leaves to pick up Grandma & J \\
    \colorbox{white} 17:00 & J arrives at G's house & J, G \\
    \colorbox{lightred} 17:00 & Michael and Emily arrive home & M, E \\
    \colorbox{white} 17:30 & J and G home; side dish done & J, G, S \\
    \colorbox{white} 18:00 & Turkey complete, dinner ready & All \\
    \hline
    \end{tabular}
    }
    \label{tab:reunion-Claude}
\end{minipage}%
\hfill
\begin{minipage}{0.486\textwidth}
    \centering
    \caption{Family Reunion - DeepSeek. Red indicates wrong time calculation.}
    {\fontsize{7pt}{9pt}\selectfont
    \begin{tabular}{|p{0.1\textwidth}|p{0.60\textwidth}|p{0.15\textwidth}|}
    \hline
    \textbf{Time} & \textbf{Task} & \textbf{Assigned} \\
    \toprule
    \hline
    \colorbox{white} 12:00 & S starts cooking the turkey (4 hours, ready by 16:00) & S \\
    \colorbox{white} 13:00 & J lands at BOS and rents a car & J \\
    \colorbox{white} 13:30 & James drives from BOS to Grandma's & J \\
    \colorbox{white} 14:30 & James picks up Grandma, Emily lands & J, G \\
    \colorbox{white} 15:00 & Michael home (from NY) & M \\
    \colorbox{white} 15:15 & James and Grandma arrive back home (30 min drive) & J, G \\
    \colorbox{lightred} 15:15 & M leaves to pick up E from BOS (60 min round trip) & M \\
    \colorbox{white} 16:00 & Turkey done; S starts side dishes (2 hours) & S \\
    \colorbox{lightred} 16:15 & M and E return home from BOS & M, E \\
    \colorbox{white} 18:00 & Family reunion dinner begins with all members present & All \\
    \hline
    \end{tabular}
    }
    \label{tab:reunion-DeepThink}
\end{minipage}
\vspace{-.1in}
\end{table}

%% file: AppendixJSSPSupplement.tex
\section{Additional JSSP Results and Analysis}
\label{app:ALAS-JSSP}
This appendix augments our core experimental findings with the full prompt specification, failure rate statistics, and pointers to supplementary visualizations. 

\vspace{-.1in}
\subsection{LLM Prompt Design}
Table \ref{tab:JSSP-prompt} shows the \emph{standardized} prompt issued to every standalone LLM and to the $\ALAS$ meta-planner (Phase 1 on Figure~\ref{fig:ALAS-architecture}).  
Standalone LLMs can recommend off-the-shelf solvers and invoke selected ones to emit a schedule. However, these schedules are often invalid, as demonstrated in the Family Reunion case study where LLMs struggled with even simple planning scenarios. Even when a valid static plan can be obtained through LLM-recommended solvers, this solution merely completes \textbf{Phase 1 / Layer 1} of the $\ALAS$ framework—essentially just generating a workflow template $\mathcal{W}_{\text{template}}$.

By contrast, $\ALAS$ feeds this preliminary plan into \emph{Phases 2–3 / Layer 1} (validation \& refinement) to yield a validated $\mathcal{W}_{\text{template}}$. Layers 2–3 then \emph{instantiate and run} a network of code-generated agents, denoted
as $\mathcal{W}_{\text{exec}}$. At runtime, the Local Reactive Compensation Protocol (LRCP) continuously logs state, detects disruptions, and triggers local repairs, capabilities that static schedules fundamentally lack (see architecture recap in Section \ref{sec:ALAS-architecture}).

\begin{table}[t]
\centering
\caption{JSSP Scheduling Prompt.  
Given a JSSP benchmark instance, the LLM first \emph{searches} for
candidate algorithms, selects the one that yields the \emph{minimum
makespan}, and returns both the algorithm’s key hyper-parameters and the resulting plan $\mathcal{W}_{\text{template}}$. \textbf{Note} that feasibility validation of $\mathcal{W}_{\text{template}}$ is yet to be performed.}
\vspace{0.1in}
\label{tab:JSSP-prompt}
{\fontsize{9pt}{9pt}\selectfont
\begin{tabular}{|>{\raggedright\arraybackslash}p{0.96\linewidth}|}
\toprule 
\hline
\textbf{Role.} You are a scheduling supervisor tasked with producing an
\emph{optimal} job-shop schedule.\\
\hline
\textbf{Objective.} Report the \textbf{minimum makespan}, the
\textbf{algorithm} used, and a \textit{schedule} $\mathcal{W}_{\text{template}}$ that achieves
this makespan. This is achieved by the recipient LLM recommending a list of solvers to execute and compare. (Execution and comparison
can be performed by the LLM or locally.) \\\hline

\textbf{Constraints.}
\begin{minipage}{0.95\linewidth}
\begin{enumerate}[leftmargin=1.2em,noitemsep]
  \item \emph{Job order}: operations of each job follow the given
        sequence.
  \item \emph{Machine capacity}: a machine processes only one operation
        at any time.
\end{enumerate}
\end{minipage}\\\hline

\textbf{Input.} A list of jobs, each as \texttt{(machine,\,duration)}
pairs.  
\textit{Example:}\\
\texttt{Job1: [(M\_A,3), (M\_B,5), (M\_C,2)]}\\
\texttt{Job2: [(M\_B,4), (M\_A,6)]}\\\hline

\textbf{Output.} Return
\begin{minipage}{0.95\linewidth}
\begin{itemize}[leftmargin=1.2em,noitemsep]
  \item \texttt{makespan} (integer)
  \item \texttt{algorithm} (string)
  \item \texttt{params} (JSON object of key hyper-parameters)  
  \item \texttt{schedule} $\mathcal{W}_{\text{template}}$: list of operations
        \{job, step (1-based), machine, start, end\};
\end{itemize}
\end{minipage}\\
\textit{Example:}\\
\begin{minipage}{0.95\linewidth}
\small\texttt{[
\{"job":"Job1","step":1,"machine":"M\_A","start":0,"end":3\},\\
\{"job":"Job2","step":1,"machine":"M\_B","start":0,"end":4\},\\
\{"job":"Job1","step":2,"machine":"M\_B","start":4,"end":9\},
\ldots ]} \\
\end{minipage}\\\hline
\bottomrule
\end{tabular}}
\end{table}

\vspace{-.1in}
\subsection{LLM Heuristic Baselines (sampled)}
For each benchmark instance we asked every standalone LLM to choose an
off-the-shelf optimization method, list key hyperparameters (when
provided), and report the makespan it expected to achieve.
Table~\ref{tab:llm-heuristics-merged} shows a \emph{five-instance sample}
per LLM; the full tables appear in the supplemental material.  
From $\ALAS$'s perspective this delivers only \textbf{Phase 1 of Layer 1}
(see Fig.~7): a static schedule $\mathcal{W}_{\text{template}}$.
Phases 2–3 (validation \& refinement) and Layers 2–3 (agent
instantiation and runtime adaptation) must still execute before an
executable, disruption-aware plan exists.

\begin{table}[ht!]
\vspace{-.05in}
\centering
\caption{Sampled LLM-proposed heuristics (5 rows per model).}
\vspace{.05in}
\label{tab:llm-heuristics-merged}
{\fontsize{7pt}{9pt}\selectfont
\begin{tabular}{lll}
\toprule
\textbf{Model} & \textbf{Dataset} & \textbf{Heuristic Strategy / Parameters} \\
\midrule
Claude 3.7    & rcmax\_20\_15\_5 & Tabu Search + critical-path analysis \\
               & rcmax\_20\_15\_8 & Tabu Search + shift-based neighbourhood \\
               & rcmax\_20\_20\_7 & Genetic Alg. + critical-path optimisation \\
               & rcmax\_30\_15\_5 & Constraint Prog. (precedence relaxation) \\
               & rcmax\_40\_15\_8 & Tabu Search + job-insertion strategy \\
\addlinespace
Gemini 2.5    & rcmax\_20\_15\_5 & “Gemini-optimised” Tabu Search \\
               & rcmax\_20\_20\_7 & Gemini-guided Simulated Annealing \\
               & rcmax\_30\_15\_5 & Gemini Constraint-Programming heuristic \\
               & rcmax\_40\_15\_10 & Gemini Shifting-Bottleneck dispatch \\
               & rcmax\_50\_20\_6 & Gemini Hybrid GA–Tabu \\
\addlinespace
GPT-4o        & rcmax\_20\_15\_5 & Genetic Alg. + adaptive mutation \\
               & rcmax\_20\_20\_8 & Particle Swarm Opt. (inertia weight=0.7) \\
               & rcmax\_30\_15\_5 & Ant Colony Opt. (pheromone $\alpha\!=\!1.0$) \\
               & rcmax\_40\_15\_10 & Bee Algorithm + neighbourhood search \\
               & rcmax\_50\_15\_4 & Simulated Annealing (adaptive $T$) \\
\addlinespace
DeepSeek-R1   & rcmax\_20\_15\_5 & GA + priority rules (pop.=50, iters=200) \\
               & rcmax\_20\_20\_8 & Simulated Annealing ($T_0\!=\!100$, cool=0.95) \\
               & rcmax\_30\_15\_4 & Ant Colony Opt. (ants=40, $\rho\!=\!0.1$) \\
               & rcmax\_40\_15\_10 & Iterated Greedy (destroy=30\%) \\
               & rcmax\_50\_20\_9 & Adaptive Large-Neighbourhood Search \\
\bottomrule
\end{tabular}}
\vspace{-.05in}
\end{table}
\vspace{-.05in}
\subsection{ALAS's Validation Replan Iterations}
Given the initial $\mathcal{W}_{\text{template}}$, $\ALAS$
completes its Layer 1 operation by executing a
validation-replan iteration cycle until a valid plan
is obtained. In our experiments, this convergence
typically requires up to 5 iterations on all benchmark
datasets, as depicted 
in Figure~\ref{fig:alas_convergence_ub}.

\begin{figure}[t!]
\centering
\begin{tikzpicture}
\begin{axis}[
    width=0.72\textwidth,
    height=4.3cm,
    xlabel={Iterations},
    ylabel={Makespan},
    xmin=1, xmax=20,
    legend style={at={(1.05,1)}, anchor=north west, font=\scriptsize},
    grid=major,
    ymajorgrids=true,
    xtick={0,10,20,30,40,50},
    every axis plot/.append style={thick},
    title={ALAS Convergence on 5 Benchmark Datasets}
]

\addplot+[red, mark=*, mark size=1pt] coordinates {
(1,6922.625) (2,5923.125) (3,5299.5625) (4,4983.0625) (5,4884.9375)
(6,4884.9375) (7,4884.9375) (8,4884.9375) (9,4884.9375) (10,4884.9375)
(11,4884.9375) (12,4884.9375) (13,4884.9375) (14,4884.9375) (15,4884.9375)
(16,4884.9375) (17,4884.9375) (18,4884.9375) (19,4884.9375) (20,4884.9375)
(21,4884.9375) (22,4884.9375) (23,4884.9375) (24,4884.9375) (25,4884.9375)
(26,4884.9375) (27,4884.9375) (28,4884.9375) (29,4884.9375) (30,4884.9375)
(31,4884.9375) (32,4884.9375) (33,4884.9375) (34,4884.9375) (35,4884.9375)
(36,4884.9375) (37,4884.9375) (38,4884.9375) (39,4884.9375) (40,4884.9375)
(41,4884.9375) (42,4884.9375) (43,4884.9375) (44,4884.9375) (45,4884.9375)
(46,4884.9375) (47,4884.9375) (48,4884.9375) (49,4884.9375) (50,4884.9375)
};
\addlegendentry{ALAS-DMU}

\draw[dashed, red] (axis cs:1,4227.44) -- (axis cs:50,4227.44);
\node[red, anchor=west, font=\scriptsize] at (axis cs:50,4884.9375) {+19.09\%};

\addplot+[red!70!black, mark=*, mark size=1pt] coordinates {
(1,4601.39) (2,3899.22) (3,3486.51) (4,3094) (5,3094)
(6,3094) (7,3094) (8,3094) (9,3094) (10,3094)
(11,3094) (12,3094) (13,3094) (14,3094) (15,3094)
(16,3094) (17,3094) (18,3094) (19,3094) (20,3094)
(21,3094) (22,3094) (23,3094) (24,3094) (25,3094)
(26,3094) (27,3094) (28,3094) (29,3094) (30,3094)
(31,3094) (32,3094) (33,3094) (34,3094) (35,3094)
(36,3094) (37,3094) (38,3094) (39,3094) (40,3094)
(41,3094) (42,3094) (43,3094) (44,3094) (45,3094)
(46,3094) (47,3094) (48,3094) (49,3094) (50,3094)
};
\addlegendentry{ALAS-TA}
\draw[dashed, red!70!black] (axis cs:1,3072) -- (axis cs:50,3072);
\node[red!70!black, anchor=west, font=\scriptsize] at (axis cs:50,3094) {+0.86\%};

\addplot+[red!40!black, mark=*, mark size=1pt] coordinates {
(1,956.6667) (2,787) (3,700.6667) (4,672.3333) (5,667)
(6,667) (7,667) (8,667) (9,667) (10,667)
(11,667) (12,667) (13,667) (14,667) (15,667)
(16,667) (17,667) (18,667) (19,667) (20,667)
(21,667) (22,667) (23,667) (24,667) (25,667)
(26,667) (27,667) (28,667) (29,667) (30,667)
(31,667) (32,667) (33,667) (34,667) (35,667)
(36,667) (37,667) (38,667) (39,667) (40,667)
(41,667) (42,667) (43,667) (44,667) (45,667)
(46,667) (47,667) (48,667) (49,667) (50,667)
};
\addlegendentry{ALAS-ABZ}
\draw[dashed, red!40!black] (axis cs:1,667) -- (axis cs:50,667);
\node[red!40!black, anchor=west, font=\scriptsize] at (axis cs:50,667) {+0\%};

\addplot+[red!10!black, mark=*, mark size=1pt] coordinates {
(1,2847.47) (2,2374.47) (3,2143.27) (4,2060.27) (5,2032.73)
(6,2032.73) (7,2032.73) (8,2032.73) (9,2032.73) (10,2032.73)
(11,2032.73) (12,2032.73) (13,2032.73) (14,2032.73) (15,2032.73)
(16,2032.73) (17,2032.73) (18,2032.73) (19,2032.73) (20,2032.73)
(21,2032.73) (22,2032.73) (23,2032.73) (24,2032.73) (25,2032.73)
(26,2032.73) (27,2032.73) (28,2032.73) (29,2032.73) (30,2032.73)
(31,2032.73) (32,2032.73) (33,2032.73) (34,2032.73) (35,2032.73)
(36,2032.73) (37,2032.73) (38,2032.73) (39,2032.73) (40,2032.73)
(41,2032.73) (42,2032.73) (43,2032.73) (44,2032.73) (45,2032.73)
(46,2032.73) (47,2032.73) (48,2032.73) (49,2032.73) (50,2032.73)
};
\addlegendentry{ALAS-SWV}
\draw[dashed, red!10!black] (axis cs:1,2032.73) -- (axis cs:50,2032.73);
\node[red!10!black, anchor=west, font=\scriptsize] at (axis cs:50,2032.73) {+0\%};

\addplot+[red!60!black, mark=*, mark size=1pt] coordinates {
(1,1243.75) (2,1039.25) (3,944.5) (4,918.25) (5,912)
(6,912) (7,912) (8,912) (9,912) (10,912)
(11,912) (12,912) (13,912) (14,912) (15,912)
(16,912) (17,912) (18,912) (19,912) (20,912)
(21,912) (22,912) (23,912) (24,912) (25,912)
(26,912) (27,912) (28,912) (29,912) (30,912)
(31,912) (32,912) (33,912) (34,912) (35,912)
(36,912) (37,912) (38,912) (39,912) (40,912)
(41,912) (42,912) (43,912) (44,912) (45,912)
(46,912) (47,912) (48,912) (49,912) (50,912)
};
\addlegendentry{ALAS-YN}
\draw[dashed, red!60!black] (axis cs:1,912) -- (axis cs:50,912);
\node[red!60!black, anchor=west, font=\scriptsize] at (axis cs:50,912) {+0\%};

\end{axis}
\end{tikzpicture}
\caption{Convergence of ALAS on five datasets and UB. RCMax from DMU Benchmark Set (DMU), Taillard Job Shop Problems (TA), Adams, Balas \& Zawack Job Shop (ABZ), Swv Job Shop Benchmark Set (SWZ), and Yamada and Nakano Benchmark Set (YN). The percentage of gaps for each dataset is 19. 09\%, 0. 86\%, 0\%, 0\%, and 0\%.}
\label{fig:alas_convergence_ub}
\end{figure}
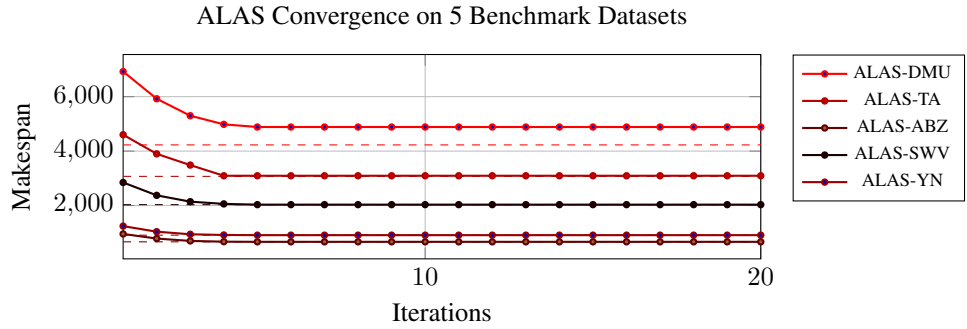

\vspace{-.1in}
\subsection{Examples of Makespan at Convergence}

After plan generation, validation, and local replanning, we obtained optimized makespan results for various benchmark instances. Figure~\ref{fig:gantt_charts} presents a selection of four representative Gantt charts showing JSSP instances of different sizes and complexity.

\newcommand{\ganttHeightScale}{0.88}  

\begin{figure}[b!]
    \centering
    \begin{subfigure}[b]{0.49\textwidth}
        \centering
        \includegraphics[width=\textwidth,height=\ganttHeightScale\textwidth]{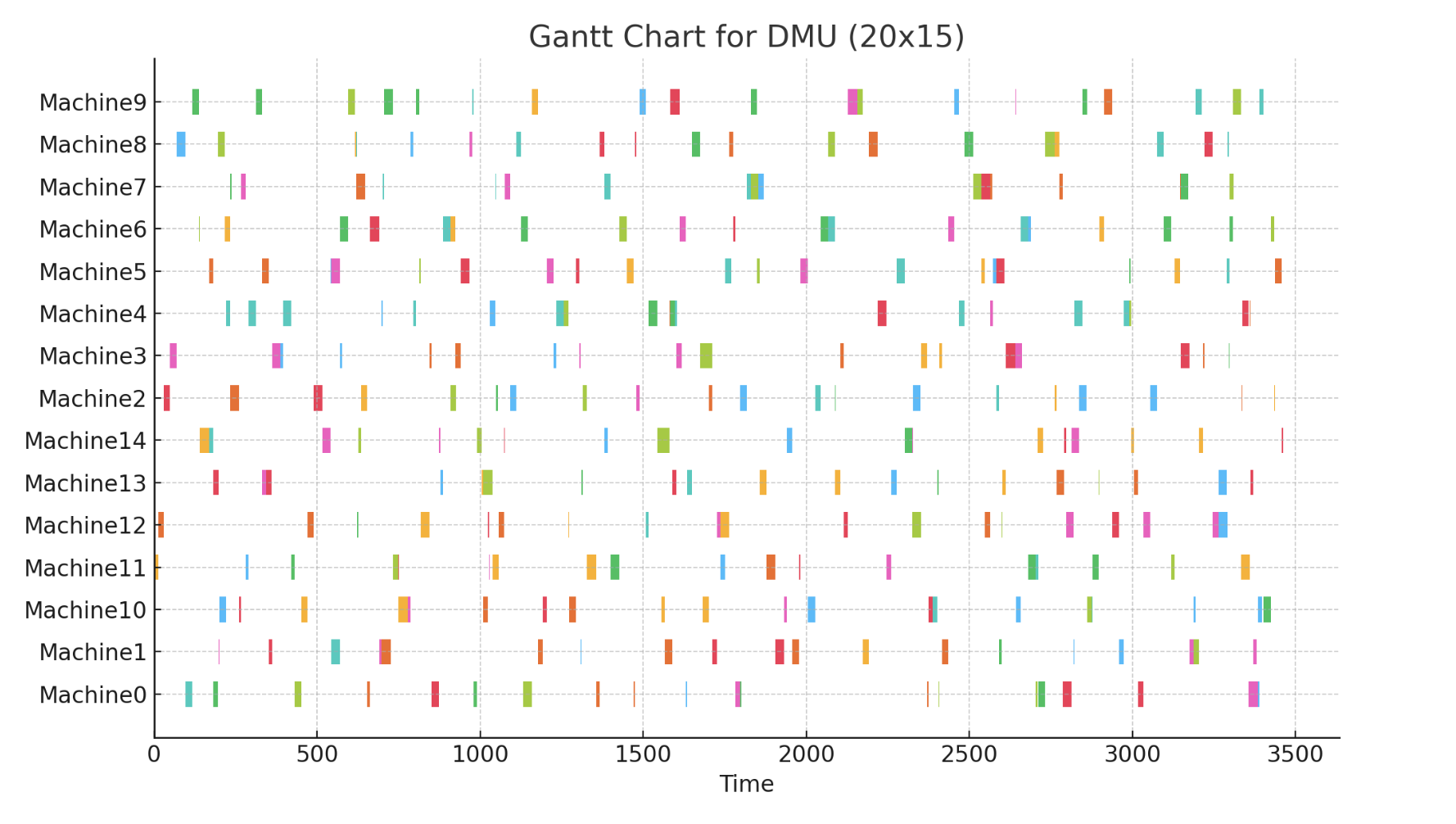}
        \caption{\textbf{rcmax\_20\_15\_5} (J=20, M=15)}
        \label{fig:gantt_rcmax}
    \end{subfigure}
    \hspace{-.2in}
    \begin{subfigure}[b]{0.49\textwidth}
        \centering
        \includegraphics[width=\textwidth,height=\ganttHeightScale\textwidth]{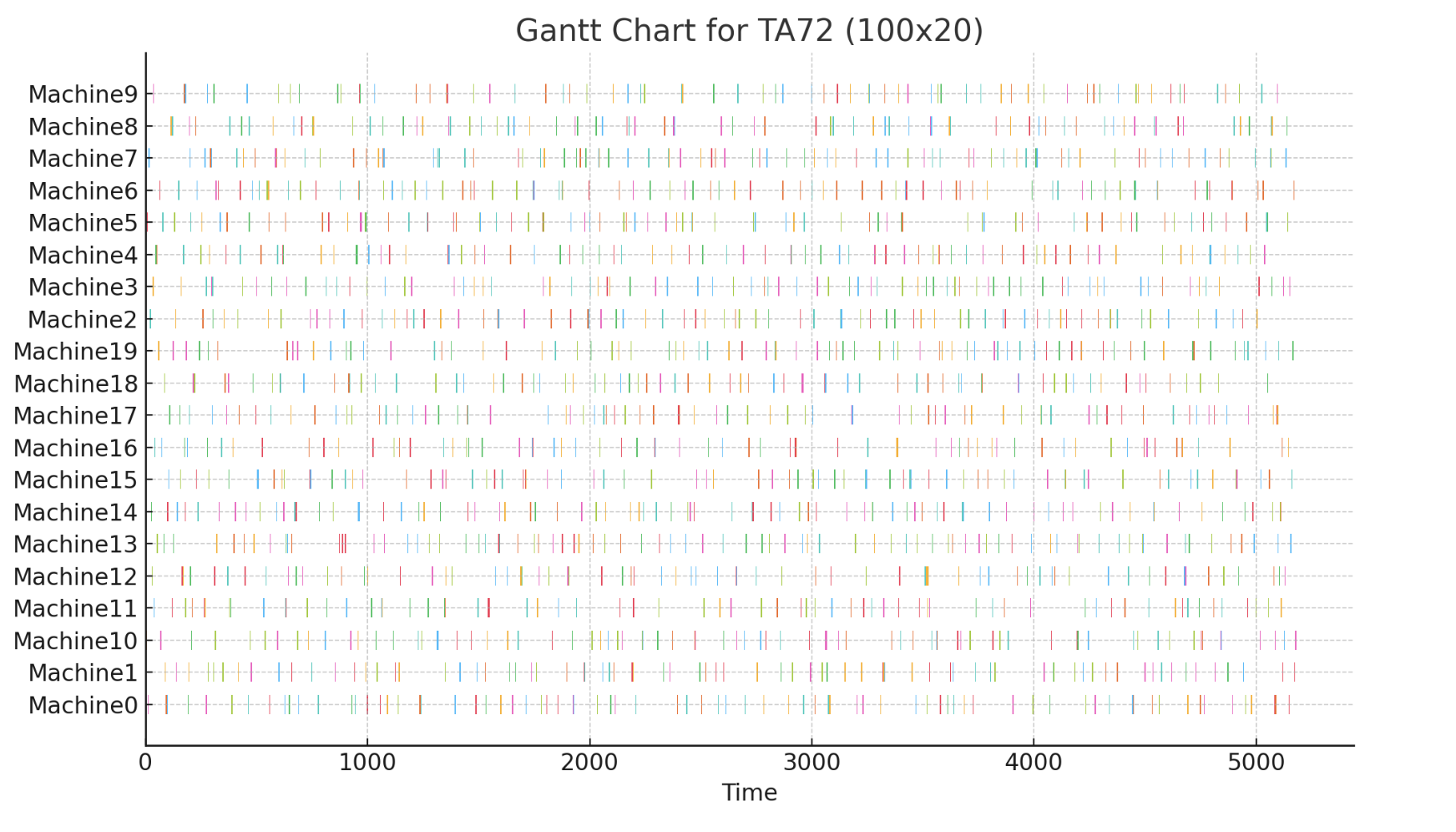}
        \caption{\textbf{Abz07} (J=20, M=15)}
        \label{fig:gantt_abz07}
    \end{subfigure}
    
    \vspace{0.3cm}
    
    \begin{subfigure}[b]{0.49\textwidth}
        \centering
        \includegraphics[width=\textwidth,height=\ganttHeightScale\textwidth]{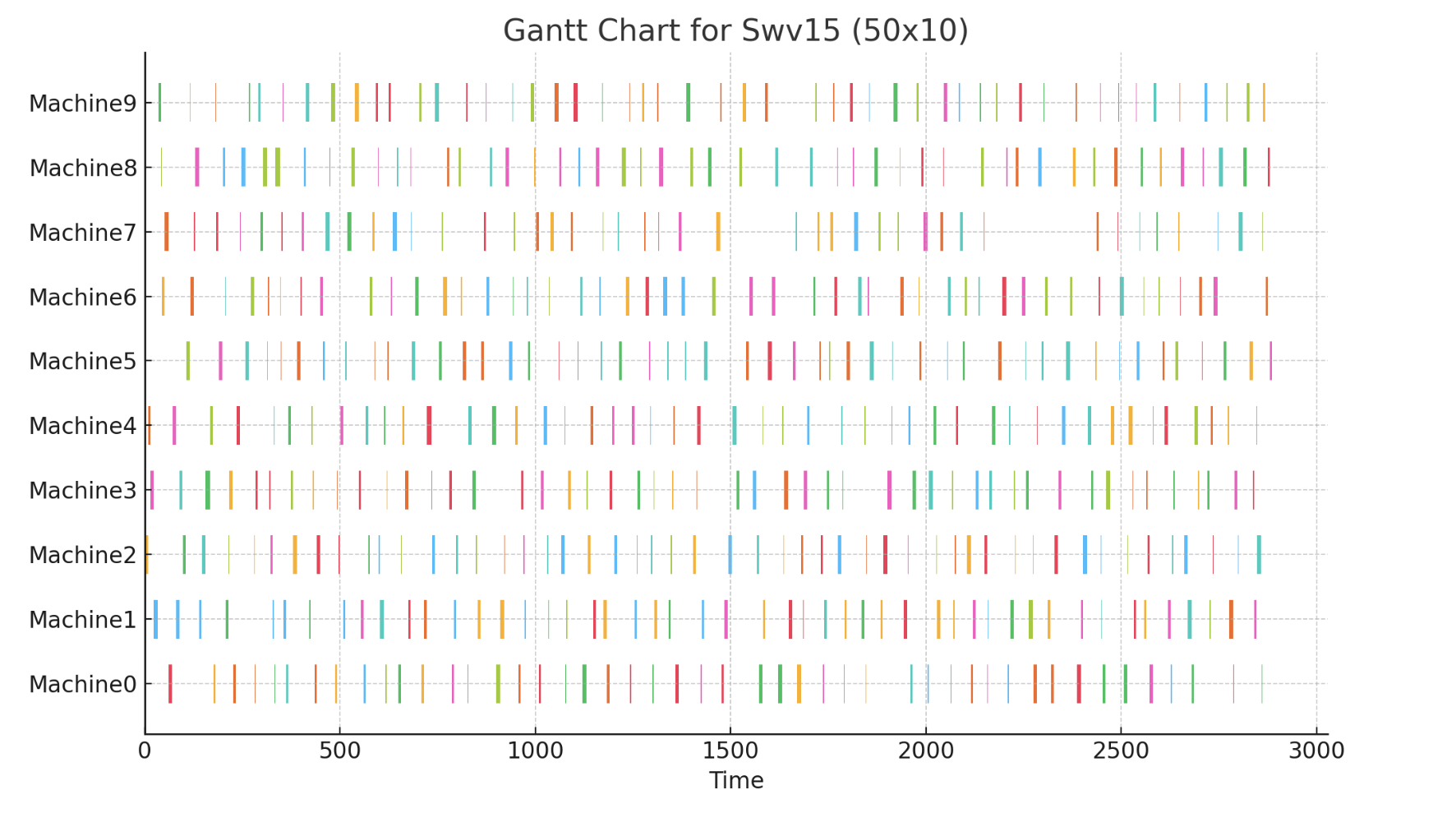}
        \caption{\textbf{Swv15} (J=50, M=10)}
        \label{fig:gantt_swv15}
    \end{subfigure}
    \hspace{-.2in}
    \begin{subfigure}[b]{0.49\textwidth}
        \centering
        \includegraphics[width=\textwidth,height=\ganttHeightScale\textwidth]{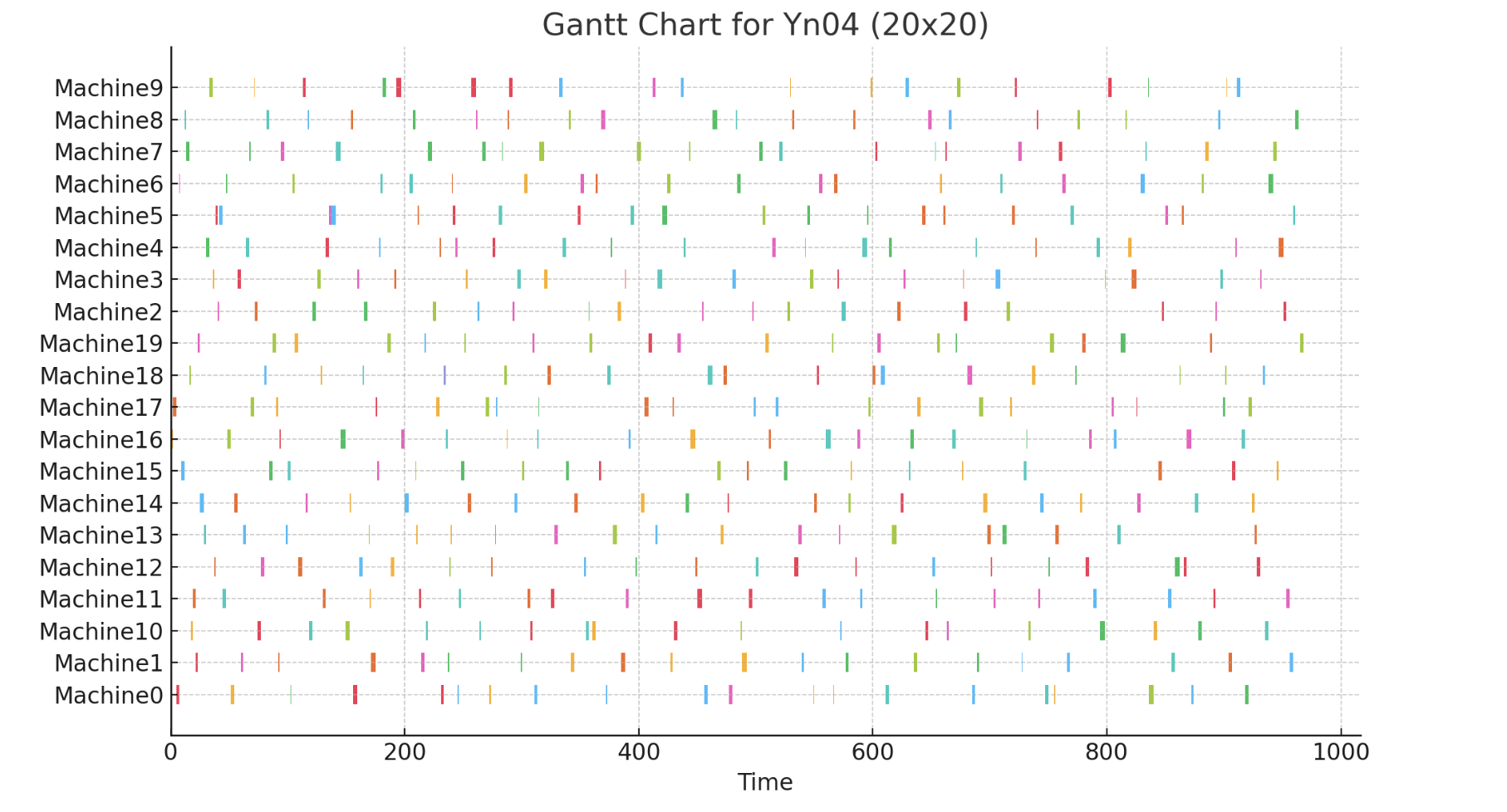}
        \caption{\textbf{Yn04} (J=20, M=20)}
        \label{fig:gantt_yn04}
    \end{subfigure}
    
    \caption{Gantt charts of optimized schedules produced by $\ALAS$ for four representative JSSP benchmark instances with varying job and machine counts. These visualizations demonstrate $\ALAS$'s ability to efficiently allocate resources and minimize makespan across different problem scales. The larger instance TA72 (J=100, M=20) is available in the supplementary materials.}
    \label{fig:gantt_charts}
\end{figure}

\vspace{-.1in}
\subsection{Additional Experimental Results and Analysis}

Figure~\ref{fig:JSSP-DMU-TA-trendline} extends the results presented in Figure~\ref{fig:ALAS-JSSP-exp1} (Section~\ref{sec:ALAS-exp-case3}), showing that both $\ALAS$ and $\ALAS$+LRCP consistently outperform competing methods in virtually all JSSP instances in the DMU and TA benchmarks. While $\ALAS$ provides effective static sequential scheduling, $\ALAS$+LRCP further enhances performance through strategic local job exchanges. Since LRCP only implements rescheduling when viable opportunities exist, it guarantees strict makespan improvements.

\begin{figure}[h!]
\vspace{-.1in}
\centering
\begin{tikzpicture}
\begin{axis}[
    width=0.72\textwidth,
    height=6.5cm,
    xlabel={Instance},
    ylabel={Makespan},
    xtick=data,
    xticklabels={DMU03, DMU04, DMU08, DMU09, DMU13, DMU14, DMU18, DMU19, DMU23, DMU24, DMU28, DMU29, DMU33, DMU34, DMU38, DMU39},
    x tick label style={rotate=45, anchor=east, font=\tiny},
    legend style={font=\scriptsize, at={(1.05,1)}, anchor=north west},
    ymajorgrids=true,
    grid style=dashed,
    every axis plot/.append style={thick},
    title={(a) ALAS and ALAS+LRCP Comparison on DMU Dataset}
]

\addplot+[mark=*, color=blue] coordinates {(0,3827) (1,3889) (2,4228) (3,4094) (4,5451) (5,5306) (6,5326) (7,5174) (8,5948) (9,6078) (10,6737) (11,6602) (12,6890) (13,7523) (14,7685) (15,8097)};
\addlegendentry{Random}

\addplot+[mark=*, color=orange] coordinates {(0,4592) (1,4047) (2,4551) (3,4511) (4,5580) (5,5591) (6,5810) (7,5787) (8,7045) (9,6484) (10,7322) (11,7386) (12,8779) (13,7991) (14,9051) (15,8514)};
\addlegendentry{LPT}

\addplot+[mark=*, color=green] coordinates {(0,3630) (1,3541) (2,4714) (3,4283) (4,4813) (5,4583) (6,6231) (7,5126) (8,6250) (9,5503) (10,6558) (11,6565) (12,7361) (13,7026) (14,7954) (15,7592)};
\addlegendentry{SPT}

\addplot+[mark=*, color=cyan] coordinates {(0,4232) (1,4642) (2,4459) (3,4690) (4,5207) (5,4811) (6,5480) (7,5203) (8,6521) (9,6595) (10,7697) (11,7690) (12,7631) (13,7740) (14,8555) (15,8908)};
\addlegendentry{STPT}

\addplot+[mark=*, color=purple] coordinates {(0,3435) (1,3355) (2,3999) (3,3869) (4,4759) (5,4238) (6,5003) (7,4930) (8,5383) (9,5358) (10,5927) (11,6107) (12,6282) (13,6359) (14,7604) (15,6953)};
\addlegendentry{MPSR}

\addplot+[mark=*, color=teal] coordinates {(0,3303) (1,3321) (2,4098) (3,3753) (4,4708) (5,4124) (6,4800) (7,4837) (8,5240) (9,5319) (10,5948) (11,5824) (12,6458) (13,6284) (14,7275) (15,6776)};
\addlegendentry{DRL-Liu}

\addplot+[mark=*, color=brown] coordinates {(0,3540) (1,3406) (2,3802) (3,4196) (4,4765) (5,4289) (6,4696) (7,4666) (8,5391) (9,5560) (10,6017) (11,6236) (12,6109) (13,6327) (14,7267) (15,6941)};
\addlegendentry{GP}

\addplot+[mark=*, color=violet] coordinates {(0,3651) (1,3499) (2,4023) (3,4136) (4,4812) (5,4213) (6,4917) (7,5245) (8,5595) (9,5458) (10,6142) (11,6224) (12,6081) (13,6279) (14,7501) (15,7124)};
\addlegendentry{GEP}

\addplot+[mark=*, color=black] coordinates {(0,3462) (1,3235) (2,3728) (3,3857) (4,4658) (5,3980) (6,4724) (7,4715) (8,5151) (9,5226) (10,5838) (11,5941) (12,6029) (13,6148) (14,7168) (15,6693)};
\addlegendentry{SeEvo(GLM3)}

\addplot+[mark=*, color=black!70] coordinates {(0,3238) (1,3212) (2,3728) (3,3828) (4,4709) (5,3980) (6,4724) (7,4816) (8,5258) (9,5316) (10,5944) (11,5825) (12,6029) (13,6146) (14,7170) (15,6590)};
\addlegendentry{SeEvo(GPT3.5)}

\addplot+[mark=*, color=red, dashed] coordinates {(0,3462) (1,3235) (2,3728) (3,3857) (4,4658) (5,3980) (6,4724) (7,4715) (8,5151) (9,5226) (10,5838) (11,5941) (12,6029) (13,6148) (14,7168) (15,6693)};
\addlegendentry{ALAS+LRCP (Claude-3.7)}

\addplot+[mark=*, color=red, dash dot] coordinates {(0,3462) (1,3235) (2,3728) (3,3857) (4,4658) (5,3980) (6,4724) (7,4715) (8,5151) (9,5226) (10,5838) (11,5941) (12,6029) (13,6148) (14,7168) (15,6693)};
\addlegendentry{ALAS (Claude-3.7)}
\end{axis}
\end{tikzpicture}

\begin{tikzpicture}
\begin{axis}[
    width=0.88\textwidth,
    height=6cm,
    xlabel={Instance},
    ylabel={Makespan},
    symbolic x coords={TA01, TA02, TA51, TA52, TA61, TA71, TA72},
    xtick=data,
    xticklabel style={rotate=45, anchor=east, font=\tiny},
    legend style={at={(0.5,-0.3)}, anchor=north, legend columns=3, font=\tiny},
    ymajorgrids=true,
    grid style={dashed,gray!30},
    title={(b) ALAS and ALAS+LRCP Comparison on TA Dataset}
]

\addplot+[mark=*, thick, blue] coordinates {
    (TA01,1957) (TA02,1759) (TA51,3844) (TA52,3715) (TA61,4188) (TA71,6754) (TA72,6674)};
\addlegendentry{LSO}

\addplot+[mark=triangle*, thick, blue] coordinates {
    (TA01,1664) (TA02,1538) (TA51,3768) (TA52,3588) (TA61,3752) (TA71,6705) (TA72,6351)};
\addlegendentry{SPT/TWKR}

\addplot+[mark=square*, thick, blue] coordinates {
    (TA01,1711) (TA02,1639) (TA51,3762) (TA52,3511) (TA61,3633) (TA71,6321) (TA72,6232)};
\addlegendentry{DRL-Chen}

\addplot+[mark=diamond*, thick, blue] coordinates {
    (TA01,1433) (TA02,1544) (TA51,3599) (TA52,3341) (TA61,3654) (TA71,6452) (TA72,5695)};
\addlegendentry{DRL-Zhang}

\addplot+[mark=*, thick, cyan] coordinates {
    (TA01,1492) (TA02,1425) (TA51,3608) (TA52,3524) (TA61,3548) (TA71,6289) (TA72,6002)};
\addlegendentry{DRL-Liu}

\addplot+[mark=triangle*, thick, cyan] coordinates {
    (TA01,1547) (TA02,1565) (TA51,3603) (TA52,3346) (TA61,3685) (TA71,6305) (TA72,5776)};
\addlegendentry{GP}

\addplot+[mark=square*, thick, cyan] coordinates {
    (TA01,1547) (TA02,1486) (TA51,3668) (TA52,3324) (TA61,3642) (TA71,6278) (TA72,5625)};
\addlegendentry{GEP}

\addplot+[mark=diamond*, thick, gray] coordinates {
    (TA01,1427) (TA02,1465) (TA51,3364) (TA52,3286) (TA61,3529) (TA71,6071) (TA72,5604)};
\addlegendentry{SeEvo(GLM3)}

\addplot+[mark=otimes*, thick, gray] coordinates {
    (TA01,1427) (TA02,1437) (TA51,3412) (TA52,3245) (TA61,3537) (TA71,6099) (TA72,5575)};
\addlegendentry{SeEvo(GPT3.5)}

\addplot+[mark=*, thick, red] coordinates {
    (TA01,1243) (TA02,1252) (TA51,2766) (TA52,2819) (TA61,2905) (TA71,5478) (TA72,5198)};
\addlegendentry{ALAS+LRCP}

\addplot+[mark=*, thick, red, dashed] coordinates {
    (TA01,1231) (TA02,1244) (TA51,2760) (TA52,2756) (TA61,nan) (TA71,5464) (TA72,nan)};
\addlegendentry{ALAS}
\end{axis}
\end{tikzpicture}
\vspace{.1in}
\caption{Performance of $\ALAS$ and $\ALAS$+LRCP compared to baseline methods across DMU and TA benchmarks. Lower makespan values indicate better performance.
The size of JSSP instances grows by the instance number, and therefore the
makespan values.}
\label{fig:JSSP-DMU-TA-trendline}
\end{figure}
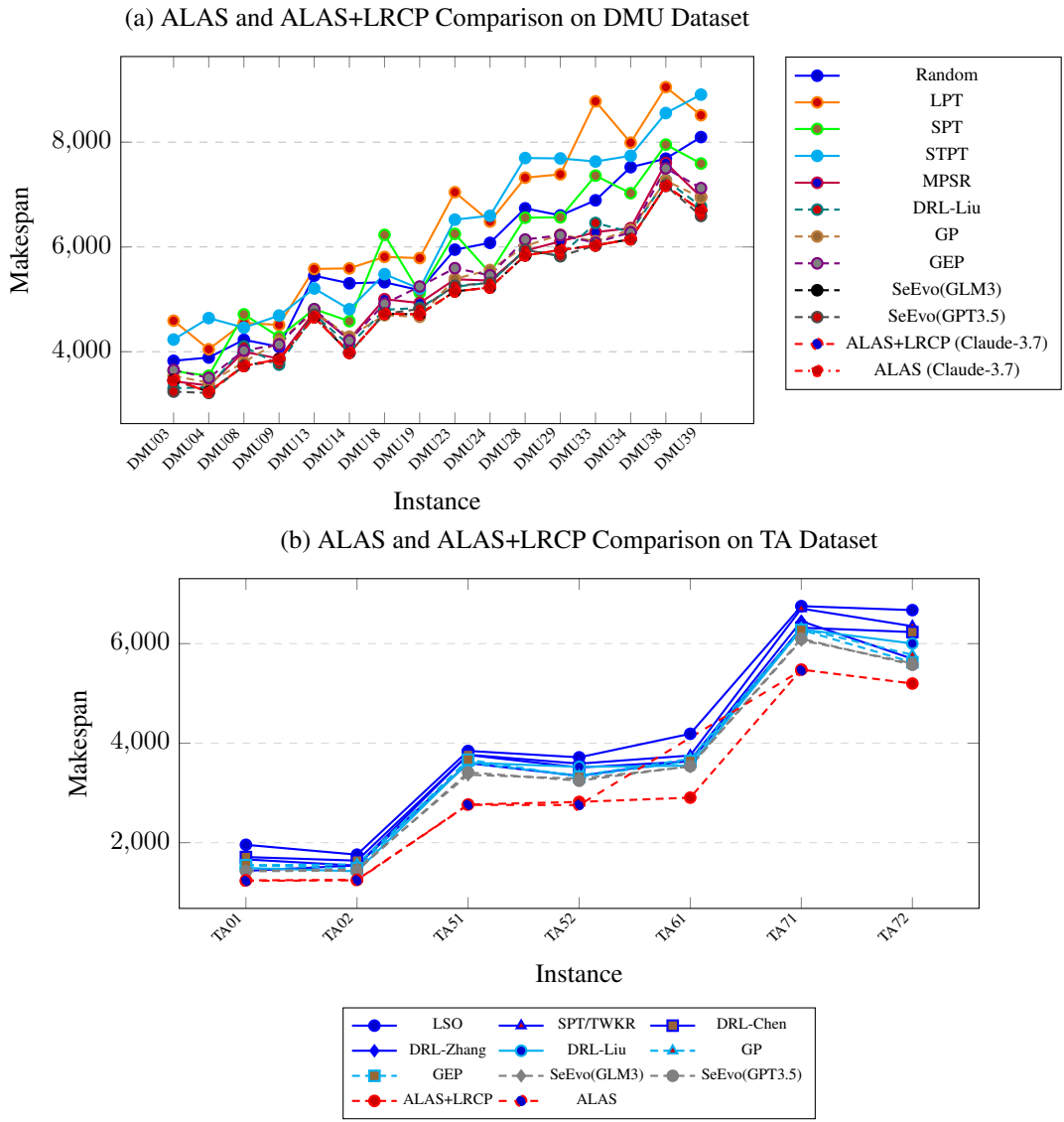
